\documentclass[a4paper,11pt]{article}
\usepackage{mattsstyle}
\usepackage{xfrac}
\usepackage{graphicx}
\usepackage{caption}
\usepackage{subcaption}
\usepackage[normalem]{ulem}
\usepackage{arydshln}

\usepackage{booktabs,tabularx,ragged2e}
\newcolumntype{Y}{>{\RaggedRight\arraybackslash}X}
\usepackage{adjustbox} 

\usepackage{multirow}

\usepackage{comment}
\usepackage{bm}

\def\TLp#1{\mathrm{TL}^{#1}}
\def\dTLp#1{d_{\TLp{#1}}}
\def\Lp#1{\mathrm{L}^{#1}}

\def\Ck#1{\mathrm{C}^{#1}}

\def\Wkp#1#2{\mathrm{W}^{#1,#2}}

\def\spaceBar{\, | \,}

\allowdisplaybreaks

\def\commentOut#1{}

\title{Higher-Order Regularization Learning on Hypergraphs}
\author[1]{Adrien Weihs\thanks{Corresponding author. Email: \texttt{weihs@math.ucla.edu}}}
\author[1]{Andrea L. Bertozzi}
\author[2]{Matthew Thorpe}
\date{November 2025}
\affil[1]{Department of Mathematics,\protect\\ University of California Los Angeles,\protect\\ Los Angeles, CA 90095, USA. \vspace{\baselineskip}}
\affil[2]{Department of Statistics,\protect\\ University of Warwick,\protect\\ Coventry, CV4 7AL, UK.}

\begin{document}

\maketitle

\begin{abstract}
\noindent Higher-Order Hypergraph Learning (HOHL)  
was recently introduced as a principled alternative to classical hypergraph regularization, enforcing higher-order smoothness via powers of multiscale Laplacians induced by the hypergraph structure. Prior work established the well- and ill-posedness of HOHL through an asymptotic consistency analysis in geometric settings. We extend this theoretical foundation by proving the consistency of a truncated version of HOHL and deriving explicit convergence rates when HOHL is used as a regularizer in fully supervised learning. We further demonstrate its strong empirical performance in active learning and in datasets lacking an underlying geometric structure, highlighting HOHL’s versatility and robustness across diverse learning settings.
\end{abstract}

\keywords{Hypergraph learning, Semi-supervised learning, Sobolev regularization, Asymptotic consistency, Multiscale learning, Non-Euclidean data, Active learning, Graph Laplacians}

\subjclass{49J55, 49J45, 62G20, 65N12}

\newpage

\section{Introduction} 

Graphs play a foundational role in machine learning, enabling effective modeling of relational data across a range of tasks—from semi-supervised learning and clustering to recommendation systems and manifold learning, e.g.~\cite{LapRef,zhou2004lazy,zhou2011semi,Mai,poissonLearning,flores2019algorithms,Merkurjev,xia2021graph,ju2024survey}. However, many real-world phenomena involve more complex interactions among sets of nodes, that are not fully captured by pairwise edges. Hypergraphs extend graphs by allowing hyperedges to connect arbitrary subsets of nodes, and hypergraph-based methods are used broadly in various areas of science such as in~\cite{hgLearningPractice,clique,hgPLaplacianGeometric,TVHg,shi2025hypergraphplaplacianequationsdata,Pirvu_2023_ICCV,ChenHg,hgWeighting,dynamicHG,neuhauser,scholkopfHyper2006,pmlr-v80-li18e,fazeny}.

A central research question concerns the comparison between hypergraph and graph-based learning methods. Many such comparisons are grounded in discrete arguments (e.g.,\cite{hypergraphGraph,jostMulas,chitra,jost,mulas}). More recently, asymptotic consistency frameworks—a popular technique for analyzing graph-based methods by relating discrete energies to continuum variational limits (e.g.,\cite{NIPS2006_5848ad95,COIFMAN20065,Gine,Singer,10.5555/3104322.3104459,calderGameTheoretic,Slepcev,Stuart,weihs2023consistency})—have been extended to the hypergraph regularization setting \cite{shi2025hypergraphplaplacianequationsdata,weihs2025Hypergraphs}. This continuum perspective allows for a principled assessment of the role of hypergraph structures, supports a classification of hypergraph learning algorithms \cite[Figure 2]{weihs2025Hypergraphs}, and enables a clearer understanding of the regularization behavior underlying complex discrete formulations.

In the analysis of graph- and hypergraph-based regularization, it is useful to distinguish between two complementary components of a regularizer:
(1) the support of interactions, i.e., which nodes influence one another (determined by the graph or hypergraph topology), and
(2) the interaction mechanism, i.e., how these influences are aggregated or penalized (e.g., via first-order differences, higher-order derivatives, or more general nonlinear terms).
Classical hypergraph learning methods typically enrich the interaction support—by allowing edges to connect sets of nodes rather than pairs—but still rely on first-order, pairwise-like regularization mechanisms \cite{scholkopfHyper2006,weihs2025Hypergraphs}.

In this context, Higher-Order Hypergraph Learning (HOHL) was introduced as a method that more effectively leverages hypergraph structure—not only by modifying which interactions are considered, but also by altering the nature of such interactions. Specifically, HOHL decomposes the hypergraph into a sequence of subgraphs that capture interactions at multiple scales. On each subgraph, a distinct regularization strength is applied, allowing the model to enforce higher-order smoothness in a structured and scale-aware manner.
In doing so, HOHL exploits the full expressive potential of the hypergraph more fully and effectively. 
From an analytical perspective, HOHL is shown to converge to a higher-order Sobolev semi-norm, making it genuinely distinct from other hypergraph methods \cite{scholkopfHyper2006,shi2025hypergraphplaplacianequationsdata} that asymptotically recover the standard \( \Wkp{1}{p} \) regularization.

In this paper, we extend both the theoretical and computational analysis of HOHL. On the theoretical side, we prove that when HOHL is used as a regularizer in the fully supervised learning setting, it yields explicit rates of convergence between the learned function and the ground-truth target. Furthermore, we analyze a truncated version of the HOHL energy—commonly employed in practice due to its reduced computational complexity—and establish that it remains consistent, converging to the same higher-order continuum limit as the full model.

On the computational side, we demonstrate that HOHL preserves the quadratic form characteristic of Laplace learning and can, in fact, be interpreted as Laplace learning on a specially constructed graph. This equivalence implies that all existing computational techniques developed for Laplace learning are directly applicable to HOHL, enabling seamless integration into established workflows. In particular, we highlight this drop-in compatibility through an active learning application, where HOHL strongly outperforms traditional Laplacian-based approaches. Finally, we extend the HOHL framework to settings where the hypergraph is not embedded in some underlying metric space. This generalization necessitates a shift in the notion of scale-aware regularization, but continues to yield strong performance, achieving state-of-the-art results on several standard hypergraph benchmarks.

\subsection{Contributions}

Our main contributions are as follows:

\begin{enumerate}
    \item \textbf{Theoretical Guarantees for Supervised Learning:} We prove that using HOHL as a regularizer in the fully supervised setting yields explicit convergence rates between the learned function and the ground-truth target.

    \item \textbf{Consistency of Truncated HOHL:} We analyze a truncated version of HOHL, commonly used in practice for its computational efficiency, and establish that it remains consistent with the full model by converging to the same higher-order continuum limit.

    \item \textbf{Connection to Laplace Learning:} We show that HOHL preserves the quadratic form of Laplace learning and can be interpreted as Laplace learning on a specially constructed graph, making all standard computational techniques for Laplace learning directly applicable.

    \item \textbf{Plug-and-Play Use in Active Learning:} We demonstrate that HOHL can serve as a drop-in replacement for Laplace learning in existing pipelines, highlighting its advantages through strong empirical performance in active learning tasks.

    \item \textbf{Extension Beyond Geometric Hypergraphs:} We generalize HOHL to hypergraphs without an underlying metric structure by redefining the notion of multiscale regularization, achieving state-of-the-art results on standard hypergraph learning benchmarks.
\end{enumerate}

\subsection{Related works}

A growing body of work has focused on the asymptotic consistency and continuum analysis of graph-based regularization in the large-sample regime. These efforts include convergence results for total variation on graphs~\cite{Trillos3}, graph cuts and Cheeger-type problems~\cite{JMLR:v17:14-490,trillos2017estimating,thorpeCheeger,doi:10.1137/16M1098309}, the Mumford–Shah functional~\cite{Caroccia_2020}, and empirical risk minimization~\cite{garcia_trillos_murray_2017}. In the semi-supervised setting, particular attention has been given to $p$-Laplace learning~\cite{Slepcev}, fractional Laplacian methods~\cite{weihs2023consistency}, Lipschitz learning~\cite{pmlr-v40-Kyng15,doi:10.1137/18M1199241,Bungert,doi.org/10.48550/arxiv.2111.12370}, game-theoretic formulations~\cite{calderGameTheoretic}, Poisson learning~\cite{poissonLearning,bungert2024convergenceratespoissonlearning}, reweighted Laplacians~\cite{shi2017weighted}, and truncated energy models~\cite{Belkin2002UsingMS,sslManifolds}. These developments reflect a general trend toward understanding the behavior of discrete algorithms through the lens of continuum variational principles. Recently, such analyses have been extended to the hypergraph setting \cite{shi2025hypergraphplaplacianequationsdata,weihs2025Hypergraphs} 

Consistency between discrete energies \(\mathcal{E}_{n}\), defined for functions \(v_n : \Omega_n \to \mathbb{R}\), and a corresponding continuum energy \(\mathcal{E}_\infty\), defined on functions \(v : \Omega \to \mathbb{R}\), can be established through several analytical approaches:

\begin{itemize}

    \item \textit{Pointwise convergence}~\cite{NIPS2006_5848ad95, COIFMAN20065, Gine, 10.1007/11503415_32, 10.1007/11776420_7, Singer, 10.5555/3104322.3104459} examines whether \(\mathcal{E}_{n}(v|_{\Omega_n}) \to \mathcal{E}_\infty(v)\) as \(n \to \infty\), for sufficiently smooth functions \(v : \Omega \to \mathbb{R}\). A related approach considers the pointwise convergence of the associated Euler--Lagrange operators~\cite{weihs2023discreteToContinuum}.

    \item \textit{Spectral convergence}~\cite{NIPS2006_5848ad95, Trillos, CALDER2022123, 10.1214/009053607000000640, JMLR:v12:pelletier11a, 10.1093/imaiai/iaw016, https://doi.org/10.48550/arxiv.1510.08110} analyzes the convergence of the eigenvalues and eigenfunctions of the discrete operator associated with (through Euler-Lagrange equations) \(\mathcal{E}_{n}\) to those of the limiting operator appearing in \(\mathcal{E}_\infty\).
    
    \item \textit{Variational convergence}~\cite{calderGameTheoretic,cristoferi_thorpe_2020,Stuart,Trillos3,JMLR:v17:14-490,trillos2017estimating,GARCIATRILLOS2018239,Slepcev,thorpe_theil_2019,Gennip} concerns the convergence of minimizers of \(\mathcal{E}_{n}\) to those of \(\mathcal{E}_\infty\), typically formalized through \(\Gamma\)-convergence~\cite{gammaConvergence}. Among the three notions, it is often the most relevant in semi-supervised learning, where the final label assignments are derived from minimizers of the objective functional.
\end{itemize}

In this work, we focus on the latter two modes of convergence. In particular, to establish the variational convergence of our truncated energies, we analyze the spectral properties of the HOHL Laplacian~\cite{weihs2025Hypergraphs}. Our results in the fully supervised setting are also of variational type, providing convergence guarantees for minimizers of the discrete energies.

While much of the literature has focused on consistency, in the graph-based setting, recent works established convergence rates in terms of various parameters such as the number of points \( n \), the labeling rate, the graph connectivity parameter \( \varepsilon \) and the smoothness of the target function~\cite{weihs2023discreteToContinuum,calder2020rates,ElBouchairi}. These rates offer important theoretical guarantees for practical applications, where the dataset is finite and the discrete approximation error must be controlled. In this work, we extend such results to the HOHL framework, similarly to \cite{trillos2022rates}, showing explicit convergence rates between the discrete minimizers and the continuum ground truth under suitable regularity assumptions.

Beyond rates, computational efficiency is a key concern for applications. In practice, Laplace learning and related graph-based methods often rely on a spectrally truncated energy formulation. While these truncations are computationally efficient and widely adopted in large-scale settings, their theoretical justification has largely remained heuristic \cite{Belkin2002UsingMS,bertozziStuart,Miller}. In this work, we contribute to closing this gap by showing that even when the HOHL energy is truncated, it remains variationally consistent with the full model and converges to the same continuum limit.

Our final computational result establishes a connection between the HOHL Laplacian, and a broad body of work on graph reweighting \cite{shi2017weighted,properly} (in classical models) and graph rewiring\cite{ricciBronstein,overSmoothing,ricciOsher,rewiringMontufar,rewiringWeber,DGM} (in graph neural networks), both of which aim to improve learning performance by structurally modifying the graph. These modifications are often employed to address limitations such as oversmoothing or oversquashing \cite{overSmoothing}. In the spirit of~\cite{hypergraphGraph}, which advocates for representing hypergraph structure within enriched graph formulations, we show that HOHL can be interpreted as Laplace learning on a modified graph constructed directly from the original hypergraph structure.

\section{Background}

This section presents the mathematical tools used throughout the paper. We begin by recalling the \(\TLp{p}\) space, which provides a natural topology for comparing functions defined on discrete empirical measures to functions on the continuum. We then review key concepts from \(\Gamma\)-convergence theory, which we rely on to study the asymptotic behavior of our discrete variational problems. References for the material presented here include \cite{Trillos3,Slepcev,weihs2023consistency,gammaConvergence}.

\subsection{The \texorpdfstring{$\TLp{p}$}{TLp} Topology}

Let \(\mathcal{P}_p(\Omega)\) denote the set of Borel probability measures on a bounded domain \(\Omega \subset \mathbb{R}^d\) with finite \(p\)-th moment. For each \(\mu \in \mathcal{P}_p(\Omega)\), we denote by \(\Lp{p}(\mu)\) the space of \(\mu\)-measurable functions with finite \(\Lp{p}\) norm. A key operation when comparing measures is the pushforward. Given a measurable map \(T : \Omega \to \cZ\) and a measure \(\mu \in \cP(\Omega)\), the pushforward measure \(T_\# \mu \in \cP(\cZ)\) is defined by:
\[
T_\# \mu(A) := \mu(T^{-1}(A)) \qquad \text{for all measurable sets } A \subset \cZ.
\]

\begin{mydef}
For an underlying domain $\Omega$, define the set
\[ \TLp{p} = \lb (\mu,u) \spaceBar \mu \in \cP_p(\Omega), u \in \Lp{p}(\mu) \rb. \]
For $(\mu,u),(\nu,v) \in \TLp{p}$, we define the $\TLp{p}$ distance $\dTLp{p}$ as follows:
\[ \dTLp{p}((\mu,u),(\nu,v)) = \inf_{\pi \in \Pi(\mu,\nu)} \l \int_{\Omega \times \Omega} \vert x-y \vert^p + \vert u(x) - v(y) \vert^p \,\dd\pi(x,y) \r^{\frac{1}{p}} \]
where \(\Pi(\mu, \nu)\) is the set of couplings between \(\mu\) and \(\nu\).
\end{mydef} 

This framework allows us to treat discrete functions—defined on sampled data—as elements of a well-defined metric space and to compare them to their continuum counterparts in a stable way. The topology is closely related to the \(p\)-Wasserstein distance \cite{villani2009,Santambrogio} on the graph of the function.

A useful characterization of convergence in \(\TLp{p}\) is the following \cite[Proposition 3.12]{Trillos3}.

\begin{proposition}
\label{prop:Back:TLp}
Let \((\mu_n, u_n) \in \TLp{p}\) be a sequence and \((\mu, u) \in \TLp{p}\). Assume that $\mu$ is absolutely continuous with respect to the Lebesgue measure. Then the following are equivalent:
\begin{enumerate}
    \item \((\mu_n, u_n) \to (\mu, u)\) in \(\TLp{p}\);
   \item $\mu_n$ converges weakly to $\mu$ and there exists a sequence of transport maps $\{T_n\}_{n=1}^\infty$ with $(T_{n})_\# \mu = \mu_n$ and $\int_\Omega \vert x - T_n(x) \vert \, \dd x \to 0$ such that
    \[
    \int_\Omega \vert u(x) - u(T_n(x))\vert^p \, \dd \mu(x) \to 0;
    \]
\end{enumerate}
\end{proposition}

To apply this result, we rely on the following result (see \cite[Theorem 2]{Trillos} based on~\cite{garcia_trillos_slepcev_2015}) establishing that such transport maps exist for empirical measures constructed from i.i.d. samples.

\begin{theorem}[Existence of transport maps]
\label{thm:transport}
Assume that $\Omega$ is the unit torus $\sfrac{\bbR^d}{\bbZ^d}$, $x_i\iid\mu\in\cP(\Omega)$ where $\mu$ has a density that is bounded above and below by positive constants.
Then, there exists a constant $C > 0$ such that $\bbP$-a.s., there exists a sequence of transport maps $\{T_n:\Omega \mapsto \Omega_n \}_{n=1}^\infty$ from $\mu$ to $\mu_n$ such that:
\begin{equation*}
\label{eq:Back:TLp:LinftyMapsRate}
\begin{cases}
\limsup_{n \to \infty} \frac{n^{1/2} \Vert \Id - T_n \Vert_{\Lp{\infty}} }{\log(n)^{3/4}} \leq C & \text{if } d = 2; \\
\limsup_{n \to \infty} \frac{n^{1/d} \Vert \Id - T_n \Vert_{\Lp{\infty}} }{\log(n)^{1/d}} \leq C &\text{if } d \geq 3.
\end{cases}
\end{equation*}
\end{theorem}

The assumptions required in the above theorem correspond to conditions~\ref{ass:Main:Ass:S2}, \ref{ass:Main:Ass:M1}, \ref{ass:Main:Ass:M2}, and~\ref{ass:Main:Ass:D1} introduced later in the paper. 
Taken together, these results enable a rigorous comparison between discrete functionals defined over sample-based measures and their continuum limits.

\subsection{\texorpdfstring{$\Gamma$}{Gamma}-Convergence of Functionals}

To analyze the asymptotic behavior of our variational formulations, we use \(\Gamma\)-convergence, a notion from the calculus of variations that captures the convergence of minimization problems.

\begin{mydef} 
Let $(Z,d_Z)$ be a metric space and $F_n:Z \to \bbR$ a sequence of functionals.
We say that $F_n$ $\Gamma$-converges to $F$ with respect to $d_Z$ if:
\begin{enumerate}
\item For every $z \in Z$ and every sequence $\{z_n\}$ with $d_Z(z_n,z) \to 0$:
\[ \liminf_{n \to \infty} F_n(z_n) \geq F(z); \]
\item For every $z \in Z$, there exists a sequence $\{z_n\}$ with $d_Z(z_n,z) \to 0$ and
\[ \limsup_{n\to \infty} F_n(z_n) \leq F(z). \]
\end{enumerate}
\end{mydef}

This notion of convergence ensures that the minimizers of \(F_n\) converge (in a suitable sense) to minimizers of \(F\), provided a compactness condition holds.

\begin{mydef} 
We say that a sequence of functionals $F_n:Z \to \bbR$ has the compactness property if the following holds: if $\{n_k\}_{k \in \bbN}$ is an increasing sequence of integers and $\{z_k\}_{k \in \bbN}$ is a bounded sequence in $Z$ for which $\sup_{k\in \bbN} F_{n_k}(z_k) < \infty$, then the closure of $\{z_k\}$ has a convergent subsequence. 
\end{mydef}

\begin{proposition}[Convergence of minimizers]
\label{prop:Back:Gamma:minimizers}
Let $F_n:Z \mapsto [0,\infty]$ be a sequence of functionals which are not identically equal to $\infty$.
Suppose that the functionals satisfy the compactness property and that they $\Gamma$-converge to $F:Z \mapsto [0,\infty]$.
Then
\[ \lim_{n\to \infty} \inf_{z\in Z} F_n(z) = \min_{z \in Z} F(z). \]
Furthermore, the closure of every bounded sequence $\{z_n\}$ for which \begin{equation} \label{eq:Back:Gamma:MinConv}
\lim_{n \to \infty} \left(F_n(z_n) - \inf_{z \in Z} F_n(z) \right) = 0
\end{equation}
has a convergent subsequence and each of its cluster points is a minimizer of $F$.
In particular, if $F$ has a unique minimizer, then any sequence satisfying~\eqref{eq:Back:Gamma:MinConv} converges to the unique minimizer of $F$.
\end{proposition}

In this work, we show that our discrete energies \(\Gamma\)-converge to continuum energies in the \(\TLp{p}\)-topology. This forms the backbone of our theoretical analysis, allowing us to rigorously link discrete regularization schemes to their continuum analogues.

Lastly, the following result shows that $\Gamma$-convergence is stable with respect to continuous perturbations. 

\begin{proposition}[Convergence of minimizers] \label{prop:additivity}
Suppose that $F_n:Z \mapsto [0,\infty]$ $\Gamma$-converge to $F:Z \mapsto [0,\infty]$. Furthermore, assume that $G:Z \mapsto [0,\infty]$ is continuous. Then, $F_n + G$ $\Gamma$-converges to $F + G$. 
\end{proposition}

\section{Main results} \label{sec:main}

In this section, we present our main results as well as the relevant notation and assumptions used for our proofs. 

\subsection{Hypergraphs}

A hypergraph \(G\) is a pair \(G = (V, E)\), where \(V\) denotes the set of vertices and \(E\) is a collection of subsets \(e \subseteq V\), called hyperedges. We say that all vertices within the same hyperedge $e$ are connected and denote the weight of hyperedge $e$ by $w_0(e) \geq 0$ and its degree/size by $\vert e \vert$. We write $V = \{v_i\}_{i=1}^{\vert V \vert}$. 

A special case of hypergraphs is when $\vert e \vert = 2$ for all $e \in E$. In this case, $(V,E)$ is called a graph and every $e$ represents a pairwise relationship between vertices (see Figure \ref{fig:graphHypergraphs}). Graphs can also be weighted and we usually use the representation $G = (V,W)$ where $W \in \mathbb{R}^{\vert V \vert \times \vert V \vert}$ is a symmetric matrix with entries $w_{ij} = w_0(e)$ if $e = \{v_i,v_j\}$. On graphs, we define the (unnormalized) Laplacian $L$ as \[
L = D-W
\] 
where $D$ is the diagonal matrix with entries $d_{ii} = \sum_{j=1}^{\vert V \vert} w_{ij}$.

\begin{figure}[H]
  \centering
  \includegraphics[scale = 0.2]{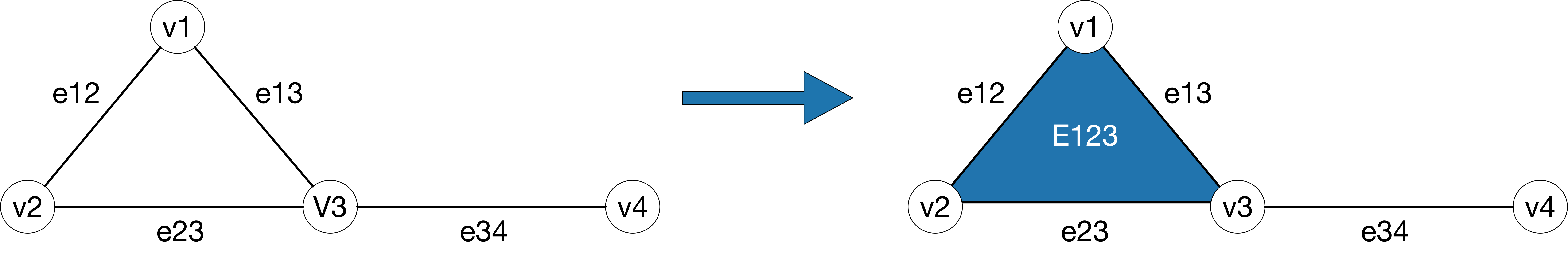}
  \caption{From graphs to hypergraphs (from \cite{weihs2025Hypergraphs}). Left: In the graph, the vertices $v_1$, $v_2$, and $v_3$ are all connected pairwise. Right: A single hyperedge is added connecting all three vertices, transitioning from a graph to a hypergraph representation.} \label{fig:graphHypergraphs}
\end{figure}

We now introduce the hypergraph-to-graph deconstruction that is the foundation of HOHL. Let \((V,E)\) be a hypergraph and define  
\(q = \max_{e \in E} |e| - 1\) as the maximum hyperedge size minus one.  
For each \(k \in \{1,\dots,q\}\), we construct a corresponding skeleton graph  
\(G^{(k)} = (V,E^{(k)})\) with
\[
E^{(k)} = \Bigl\{ \{v_i,v_j\} \,\Big|\, \exists\, e \in E \text{ with } |e| = k+1 \text{ and } \{v_i,v_j\} \subset e \Bigr\},
\]
that is, \(G^{(k)}\) contains all pairwise edges induced by hyperedges of size \(k+1\). We refer to Figure \ref{fig:skeleton} for a visual representation of of the decomposition. Let \(L^{(k)}\) denote the graph Laplacian associated with \(G^{(k)}\).

\begin{figure}[H]
  \centering
  \includegraphics[scale = 0.2]{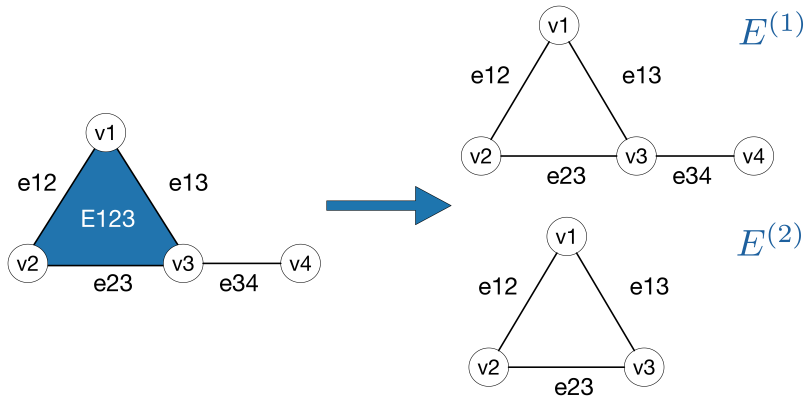}
  \caption{Skeleton graphs with $q = 2$ (from \cite{weihs2025Hypergraphs}).} \label{fig:skeleton}
\end{figure}

\subsection{HOHL} \label{subsec:HOHL}

On graphs, a widely used regularizer is constructed using the graph Laplacian~\cite{LapRef,TutSpec}. For a function \( u : V \to \mathbb{R} \) (which we also identify with a vector in \( \mathbb{R}^{|V|} \)), its first-order smoothness is quantified by
\[
u^\top L u = \frac{1}{2} \sum_{i,j=1}^{|V|} w_{ij} \left(u(v_i) - u(v_j)\right)^2.
\]
Minimizing this expression encourages \( u \) to take similar values on adjacent vertices. On certain graphs, this functional can be interpreted as a discrete analogue of the Sobolev \( \Wkp{1}{2} \) semi-norm, which formalizes the idea of penalizing the first derivative of a function defined on the graph~\cite{Slepcev}. More generally, the regularizer \( v^\top L^s v \), with \( s \in \mathbb{R} \), corresponds to a discrete Sobolev \( \Wkp{s}{2} \) semi-norm and penalizes variations of \( v \) up to order \( s \)~\cite{Stuart,weihs2023consistency}.

The HOHL energy, introduced in~\cite{weihs2025Hypergraphs}, extends graph Laplacian regularization to the hypergraph setting. It is defined as
\begin{equation} \label{eq:discussion:higherOrder}
    u^\top \left[ \sum_{k=1}^q \lambda_k (L^{(k)})^k \right] u =: u^\top \mathcal{L}^{(q)}_{\mathrm{dis}} u,
\end{equation}
for \( u \in \bbR^n \),
where \( 0 < p_1 < \ldots < p_q \) are powers and \( \lambda_1, \ldots, \lambda_q > 0 \) are tuning parameters. In practice, we often set $p_k = k$ for simplicity, although the same reasoning applies to any positive and increasing sequence $\{p_k\}_{k=1}^q$.
This energy imposes a hierarchical, scale-aware regularization: for each skeleton graph \( G^{(k)} \), the corresponding Laplacian power \( (L^{(k)})^{p_k} \) enforces smoothness at a specific scale, with the index \( k \) controlling the granularity of the regularization.

We now discuss the geometric setting, where \( V \subset \mathbb{R}^d \), and the hyperedge set \( E \) is not given à priori. In such cases, it is common to construct \( E \) using geometric principles. The underlying intuition is that a meaningful hyperedge should connect vertices that are close in some metric space. 

In the graph setting, this idea is typically implemented via \(k\)-nearest neighbor ($k$-NN) graphs~\cite{TutSpec} or random geometric graphs~\cite{DBLP:books/ox/P2003}, both of which rely on locality: edges are formed either by linking the \(k\) nearest neighbors or by connecting points within an \(\varepsilon\)-radius neighborhood. Analogous locality-based constructions for hypergraphs have been proposed, e.g., in~\cite{shi2025hypergraphplaplacianequationsdata}; see also~\cite{hgLearningPractice} for a broader discussion. A notable instance is also the random geometric hypergraph model introduced in~\cite{weihs2025Hypergraphs}.

As established in~\cite{weihs2025Hypergraphs}, the hierarchical, scale-aware regularization principle underlying HOHL admits an effective surrogate in geometric settings via a multiscale graph construction, as proposed in~\cite{Merkurjev}. In what follows, we introduce this alternative formulation.

Let \(\Omega_n = \{x_i\}_{i=1}^n \subset \Omega \subset \mathbb{R}^d\) be a set of \(n\) feature vectors, where we assume that \(x_i \overset{\text{i.i.d.}}{\sim} \mu \in \mathcal{P}(\Omega)\). We adopt the same probabilistic framework as in~\cite{weihs2023consistency}. Specifically, we consider a probability space \((\Omega, \mathbb{P})\) whose elements are infinite sequences \(\{x_i\}_{i=1}^\infty\). Our results are stated in terms of the measure \(\mathbb{P}\), establishing that the desired properties hold on a high-probability subset \(\mathcal{X} \subset \Omega\) consisting of such sequences. For a set $E$, we denote its complement by $E^c$.

 Given a length-scale \(\varepsilon > 0\), and a kernel function \(\eta\), we define the edge weights \(w_{\varepsilon,ij}\) between vertices \(x_i\) and \(x_j\) by  
\begin{equation} \notag 
    w_{\varepsilon,ij} = \eta\!\left(\frac{\lvert x_i - x_j \rvert}{\varepsilon}\right).
\end{equation}
Let \(D_{n,\varepsilon}\) be the diagonal degree matrix with entries  
\(d_{n,\varepsilon,ii} = \sum_{j=1}^n w_{\varepsilon,ij}\), and define the normalizing constant  
\[
\sigma_\eta = \frac{1}{d} \int_{\mathbb{R}^d} \eta(\lvert h \rvert)\lvert h \rvert^2 \, \mathrm{d}h < \infty.
\]
The (unnormalized) graph Laplacian is then given by  
\begin{equation} \notag 
    \Delta_{n,\varepsilon} := \frac{2}{\sigma_\eta n \varepsilon^{d+2}} 
    \bigl(D_{n,\varepsilon} - W_{n,\varepsilon}\bigr).
\end{equation}
We note that this is the rescaled version of $L$, i.e. $\Delta_{n,\eps} = \frac{2}{\sigma_\eta n\eps^{d+2}} L$.
With an abuse of notation, this Laplacian can be interpreted either as a matrix \(\Delta_{n,\varepsilon} \in \mathbb{R}^{n \times n}\) or as an operator  
\(\Delta_{n,\varepsilon} : \Lp{2}(\mu_n) \to \Lp{2}(\mu_n)\), where \(\mu_n = \frac{1}{n} \sum_{i=1}^n \delta_{x_i}\) is the empirical measure.

For functions \(u_n, v_n : \Omega_n \to \mathbb{R}\), we define the \(\Lp{2}(\mu_n)\) inner product by  
\begin{equation} \notag 
    \langle u_n, v_n \rangle_{\Lp{2}(\mu_n)}
    = \frac{1}{n} \sum_{i=1}^n u_n(x_i) v_n(x_i).
\end{equation}
Such functions can be regarded as vectors in \(\mathbb{R}^n\); in what follows, we will use the notation \(u_n\) both for the function \(u_n : \Omega_n \to \mathbb{R}\) and for the associated vector in \(\mathbb{R}^n\).

We denote by \(\{(a_{n,\varepsilon,k}, \phi_{n,\varepsilon,k})\}_{k=1}^n\) the eigenpairs of \(\Delta_{n,\varepsilon}\), where the eigenvalues are ordered nondecreasingly:  
\(0 = a_{n,\varepsilon,1} < a_{n,\varepsilon,2} \leq a_{n,\varepsilon,3} \leq \dots \leq a_{n,\varepsilon,n}\) (with strict inequality between \(a_{n,\varepsilon,1}\) and \(a_{n,\varepsilon,2}\) whenever the graph \((\Omega_n,W_{n,\eps})\) is connected). The corresponding eigenfunctions \(\{\phi_{n,\varepsilon,k}\}_{k=1}^n\) form an orthonormal basis of \(\Lp{2}(\mu_n)\).

Given the Laplacians defined above, the surrogate for HOHL \eqref{eq:discussion:higherOrder} is
\begin{equation} \label{eq:multiscale}
v\top \ls \sum_{k=1}^q \lambda_k \Delta_{n,\eps^{(k)}}^{p_k} \rs v,
\end{equation} 
where \( \varepsilon^{(1)} > \cdots > \varepsilon^{(q)} \), and \( p_k > 0 \) controls the regularity imposed at each scale. We will allow the length-scales to vary with the number of data point, i.e. $\eps^{(k)} = \eps_n^{(k)}$, and in this case, we write $E_n := \{\eps_n^{(k)}\}_{k=1}^q$. The well and ill-posedness of \eqref{eq:multiscale} in semi-supervised learning is precisely characterized in \cite[Theorem 3.5]{weihs2025Hypergraphs} as a function of $\eps_n^{(q)}$. 

We now define the continuum analogues of our discrete Laplacian operators. Let $\Delta_\rho$ be the continuum weighted Laplacian operator defined by
\begin{equation} \notag 
\Delta_\rho u(x) = -\frac{1}{\rho(x)}\Div(\rho^2\nabla u)(x), \, x \in \Omega
\end{equation}
and let $\{(\beta_i,\psi_i)\}_{i=1}^\infty$ be its associated eigenpairs where $\beta_1 = 0 < \beta_2 \leq \beta_3 \leq \hdots$. Here $\rho$ denotes the density of $\mu$ with respect to Lebesgue measure. We note that $\{\psi_i\}_{i=1}^\infty$ form a basis of $\Lp{2}(\mu)$ and we also define 
\begin{equation} \label{eq:Main:Not:Setting:H}
\cH^s(\Omega) = \left\{ h \in \Lp{2}(\mu) \spaceBar \Vert h \Vert_{\cH^{s}(\Omega)}^2 := \sum_{i=1}^\infty \beta_i^{s} \langle h, \psi_i \rangle_{\Lp{2}(\mu)}^2 < +\infty \right\}. 
\end{equation}
The space $\cH^s(\Omega)$ is closely related to the Sobolev space $\Wkp{s}{2}(\Omega)$ \cite[Lemma 17]{Stuart}.

\paragraph{Fully supervised problem with HOHL regularization.}

We now turn our attention to the fully supervised problem, where \eqref{eq:multiscale} is used as a regularizer. Specifically, for some sequence of points $\mathbf{g}_n = \{g_i\}_{i=1}^n$, parameter $\tau >0$ and $v_n:\Omega_n \mapsto \bbR$, we define the fully supervised learning problem
\begin{align*}
\cR_{n,\tau}^{(\mathbf{g}_n)}(v_n) 
&= \frac{1}{n} \sum_{i=1}^n \vert v_n(x_i) - g_i \vert^2 + \tau \sum_{k=1}^q \lambda_k \langle v_n, \Delta_{n,\eps_n^{(k)}}^{p_k} v_n \rangle_{\Lp{2}(\mu_n)}.
\end{align*}
For some $g\in \Ck{0}$ 
and $v:\Omega \mapsto \bbR$, the continuum counterpart to the above is 
\begin{align*}
    \cR_{\infty,\tau}^{(g)}(u)
    &= \int_\Omega \vert v(x) - g(x) \vert^2 \rho(x) \, \dd x + \tau \sum_{k=1}^q \lambda_k\langle v, \Delta_{\rho}^{p_k} v \rangle_{\Lp{2}(\mu)}.
\end{align*}

We are mainly interested in the case of noisy labels, i.e. when for some $g\in \Ck{0}(\Omega)$, we have labels $\mathbf{y}_n = \{y_i\}_{i=1}^n$ where $y_i = g(x_i) + \xi_i$ and $\xi_i \in \bbR$ are independent and identically distributed sub-Gaussian centered noise. Adopting the terminology of~\cite{Cucker}, for $u_{n,\tau}^{(\mathbf{y}_n)}$ and $u_\tau$ the minimizers of $\cR_{n,\tau}^{(\mathbf{y}_n)}$ and $\cR_{\infty,\tau}^{(g)}$ respectively, we regard $u_{n,\tau}^{(\mathbf{y}_n)}$ as an estimator of $g$. We further decompose its error into a variance and bias component, defined as
\[
\|u_{n,\tau}^{(\mathbf{y}_n)} - u_{\tau}|_{\Omega_n}\|_{\Lp{2}(\mu_n)} \qquad \qquad\text{and} \qquad \qquad \|u_{\tau} - g\|_{\Lp{2}(\mu)}
\]
respectively.
Intuitively, the variance term quantifies the fluctuation caused by finite sampling,
while the bias term measures the approximation error introduced by the regularization.

\paragraph{Truncated HOHL energies.}

We also consider the truncated versions of our energies. We define the matrix $\mathcal{L}_n^{(q)} = \sum_{k=1}^q \lambda_k \Delta_{n,\eps_n^{(k)}}^{p_k}$ and its continuum counterpart $\mathcal{L}^{(q)} = \sum_{k=1}^q \lambda_k \Delta_\rho^{p_k}$. In particular, $\cL_n^{(q)}$ is positive semi-definite and we denote its ordered eigenpairs by $\{(\beta_{n,i},\psi_{n,i})\}_{i=1}^n$. Then, \[
\langle v , \cL_n^{(q)} v \rangle_{\Lp{2}(\mu_n)} = \sum_{i=1}^n \beta_{n,i} \langle v , \psi_{n,i} \rangle_{\Lp{2}(\mu_n)}^2
\]
and the truncated energy for some threshold $T \leq n$ is \[
\sum_{i=1}^{T} \beta_{n,i} \langle v , \psi_{n,i} \rangle_{\Lp{2}(\mu_n)}^2.
\]
We define the variational problems \[
(\cS\cJ)_{n,E_n,\Psi,T}^{(q,P)}((\nu,v)) = \begin{cases} \sum_{i=1}^{T} \beta_{n,i} \langle v , \psi_{n,i} \rangle_{\Lp{2}(\mu_n)}^2 + \Psi((\nu,v)) &\text{if $\nu = \mu_n$ and $\langle v, \psi_{n,k} \rangle_{\Lp{2}(\mu_n)} = 0$}\\
&\text{for all $k > T$}, \\ +\infty & \text{else,} \end{cases}
\]
and 
\[
(\cS\cJ)_{\infty,\Psi}^{(q,P)}((\nu,v)) = \begin{cases} \sum_{i=1}^{\infty} \l \sum_{k=1}^{q}\lambda_k \beta_{i}^{p_k}\r \langle v , \psi_{i} \rangle_{\Lp{2}(\mu)}^2 + \Psi((\nu,v)) &\text{if } \nu = \mu, \\ +\infty & \text{else,} \end{cases}
\]
where $\Psi:\TLp{2}(\Omega) \mapsto \bbR$ is a continuous function acting as data-fidelity term (for example $\Psi((\nu,v)) = \int_\Omega \vert v(x) - y(x) \vert^2 \, \dd \nu(x)$ where $y:\Omega \mapsto \bbR$ is Lipschitz continuous). The minimizers of $(\cS\cJ)_{n,E_n,\Psi,T}^{(q,P)}((\nu,v))$ are spanned by the first $T$ eigenvectors $\psi_{n,i}$.

\subsection{Assumptions}

In this section, we list the assumptions used throughout the paper. 

\begin{assumptions}
Assumption on the space.

\begin{enumerate}[label=\textbf{S.\arabic*}]
\item The feature vector space $\Omega$ is the unit torus $\sfrac{\bbR^d}{\bbZ^d}$. \label{ass:Main:Ass:S2} 
\end{enumerate}
\end{assumptions}

\begin{assumptions}
Assumptions on the measure.
\begin{enumerate}[label=\textbf{M.\arabic*}]
\item The measure $\mu$ is a probability measure on $\Omega$. \label{ass:Main:Ass:M1} 
\item There is a continuous Lebesgue density $\rho$ of $\mu$ which is bounded from above and below by strictly positive constants, i.e. $0< \min_{x\in\Omega} \rho(x) \leq \max_{x\in\Omega} \rho(x) < +\infty.$ \label{ass:Main:Ass:M2}
\end{enumerate}
\end{assumptions}

The data consists of feature vectors $\{x_i\}_{i=1}^n$ 
and we make the following assumptions.

\begin{assumptions}
Assumptions on the data.
\begin{enumerate}[label=\textbf{D.\arabic*}]
\item Feature vectors $\Omega_n = \{x_i\}_{i=1}^n$ are iid samples from a measure $\mu$ satisfying \ref{ass:Main:Ass:M1}.
We denote by $\mu_n$ the empirical measure associated to our samples. \label{ass:Main:Ass:D1}
\end{enumerate}
\end{assumptions}

The weight function $\eta$ is assumed to satisfy the following assumptions.

\begin{assumptions}
Assumptions on the weight function or kernel.
\begin{enumerate}[label=\textbf{W.\arabic*}]
\item The function $\eta:[0,\infty) \to [0,\infty)$ is non-increasing, has compact support, is continuous and positive at $x=0$.
\label{ass:Main:Ass:W2}
\item The function $\eta:[0,\infty) \to [0,\infty)$ satisfies $\eta(t) > \frac{1}{2}$ for $t \leq \frac{1}{2}$, $\eta(t) = 0$ for all $t \geq 1$ and is decreasing. 
\label{ass:Main:Ass:W3}
\end{enumerate}
\end{assumptions}

The assumption that \(\eta\) has compact support reflects the practical constraint in most applications: for computational efficiency, one typically limits the interaction range between vertices in the hypergraph.

\subsection{Main results}

\subsubsection{Fully supervised problem with HOHL regularization}

We start by establishing the following rates of convergence between the minimizer $u_{n,\tau}^{(\mathbf{y}_n)}$ of $\cR_{n,\tau}^{(\mathbf{y}_n)}$ and $g$. 
The function $u_{n,\tau}^{(\mathbf{y}_n)}$ is the best regularized approximation of $g$ on the graph given the label noise. 

\begin{theorem}[Rates between discrete minimizers and labelling function] \label{thm:ratesDiscreteLebellingFunction}
Assume that \ref{ass:Main:Ass:S2}, \ref{ass:Main:Ass:M1}, \ref{ass:Main:Ass:M2}, \ref{ass:Main:Ass:D1} and \ref{ass:Main:Ass:W3} hold.  Let $q \geq 1$, $\{\lambda_k\}_{k=1}^q$ be a sequence of positive numbers, $P = \{p_k\}_{k=1}^q \subseteq \bbN$ with $1 \leq p_1 \leq \cdots \leq p_q$ and $ E_n = \{\eps_n^{(k)}\}_{k=1}^q$ with $\eps_n^{(1)} > \cdots > \eps_n^{(q)} >0 $. Furthermore, let $\rho \in \Ck{\infty}$ and assume that $W_{\eps_n^{(k)},ii} = 0$. Let $\xi_i$ be iid, mean zero, sub-Gaussian random variables, $g \in \Ck{\infty}$ and $\mathbf{y}_n = \{y_i\}_{i=1}^n$ with $y_i = g(x_i) + \xi_i$. Then, for all $\alpha > 1$ and $\tau_0$, there exists $\eps_0 > 0$ and $C > c > 0$ such that for all $E_n$ satisfying  \[
    \eps_0 \geq \eps_n^{(1)} \geq \cdots \geq \eps_n^{(q)} \geq C \l \frac{\log(n)}{n} \r^{1/d},
    \]
    and $0 < \tau < \tau_0$, the following holds with probability $1-Cn^{-\alpha} - Cne^{-cn\l \eps_n^{(q)} \r^{d + 4p_q}}$:
    \begin{align}
    \Vert u_{n,\tau}^{(\mathbf{y}_n)} -g\vert_{\Omega_n} \Vert_{\Lp{2}(\mu_n)} &\leq C \ls \sum_{k=1}^q \lambda_k \frac{\l \eps_n^{(1)} \r^{2p_1}}{\l\eps_n^{(k)}\r^{2p_k}} \l \frac{\log(n)}{n\l \eps_n^{(k)} \r^d} \r^{1/2}  + \frac{\l \eps_n^{(1)} \r^{2p_1}}{\tau} + \tau \l 1 + \sum_{k=1}^q \lambda_k \eps_n^{(k)}\r \rs \label{eq:thmRates:finalRates}
    \end{align}
    where $u_{n,\tau}^{(\mathbf{y}_n)}$ is the minimizer of $\cR_{n,\tau}^{(\mathbf{y}_n)}$.
\end{theorem}

We note from Propositions~\ref{prop:ratesDiscreteDiscrete} and~\ref{prop:ratesDiscreteContinuum} that the variance term scales as
\[
 \sum_{k=1}^q \lambda_k 
 \frac{(\varepsilon_n^{(1)})^{2p_1}}{(\varepsilon_n^{(k)})^{2p_k}}
 \Big(\frac{\log(n)}{n(\varepsilon_n^{(k)})^d}\Big)^{1/2}  
 + \frac{(\varepsilon_n^{(1)})^{2p_1}}{\tau}
 + \tau \sum_{k=1}^q \lambda_k \varepsilon_n^{(k)},
\]
while, by the proof of Theorem~\ref{thm:ratesDiscreteLebellingFunction}, the bias term is of order~\(\tau\).
This result highlights the nontrivial interplay between the multiple length-scales 
\(\{\varepsilon_n^{(k)}\}_{k=1}^q\) and the regularization parameter~\(\tau\), which jointly determine the convergence behavior of the discrete minimizers.
Understanding this relationship provides practical guidance for choosing these parameters to achieve an appropriate balance between bias, variance, and computational cost. Finally, the theorem generalizes known rates for graph-based learning: 
when \(q=1\), we recover the convergence rates established in~\cite[Corollary~1.8]{trillos2022rates}, thereby situating our result within and extending the existing theoretical framework on graphs.

\begin{remark}[Optimal regularization parameter \(\tau\)]
For fixed bandwidth parameters \(\varepsilon_n^{(k)}\), the last two terms in 
\eqref{eq:thmRates:finalRates} have the form \(A/\tau + B\tau\), where
\(
A = (\varepsilon_n^{(1)})^{2p_1} \text{ and }
B = 1 + \sum_{k=1}^q \lambda_k \varepsilon_n^{(k)}.
\)
Minimizing this expression over \(\tau>0\) yields the optimal choice
\(
\tau^\star = \frac{(\varepsilon_n^{(1)})^{p_1}}{\sqrt{B}}
\)
and, substituting the latter in \eqref{eq:thmRates:finalRates} gives
\begin{equation*} \label{eq:rem:optimalTau}
\|u_{n,\tau^\star}^{(\mathbf{y}_n)} - g|_{\Omega_n}\|_{L^2(\mu_n)}
\leq
C\left[
\sum_{k=1}^q \lambda_k
\frac{\l\varepsilon_n^{(1)}\r^{2p_1}}{\l\varepsilon_n^{(k)}\r^{2p_k}}
\l\frac{\log(n)}{n \l\varepsilon_n^{(k)}\r^d}\r^{1/2}
+ 2\l\varepsilon_n^{(1)}\r^{p_1}\sqrt{1 + \sum_{k=1}^q \lambda_k \varepsilon_n^{(k)}}
\right].
\end{equation*}
\end{remark}

\subsubsection{Truncated energies}

Next, we show that we can use the truncated version of HOHL in practice. In fact, going beyond heuristics, the below results shows that truncated energies converges to the same continuum energy as the full energy (see \cite[Theorem 3.5]{weihs2025Hypergraphs}) This signifies that for large enough $n$, truncated and full energies will lead to arbitrarily close minimizers. 

\begin{theorem}[Consistency of the truncated sum of Laplacians] \label{thm:truncatedEnergies}
    Assume that \ref{ass:Main:Ass:S2}, \ref{ass:Main:Ass:M1}, \ref{ass:Main:Ass:M2}, \ref{ass:Main:Ass:W2} and \ref{ass:Main:Ass:D1} hold. Let $q \geq 1$, $P = \{p_k\}_{k=1}^q \subseteq \bbR$ with $p_1 \leq \cdots \leq p_q$ and $ E_n = \{\eps_n^{(k)}\}_{k=1}^q$ with $\eps_n^{(1)} > \cdots > \eps_n^{(q)}$. Assume that $\rho \in \Ck{\infty}$ and that $\eps_n^{(q)}$ satisfies \[
    \lim_{n \to \infty} \frac{\log(n)}{n \l\eps_n^{(q)}\r^{d+4p_k}} = 0.
    \]
    Let $R_n \leq n$ be a sequence with $R_n \to \infty$, $\Psi:\TLp{2}(\Omega) \to \bbR$ a continuous function and  
    $(\mu_n,u_n)$ the minimizer of $(\cS\cJ)_{n,E_n,\Psi,R_n}^{(q,P)}$. Then, $\bbP$-a.e., there exists a subsequence $(\mu_{n_k},u_{n_k})$ converging to $(\mu,u)$ in $\TLp{2}(\Omega)$ where $(\mu,u)$ is a minimizer of $(\cS\cJ)_{\infty,\Psi}^{(q,P)}$.
\end{theorem}

We emphasize that the convergence conditions we impose on the truncation are mild: it suffices that the truncation threshold tends to infinity. This grants practitioners considerable flexibility in applying HOHL in practice.

Moreover, we note from Lemma~\ref{lem:eigenpairs} that the limiting continuum energy in Theorem \ref{thm:truncatedEnergies}
\[
\sum_{i=1}^{\infty} \Bigl( \sum_{k=1}^{q} \lambda_k \beta_{i}^{p_k} \Bigr) \langle v , \psi_{i} \rangle_{\Lp{2}(\mu)}^2
\]
is equal to \(\langle v, \mathcal{L}^{(q)} v \rangle_{\Lp{2}(\mu)}\). In particular, this implies that $\mathcal{L}^{(q)}$ could also be defined directly through its spectrum: it shares the same eigenfunctions as $\Delta_\rho$, while its eigenvalues are given by functions of those of $\Delta_\rho$. This places our approach firmly within the framework of spectral kernel learning, where it is common to regularize with operators derived from the Laplacian. Spectral learning—and its analysis through reproducing kernel Hilbert space techniques—has been shown to yield powerful results for uncertainty quantification, enabling explicit bounds on expected error as well as estimates of prediction variance in semi-supervised learning (see \cite{zhang2005spectral} and references therein).

\subsubsection{Non-geometrical setting} \label{sec:main:nongeometric}

All of the preceding results focused on applying HOHL within the geometric setting. The following result extends the analysis to arbitrary hypergraphs, demonstrating that the matrix $\mathcal{L}^{(q)}_{\mathrm{dis}}$ can be interpreted as the Laplacian of a specially constructed graph (which may be signed \cite{signedGraphs}).

\begin{proposition}\label{prop:laplacian}
    There exists a graph $\tilde{G}$ whose Laplacian matrix is given by $\mathcal{L}^{(q)}_{\mathrm{dis}}$. Furthermore, $\mathcal{L}^{(q)}_{\mathrm{dis}}$ is positive semi-definite and symmetric, and \eqref{eq:discussion:higherOrder} is a quadratic form.
\end{proposition}

This result is particularly noteworthy as it implies that standard numerical techniques developed for Laplace learning are directly applicable to HOHL. These include spectral truncation (see Theorem~\ref{thm:truncatedEnergies}), Nyström extensions~\cite{fowlkes2004spectral}, conjugate gradient methods for Laplacian inversion, and more. Moreover, it suggests that HOHL can function as a drop-in replacement for Laplace learning within existing machine learning pipelines. To demonstrate this in practice, Section~\ref{sec:numerical} presents active learning experiments where the HOHL matrix $\mathcal{L}^{(q)}_{n}$ defines a Gaussian prior over functions.

We can extend the HOHL energy~\eqref{eq:discussion:higherOrder} to non-geometric datasets, where geometric embeddings for the vertices are unavailable and, for example, the weight models described in Section~\ref{subsec:HOHL} do not apply.  
In such settings, the standard feature-based hypergraph construction, e.g.~\cite{scholkopfHyper2006,TVHg}, forms a hyperedge among all nodes that share a common categorical feature value. Each hyperedge is also assigned unit weight. 

Unlike previous methods that rely on global hyperedge smoothing or iterative optimization, our approach introduces scalable, structure-aware regularization tailored to categorical feature data. Crucially, in contrast to the geometric setting, hyperedge size here does not reflect sample proximity but rather the frequency of shared attribute values. Large hyperedges correspond to common features and tend to encode coarse relationships, while small hyperedges capture more specific, and potentially more informative, structure. Promoting regularity over these smaller subsets is thus useful for fine-grained label propagation. This represents the inverse perspective of the geometric setting, where larger hyperedges encode finer local interactions. 
We summarize the main differences of HOHL in the geometric and non-geometric setting in Table \ref{tab:geometric_vs_nongeometric}.

In real datasets however, even small hyperedges can contain many nodes, and large ones are common. This poses computational challenges for HOHL, which penalizes through powers of Laplacians on skeleton graphs. To address this, Algorithm~\ref{alg:hohl} groups hyperedges by size and aggregates their skeleton graphs into a fixed number of levels. This reduces computational cost and imposes a multiscale hierarchy that prioritizes structurally meaningful interactions. In Section \ref{sec:numerical}, we demonstrate that HOHL outperforms many other hypergraph methods in semi-supervised learning. 

\begin{table}
\centering
\small
\renewcommand{\arraystretch}{1.5}
\setlength{\tabcolsep}{8pt}
\begin{tabularx}{\linewidth}{>{\bfseries}p{5cm} Y Y}
\toprule
\textbf{Aspect} & \textbf{Geometric Setting} & \textbf{Non-Geometric Setting} \\
\midrule
Vertex set $V$ & $\Omega_n = \{x_i\}_{i=1}^n \subset \mathbb{R}^d$ & Arbitrary object set (no embedding in $\mathbb{R}^d$) \\
\midrule
Hyperedge construction & Based on distance/proximity (e.g., $\varepsilon$-neighborhoods) & Based on shared attributes or features \\
Interpretation of hyperedge size & 
Smaller hyperedges correspond to longer-range geometric connections; larger hyperedges capture denser local neighborhoods & 
Smaller hyperedges reflect more specific or rare attributes; larger hyperedges correspond to common, broad features \\
Use of length scales $\varepsilon$ & Essential for defining Laplacians $\Delta_{n,\varepsilon}$ & Not applicable \\
\midrule
HOHL regularization & Higher regularization on large hyperedges & Higher regularization on small hyperedges\\
\midrule
Continuum limit of HOHL & $\Wkp{p_q}{2}$ semi-norm  \cite{weihs2025Hypergraphs} & No natural continuum limit \\
Characterization of well/ill-posedness of HOHL in SSL & \checkmark (\cite[Theorem 3.5]{weihs2025Hypergraphs}) & — \\
Rates of convergence for HOHL regularizer & \checkmark (Theorem \ref{thm:ratesDiscreteLebellingFunction}) & —\\
\midrule
Use of spectral truncation & \checkmark (consistency in Theorem \ref{thm:truncatedEnergies}) & \checkmark \\
HOHL is quadratic form & \checkmark & \checkmark \\
\bottomrule
\end{tabularx}
\caption{Comparison of HOHL in geometric and non-geometric settings.}
\label{tab:geometric_vs_nongeometric}
\end{table}

\begin{algorithm}
\noindent\textbf{Input:} Hypergraph \( G = (V, E) \); number of skeleton graphs \( q \)\par
\noindent\textbf{Output:} List of Laplacian matrices $\{L^{(k)}\}_{k=1}^q$ to be used in \eqref{eq:discussion:higherOrder} \par 
\vspace{0.5em}
\begin{algorithmic}[1]
\State Group hyperedges by size: \( A[j] \gets \{ e \in E : |e| = j \} \)
\State Let \( \mathrm{Ord} \gets \) sorted list of unique hyperedge sizes (descending)
\State Initialize adjacency matrix list: \( \mathrm{Adj} \gets [\ ] \)
\For{each \( j \in \mathrm{Ord} \)}
  \State Construct skeleton graph from \( A[j] \) and append its adjacency matrix to \( \mathrm{Adj} \)
\EndFor
\State Define uniform thresholds to split \( \mathrm{Adj} \) into \( q \) segments and store them in the list Thresholds
\For{each \( k = 1 \) to \( q \)}
  \State Let \( \text{start}_k \gets \mathrm{Thresholds}[k - 1] \) \Comment{First index of segment \( k \)}
  \State Let \( \text{end}_k \gets \mathrm{Thresholds}[k] \) \Comment{One past the last index of segment \( k \)}
  \State Set \( W_n^{(k)} \gets 0 \)
  \For{each \( m = \text{start}_k \) to \( \text{end}_k - 1 \)}
    \State \( W^{(k)} \gets W^{(k)} + \mathrm{Adj}[m] \)
  \EndFor
  \State Compute Laplacian \( L^{(k)} \) from \( W_n^{(k)} \)
\EndFor
\State \Return $\{L^{(k)}\}_{k=1}^q$
\end{algorithmic}
\caption{Construction of multiscale Laplacians for HOHL. Hyperedges are grouped by size, skeletons are aggregated into $q$ segments, and Laplacians $\{L^{(k)}\}$ are computed for use in~\eqref{eq:discussion:higherOrder}.}
\label{alg:hohl}
\end{algorithm}

\section{Proofs} 

In this section, we present the proofs of our results.

\subsection{Fully supervised problem with HOHL regularization}

For this section only, we proceed to a constant re-scaling of the Laplacians in Section \ref{subsec:HOHL}. In particular, we define:
\[
\Delta_{n,\varepsilon} = \frac{2}{n \varepsilon^{d+2}} 
    \bigl(D_{n,\varepsilon} - W_{n,\varepsilon}\bigr) \quad \text{and} \quad \Delta_\rho u(x) = -\frac{\sigma_\eta}{\rho(x)}\Div(\rho^2\nabla u)(x).
\]
We also recall that $E^c$ denotes the complement of the set $E$.

First, the aim is to show the analogue of \cite[Proposition 2.1]{trillos2022rates} and to this purpose, we define \begin{equation} \notag 
    w_n = \l \Id + \tau \sum_{k=1}^q \lambda_k   \Delta_{n,\eps_n^{(k)}}^{p_k} \r^{-1} \bm{\xi}_n
\end{equation}
as well as
\begin{equation} \label{eq:proofs:wnTilde}
    \tilde{w}_n = \l \Id + \tau \sum_{k=1}^q \lambda_k \l \frac{2}{n\l\eps_n^{(k)}\r^2} D_{n,\eps_n^{(k)}}\r^{p_k} \r^{-1} \bm{\xi}_n
\end{equation}
where $\bm{\xi}_n = (\xi_1,\dots,\xi_n)$ and $D_{n,\eps_n^{(k)}}$ is the diagonal degree matrix defined in Section \ref{subsec:HOHL}.  

\begin{lemma}[Bound on matrix product] \label{lem:proofs:matrixProduct}
    Assume that \ref{ass:Main:Ass:S2}, \ref{ass:Main:Ass:M1}, \ref{ass:Main:Ass:M2}, \ref{ass:Main:Ass:D1} and \ref{ass:Main:Ass:W3} hold. Furthermore, let $\rho \in \Ck{\infty}$ and assume that $W_{ii} = 0$. Let $\ell \in \bbN$, $q \geq 1$, $1 \leq k \leq q$, $\{\lambda_r\}_{r=1}^q$ be a sequence of positive numbers, $P = \{p_r\}_{r=1}^q \subseteq \bbN$ with $1 \leq p_1 \leq \cdots \leq p_q$ and $ E_n = \{\eps_n^{(r)}\}_{r=1}^q$ with $\eps_n^{(1)} > \cdots > \eps_n^{(q)} >0 $. Let $\xi_i$ be iid, mean zero, sub-Gaussian random variables and $\tilde{w}_n$ be defined in \eqref{eq:proofs:wnTilde}. Then, for $\alpha > 1$, $\tau >0$  and $\eps_n^{(q)}$ satisfying \[
    \eps_n^{(q)} \geq C \l \frac{\log(n)}{n} \r^{1/d},
    \] 
    there exists $C > 0$ such that
    \begin{equation} \label{eq:matrixProduct:matrixProduct}
   \left\Vert W_{n,\eps_n^{(k)}}D_{n,\eps_n^{(k)}}^{\ell-1}\tilde{w}_n \right\Vert_{\Lp{2}(\mu_n)} \leq  \frac{C n^\ell\l \eps_n^{(1)} \r^{2p_1}}{\tau} \l \frac{\log(n)}{n\l \eps_n^{(k)} \r^d} \r^{1/2}
\end{equation}
with probability $1 - C n^{-\alpha}$.

\end{lemma}

\begin{proof}
In the proof $C>0$ ($c >0$) will denote a constant that can be arbitrarily large (small), is independent of $n$, and that may change from line to line.

For notational convenience, we define $d_{n,i,\eps_n^{(r)}} =  \sum_{j=1}^n \l W_{n,\eps_n^{(r)}} \r_{ij}$. 
For $1 \leq r \leq q$, we let $E_r$ be the event where the graph $G_n$ satisfies the following inequalities 
\begin{itemize}
        \item there exists constants $C_1$ and $C_2$ such that 
        \begin{equation} \label{eq:matrixProduct:degreeBound}
        C_1 \leq n^{-1} d_{n,i,\eps_n^{(r)}} \leq C_2
        \end{equation}
        for all $1 \leq i \leq n$;
        \item $\#\{j \, \vert \, \l W_{n,\eps_n^{(r)}} \r_{ij} > 0 \} \leq C n \l \eps_n^{(r)} \r^d $ for $1 \leq i \leq n$.
    \end{itemize}  
Let $E = \cap_{r=1}^q E_r$ be the set of events such that the above inequalities hold for all $\leq r\leq q$.
By~\cite[Lemma 2.2]{trillos2022rates}, we know that $\bbP(E_r) \geq 1 - 2ne^{-c(r)n\l\eps_n^{(r)}\r^d}$. 
Hence, 
\[
\bbP\l E^c\r \leq 
\sum_{r=1}^q \bbP \l E_r^c \r \leq \sum_{r=1}^q 2n e^{-c(r)n\l\eps_n^{(r)}\r^d} \leq C n e^{-c n\l\eps_n^{(q)}\r^d}
\]
implying that \begin{equation} \notag 
    \bbP(E) \geq 1 - Cne^{-c n\l\eps_n^{(q)}\r^d}. 
\end{equation}

Now, let $G_n$ be a graph in the event $E$ and fix $1\leq i \leq n$. For $1 \leq j \leq n$, let 
\[
q_j^{i} = \frac{\tau \l W_{n,\eps_n^{(k)}} \r_{ij} \l d_{n,j,\eps_n^{(k)}} \r^{\ell-1} \xi_j }{ 1 + \tau \sum_{r=1}^{q} \lambda_r \l \frac{2}{n \l \eps_n^{(r)} \r^2} d_{n,j,\eps_n^{(r)}} \r^{p_r}  }
\]
and we note that 
\begin{equation} \label{eq:matrixProduct:qj}
    \tau\l W_{n,\eps_n^{(k)}}D_{n,\eps_n^{(k)}}^{\ell-1}\tilde{w}_n \r_i = \sum_{j=1}^n q_j.
\end{equation}
Now, $q_j^{i}$ are centered and independent random variables. Furthermore, we estimate as follows: \begin{align}
   \frac{1}{n^{\ell-1} \l \eps_n^{(1)} \r^{2 p_1}} \vert q_j^i \vert &= \tau\vert \xi_j \vert  \l W_{n,\eps_n^{(k)}} \r_{ij} \frac{\l d_{n,j,\eps_n^{(k)}} \r^{\ell-1}}{n^{\ell-1}} \frac{1}{\l \eps_n^{(1)} \r^{2 p_1}} \frac{1}{ 1 + \tau \sum_{r=1}^{q} \lambda_r \l \frac{2}{n \l \eps_n^{(r)} \r^2} d_{n,j,\eps_n^{(r)}} \r^{p_r}  } \notag \\
   &\leq \frac{C \vert \xi_j \vert}{\l \eps_n^{(k)} \r^d} \frac{1}{\l \eps_n^{(1)} \r^{2 p_1}} \frac{1}{ \sum_{r=1}^{q} \lambda_r \frac{1}{\l \eps_n^{(r)} \r^{2p_r}} } \label{eq:matrixProduct:estimate1} \\
   &\leq \frac{C \vert \xi_j \vert}{\l \eps_n^{(k)} \r^d} \label{eq:matrixProduct:estimate2}
\end{align}
where we used the fact that $\l W_{n,\eps_n^{(k)}} \r_{ij} \leq C \l \eps_n^{(k)} \r^{-d}$ and \eqref{eq:matrixProduct:degreeBound} for \eqref{eq:matrixProduct:estimate1} and the fact that $p_1 \leq \cdots \leq p_q$ and $\eps_n^{(1)} > \cdots > \eps_n^{(q)}$ for \eqref{eq:matrixProduct:estimate2}. This implies that $[n^{\ell-1} (\eps_n^{(1)})]^{-1}q_j^i$ are sub-Gaussian and satisfy the same inequalities in the Birnbaum-Orlicz norm as in \cite[Lemma 2.4]{trillos2022rates}. By applying the same Hoeffding inequality as in the latter, for any $t > 0$, we therefore obtain \[
\bbP\l \frac{1}{n^{\ell-1}\l \eps_n^{(1)} \r^{2 p_1}} \left| \sum_{j=1}^n q_j^i\right| > t  \spaceBar E \r \leq 2e^{-ct^2 \l \eps_n^{(k)} \r^{d}/n }.
\] 
We then choose $t = \lambda \sqrt{\frac{n\log(n)}{\l \eps_n^{(k)} \r^{d}}}$ so that, using \eqref{eq:matrixProduct:qj}, \begin{equation} \label{eq:matrixProduct:estimate3}
    \frac{\tau}{n^{\ell-1} \l \eps_n^{(1)} \r^{2 p_1}} \left| \l W_{n,\eps_n^{(k)}}D_{n,\eps_n^{(k)}}^{\ell-1}\tilde{w}_n \r_i \right| = \frac{\tau}{n^{\ell-1} \l \eps_n^{(1)} \r^{2 p_1} } \left| \sum_{j=1}^n q_j^i\right| \leq \lambda \sqrt{\frac{n\log(n)}{\l \eps_n^{(k)} \r^{d}}} 
\end{equation}
with probability at least $1-2n^{-c\lambda^2}$ conditioned on $E$. We pick $\lambda = \sqrt{\frac{\alpha +1}{c}}$ and, through an union bound, obtain that \eqref{eq:matrixProduct:estimate3} holds for all $1 \leq i \leq n$ with probability at least $1-2n^{1-c\lambda^2} = 1-2n^{-\alpha}$, conditioned on $E$. Starting from \eqref{eq:matrixProduct:estimate3}, we get
\begin{equation} \label{eq:matrixProduct:final}
   \left\Vert W_{n,\eps_n^{(k)}}D_{n,\eps_n^{(k)}}^{\ell-1}\tilde{w}_n \right\Vert_{\Lp{2}(\mu_n)} \leq  \left\Vert W_{n,\eps_n^{(k)}}D_{n,\eps_n^{(k)}}^{\ell-1}\tilde{w}_n \right\Vert_{\Lp{\infty}(\mu_n)} \leq \frac{C n^\ell\l \eps_n^{(1)} \r^{2p_1}}{\tau} \l \frac{\log(n)}{n\l \eps_n^{(k)} \r^d} \r^{1/2} 
\end{equation}
conditioned on $E$ with probability at least $ 1 - Cn^{-\alpha}$. Let $A$ be the event such that \eqref{eq:matrixProduct:matrixProduct} holds. By \eqref{eq:matrixProduct:final}, \[
\bbP(A) = \bbP(A \, \vert \, E) \bbP(E) + \bbP(A \, \vert \, E^c) \bbP(E^c) \geq (1 - Cn^{-\alpha}) \cdot \l 1 - Cne^{-c n\l\eps_n^{(q)}\r^d} \r
\]
and, to conclude, we can pick $C$ large enough so that $\bbP(A) \geq 1 - Cn^{-\alpha}$. 
\end{proof}

\begin{lemma}[Bounds on $\tilde{w}_n$] 
    Assume that \ref{ass:Main:Ass:S2}, \ref{ass:Main:Ass:M1}, \ref{ass:Main:Ass:M2}, \ref{ass:Main:Ass:D1} and \ref{ass:Main:Ass:W3} hold. Furthermore, let $\rho \in \Ck{\infty}$ and assume that $W_{ii} = 0$.  Let $q \geq 1$, $\{\lambda_k\}_{k=1}^q$ be a sequence of positive numbers, $P = \{p_k\}_{k=1}^q \subseteq \bbN$ with $1 \leq p_1 \leq \cdots \leq p_q$ and $ E_n = \{\eps_n^{(k)}\}_{k=1}^q$ with $\eps_n^{(1)} > \cdots > \eps_n^{(q)} >0 $. Let $\xi_i$ be iid, mean zero, sub-Gaussian random variables and $\tilde{w}_n$ be defined in \eqref{eq:proofs:wnTilde}. Then, for all $\alpha > 1$, there exists $\eps_0 > 0$ and $C > 0$ such that for all $E_n$ satisfying  \[
    \eps_0 \geq \eps_n^{(1)} \geq \cdots \geq \eps_n^{(q)} \geq C \l \frac{\log(n)}{n} \r^{1/d},
    \]
    and $\tau > 0$, the following holds with probability $1-Cn^{-\alpha}$: \begin{enumerate}
        \item \begin{equation} \label{eq:differenceWn:boundGradient}
        \left\Vert \frac{1}{2} \nabla \cR_{n,\tau}^{(\bm{\xi}_n)}(\tilde{w}_n) \right\Vert_{\Lp{2}(\mu_n)} \leq C \sum_{k=1}^q \lambda_k \frac{\l \eps_n^{(1)} \r^{2p_1}}{\l\eps_n^{(k)}\r^{2p_k}} \l \frac{\log(n)}{n\l \eps_n^{(k)} \r^d} \r^{1/2};
        \end{equation}
        \item \begin{equation} \label{eq:differenceWn:boundWnTilde}
         \Vert \tilde{w}_n \Vert_{\Lp{2}(\mu_n)} \leq \frac{C}{\tau} \l \eps_n^{(1)} \r^{2p_1}.
        \end{equation}
    \end{enumerate}
\end{lemma}

\begin{proof}
    In the proof $C>0$ will denote a constant that can be arbitrarily large, is independent of $n$, and that may change from line to line. Let $\Vert \cdot \Vert_{\mathrm{op}}$ denote the operator norm.

We start by noting that \begin{equation} \label{eq:differenceWn:gradientEnergy}
    \frac{1}{2} \nabla \cR_{n,\tau}^{(\mathbf{a}_n)}(v_n) = v_n - \mathbf{a}_n + \tau \sum_{k=1}^q \lambda_k \Delta_{n,\eps_n^{(k)}}^{p_k} v_n = \l \Id + \tau \sum_{k=1}^q \lambda_k \Delta_{n,\eps_n^{(k)}}^{p_k} \r v_n - \mathbf{a}_n.
\end{equation}
In particular, this implies that, with probability at least $1 - Cn^{-\alpha}$ (see below), we can estimate as follows: 
\begin{align}
    &\left\Vert \frac{1}{2} \nabla \cR_{n,\tau}^{(\bm{\xi}_n)}(\tilde{w}_n) \right\Vert_{\Lp{2}(\mu_n)} = \left\Vert \l \Id + \tau \sum_{k=1}^q \lambda_k \Delta_{n,\eps_n^{(k)}}^{p_k} \r  \tilde{w}_n - \bm{\xi}_n \right\Vert_{\Lp{2}(\mu_n)}\label{eq:differenceWn:estimate1} \\
    &= \left\Vert \l \Id + \tau \sum_{k=1}^q \lambda_k \Delta_{n,\eps_n^{(k)}}^{p_k} \r  \tilde{w}_n - \l \Id + \tau \sum_{k=1}^q \lambda_k \l \frac{2}{n\l\eps_n^{(k)}\r^2} D_{n,\eps_n^{(k)}}\r^{p_k} \r \tilde{w}_n \right\Vert_{\Lp{2}(\mu_n)} \label{eq:differenceWn:estimate2} \\
    &\leq C\tau \sum_{k=1}^q  \frac{\lambda_k}{n^{p_k}\l\eps_n^{(k)}\r^{2p_k}} \left\Vert \left[ \l D_{n,\eps_n^{(k)}} - W_{n,\eps_n^{(k)}} \r^{p_k} - D_{n,\eps_n^{(k)}}^{p_k} \right] \tilde{w}_n \right\Vert_{\Lp{2}(\mu_n)} \notag \\
    &= C\tau \sum_{k=1}^q  \frac{\lambda_k}{n^{p_k}\l\eps_n^{(k)}\r^{2p_k}} \left\Vert \left[ \l \sum_{\chi \in \{0,1\}^{p_k}} \prod^{{p_k}}_{i=1}  D_{n,\eps_n^{(k)}}^{\chi_i} (- W_{n,\eps_n^{(k)}})^{1-\chi_i} \r - D_{n,\eps_n^{(k)}}^{p_k} \right] \tilde{w}_n \right\Vert_{\Lp{2}(\mu_n)} \label{eq:differenceWn:estimateProduct}
\end{align}
where we used \eqref{eq:differenceWn:gradientEnergy} in \eqref{eq:differenceWn:estimate1}, \eqref{eq:proofs:wnTilde} in \eqref{eq:differenceWn:estimate2}, and the expansion
\[
(D_{n,\eps_n^{(k)}} - W_{n,\eps_n^{(k)}})^{p_k}
= \sum_{\chi \in \{0,1\}^{p_k}}
\prod_{i=1}^{p_k} D_{n,\eps_n^{(k)}}^{\chi_i}
(-W_{n,\eps_n^{(k)}})^{1-\chi_i}
\]
for \eqref{eq:differenceWn:estimateProduct}. Subtracting the term $D_{n,\eps_n^{(k)}}^{p_k}$ from $\l \sum_{\chi \in \{0,1\}^{p_k}} \prod^{{p_k}}_{i=1}  D_{n,\eps_n^{(k)}}^{\chi_i} (- W_{n,\eps_n^{(k)}})^{1-\chi_i} \r$ removes the summand associated with
$\chi = (1,1,\dots,1)$, so that every remaining product in the sum contains at least
one factor of $W_{n,\eps_n^{(k)}}$ and 
\[
\Big(\sum_{\chi\in\{0,1\}^{p_k}}
\prod_{i=1}^{p_k} D_{n,\eps_n^{(k)}}^{\chi_i}
(-W_{n,\eps_n^{(k)}})^{1-\chi_i}\Big)- D_{n,\eps_n^{(k)}}^{p_k}= \sum_{\substack{\chi\in\{0,1\}^{p_k}\\ \chi\neq(1,\dots,1)}}
\prod_{i=1}^{p_k} D_{n,\eps_n^{(k)}}^{\chi_i}
(-W_{n,\eps_n^{(k)}})^{1-\chi_i}
\]
For any fixed $\chi \in \{0,1\}^{p_k}$ with $\chi \neq (1,1,\dots,1)$, let $r_\chi$ denote the index of the first occurrence of
$W_{n,\eps_n^{(k)}}$ when reading the product from right to left (the index exists since all terms with $\chi \neq (1,1,\dots,1)$ contain at least one factor of $W_{n,\eps_n^{(k)}}$). We can then
factor the product as
\[
\prod_{i=1}^{p_k} D_{n,\eps_n^{(k)}}^{\chi_i}
(-W_{n,\eps_n^{(k)}})^{1-\chi_i}
=
\underbrace{\l
\prod_{i=1}^{p_k - r_\chi}
D_{n,\eps_n^{(k)}}^{\chi_i}
(-W_{n,\eps_n^{(k)}})^{1-\chi_i}
\r}_{=:T_{-r_\chi}}
(-W_{n,\eps_n^{(k)}})
D_{n,\eps_n^{(k)}}^{r_{\chi}-1}.
\]
The term $T_{-r_\chi}$ contains $p_k - r_\chi$ factors, each equal to
either $D_{n,\eps_n^{(k)}}$ or $W_{n,\eps_n^{(k)}}$.  Using the operator-norm bounds
from \cite[Lemma 2.3]{trillos2022rates}, we have
$\Vert T_{-r_\chi} \Vert_{\mathrm{op}}\le (C n)^{p_k-r_\chi}$. This implies that \begin{align}
    &\left\Vert \left[ \l \sum_{\chi \in \{0,1\}^{p_k}} \prod^{{p_k}}_{i=1}  D_{n,\eps_n^{(k)}}^{\chi_i} (- W_{n,\eps_n^{(k)}})^{1-\chi_i} \r - D_{n,\eps_n^{(k)}}^{p_k} \right]\tilde{w}_n \right\Vert_{\Lp{2}(\mu_n)} \notag \\
    &\qquad \qquad = \left\Vert \ls \sum_{\substack{\chi\in\{0,1\}^{p_k}\\ \chi\neq(1,\dots,1)}}
\prod_{i=1}^{p_k} D_{n,\eps_n^{(k)}}^{\chi_i}
(-W_{n,\eps_n^{(k)}})^{1-\chi_i}  \rs \tilde{w}_n \right\Vert_{\Lp{2}(\mu_n)} \notag \\ 
    &\qquad \qquad \leq   \sum_{\substack{\chi\in\{0,1\}^{p_k}\\ \chi\neq(1,\dots,1)}}
\left\Vert \prod_{i=1}^{p_k} D_{n,\eps_n^{(k)}}^{\chi_i}
(-W_{n,\eps_n^{(k)}})^{1-\chi_i}  \tilde{w}_n \right\Vert_{\Lp{2}(\mu_n)} \notag \\ 
&\qquad \qquad\leq \sum_{\substack{\chi\in\{0,1\}^{p_k}\\ \chi\neq(1,\dots,1)}} \Vert T_{-r_\chi} \Vert_{\mathrm{op}}
\left\Vert W_{n,\eps_n^{(k)}} D_{n,\eps_n^{(k)}}^{r_\chi-1} \tilde{w}_n \right\Vert_{\Lp{2}(\mu_n)} \notag \\ 
&\qquad \qquad\leq C  \sum_{\substack{\chi\in\{0,1\}^{p_k}\\ \chi\neq(1,\dots,1)}} n^{p_k-r_\chi} \frac{n^{r_\chi}\l \eps_n^{(1)} \r^{2p_1}}{\tau} \l \frac{\log(n)}{n\l \eps_n^{(k)} \r^d} \r^{1/2} \label{eq:differenceWn:estimateProduct2} \\
&\qquad \qquad= C (2^{p_k}-1) \frac{n^{p_k}\l \eps_n^{(1)} \r^{2p_1}}{\tau} \l \frac{\log(n)}{n\l \eps_n^{(k)} \r^d} \r^{1/2} \label{eq:differenceWn:estimateProduct3}
\end{align} 
where we used Lemma \ref{lem:proofs:matrixProduct} for \eqref{eq:differenceWn:estimateProduct2}. Inserting \eqref{eq:differenceWn:estimateProduct3} into \eqref{eq:differenceWn:estimateProduct}, we obtain \begin{align}
    \left\Vert \frac{1}{2} \nabla \cR_{n,\tau}^{(\bm{\xi}_n)}(\tilde{w}_n) \right\Vert_{\Lp{2}(\mu_n)} &\leq C\tau \sum_{k=1}^q  \frac{\lambda_k}{n^{p_k}\l\eps_n^{(k)}\r^{2p_k}} \frac{n^{p_k}\l \eps_n^{(1)} \r^{2p_1}}{\tau} \l \frac{\log(n)}{n\l \eps_n^{(k)} \r^d} \r^{1/2} \notag \\
    &= C \sum_{k=1}^q  \frac{\lambda_k\l \eps_n^{(1)} \r^{2p_1}}{\l\eps_n^{(k)}\r^{2p_k}} \l \frac{\log(n)}{n\l \eps_n^{(k)} \r^d} \r^{1/2}.
\end{align}

For the second claim of the lemma, let us start by assuming that $G_n$ is a graph in the event $E$ from the proof of Lemma \ref{lem:proofs:matrixProduct}. Then, 
\begin{align}
    \Vert \tilde{w}_n \Vert_{\Lp{2}(\mu_n)}^2 &= \frac{1}{n} \sum_{i=1}^n \frac{\xi_i^2}{\l 1 + \tau \sum_{r=1}^{q} \lambda_r \l \frac{2}{n \l \eps_n^{(r)} \r^2} d_{n,i,\eps_n^{(r)}} \r^{p_r} \r^2} \notag \\
    &\leq \frac{C}{n}  \sum_{i=1}^n \frac{\xi_i^2}{\l \tau \sum_{r=1}^{q} \frac{\lambda_r}{\l \eps_n^{(r)} \r^{2p_r}} \r^2} \label{eq:differenceWn:estimate5} \\
    &\leq \frac{C}{n\tau^2} \l \eps_n^{(1)} \r^{4p_1} \sum_{i=1}^n \xi_i^2 \label{eq:differenceWn:estimate6}
\end{align}
where we used the fact that there exists $C_1 \leq n^{-1} d_{n,i,\eps_n^{(r)}} \leq C_2$ for all $1 \leq i \leq n$ and $1 \leq r \leq q$ for \eqref{eq:differenceWn:estimate5} and the fact that $p_1 \leq \cdots \leq p_q$ and $\eps_n^{(1)} > \cdots > \eps_n^{(q)}$ for \eqref{eq:differenceWn:estimate6}. Let $A$ be the event such that \eqref{eq:differenceWn:boundWnTilde} holds. Then, arguing as in \cite[Lemma 2.6]{Trillos3}, we can show that $\bbP(A \, \vert \, E) \geq 1 - Cn^{-\alpha}$. Analogously to the proof of Lemma \ref{lem:proofs:matrixProduct}, we conclude that $\bbP(A) \geq 1 - Cn^{-\alpha}$. 
\end{proof}

\begin{proposition}[Rates between discrete noisy and noiseless minimizers] \label{prop:ratesDiscreteDiscrete}
     Assume that \ref{ass:Main:Ass:S2}, \ref{ass:Main:Ass:M1}, \ref{ass:Main:Ass:M2}, \ref{ass:Main:Ass:D1} and \ref{ass:Main:Ass:W3} hold. Furthermore, let $\rho \in \Ck{\infty}$ and assume that $W_{ii} = 0$.  Let $q \geq 1$, $\{\lambda_k\}_{k=1}^q$ be a sequence of positive numbers, $P = \{p_k\}_{k=1}^q \subseteq \bbN$ with $1 \leq p_1 \leq \cdots \leq p_q$ and $ E_n = \{\eps_n^{(k)}\}_{k=1}^q$ with $\eps_n^{(1)} > \cdots > \eps_n^{(q)} >0 $. Let $\xi_i$ be iid, mean zero, sub-Gaussian random variables, $g \in \Ck{\infty}$, $\mathbf{y}_n = \{y_i\}_{i=1}^n$ with $y_i = g(x_i) + \xi_i$ and $\mathbf{g}_n = \{g(x_i)\}_{i=1}^n$. Then, for all $\alpha > 1$, there exists $\eps_0 > 0$ and $C > 0$ such that for all $E_n$ satisfying  \[
    \eps_0 \geq \eps_n^{(1)} \geq \cdots \geq \eps_n^{(q)} \geq C \l \frac{\log(n)}{n} \r^{1/d},
    \]
    and $\tau > 0$, the following holds with probability $1-Cn^{-\alpha}$:
    \[
    \Vert u_{n,\tau}^{(\mathbf{y}_n)} - u_{n,\tau}^{(\mathbf{g}_n)} \Vert_{\Lp{2}(\mu_n)} \leq C \l \sum_{k=1}^q \lambda_k \frac{\l \eps_n^{(1)} \r^{2p_1}}{\l\eps_n^{(k)}\r^{2p_k}} \l \frac{\log(n)}{n\l \eps_n^{(k)} \r^d} \r^{1/2}  + \frac{\l \eps_n^{(1)} \r^{2p_1}}{\tau} \r
    \]
    where $u_{n,\tau}^{(\mathbf{y}_n)}$ and $u_{n,\tau}^{(\mathbf{g}_n)}$ are the minimizers of $\cR_{n,\tau}^{(\mathbf{y}_n)}$ and $\cR_{n,\tau}^{(\mathbf{g}_n)}$ respectively.
\end{proposition}

\begin{proof}
In the proof $C>0$ will denote a constant that can be arbitrarily large, is independent of $n$, and that may change from line to line.

For $v_n^{(1)}, v_n^{(2)}:\Omega \mapsto \bbR$, we start by estimating as follows using \eqref{eq:differenceWn:gradientEnergy}: \begin{align*}
    &\left\langle \frac{1}{2} \nabla \cR_{n,\tau}^{(\mathbf{a}_n)}(v^{(1)}_n) - \frac{1}{2} \nabla \cR_{n,\tau}^{(\mathbf{a}_n)}(v_n^{(2)}) , v_n^{(1)} - v_n^{(2)} \right\rangle_{\Lp{2}(\mu_n)} = \Vert v_n^{(1)} - v_n^{(2)} \Vert_{\Lp{2}(\mu_n)}^2 \\
    &\qquad \qquad \qquad \qquad \qquad \qquad \qquad + \tau \sum_{k=1}^q \lambda_k \left \langle \Delta_{n,\eps_n^{(k)}}^{p_k} \l v_n^{(1)} - v_n^{(2)} \r ,v_n^{(1)} - v_n^{(2)} \right\rangle_{\Lp{2}(\mu_n)}.
\end{align*}
Since $\Delta_{n,\eps_n^{(k)}}^{p_k}$ is positive semi-definite, using the Cauchy-Schwarz inequality, we can conclude that \begin{equation}\notag
 \Vert v_n^{(1)} - v_n^{(2)} \Vert_{\Lp{2}(\mu_n)} \leq \frac{1}{2} \Vert \nabla \cR_{n,\tau}^{(\mathbf{a}_n)}(v^{(1)}_n) - \nabla \cR_{n,\tau}^{(\mathbf{a}_n)}(v^{(2)}_n)\Vert_{\Lp{2}(\mu_n)}.
\end{equation}
Furthermore, by first order optimality and \eqref{eq:differenceWn:gradientEnergy}, we have that  
\[
u_{n,\tau}^{(\mathbf{y}_n)} - \mathbf{y}_n + \tau \sum_{k=1}^q \lambda_k \Delta_{n,\eps_n^{(k)}}^{p_k} u_{n,\tau}^{(\mathbf{y}_n)}  = 0
\]
and 
\[
u_{n,\tau}^{(\mathbf{g}_n)} - \mathbf{g}_n + \tau \sum_{k=1}^q \lambda_k \Delta_{n,\eps_n^{(k)}}^{p_k} u_{n,\tau}^{(\mathbf{g}_n)}  = 0
\]
implying that \begin{equation} \label{eq:ratesDiscrete:optimalityCondition}
u_{n,\tau}^{(\mathbf{y}_n)} - u_{n,\tau}^{(\mathbf{g}_n)}  + \tau \sum_{k=1}^q \lambda_k \Delta_{n,\eps_n^{(k)}}^{p_k} \l u_{n,\tau}^{(\mathbf{y}_n)} - u_{n,\tau}^{(\mathbf{g}_n)} \r = \bm{\xi}_n
\end{equation}
or equivalently 
\begin{equation} \label{eq:ratesDiscrete:equality}
u_{n,\tau}^{(\mathbf{y}_n)} - u_{n,\tau}^{(\mathbf{g}_n)} = \l \Id + \tau \sum_{k=1}^q \lambda_k   \Delta_{n,\eps_n^{(k)}}^{p_k} \r^{-1} \bm{\xi}_n = w_n.
\end{equation}

We can now estimate as follows, with probability $1 - Cn^{-\alpha}$ (see below): \begin{align}
    \Vert u_{n,\tau}^{(\mathbf{y}_n)} - u_{n,\tau}^{(\mathbf{g}_n)} \Vert_{\Lp{2}(\mu_n)} &\leq \Vert w_n - \tilde{w}_n \Vert_{\Lp{2}(\mu_n)} + \Vert \tilde{w}_n \Vert_{\Lp{2}(\mu_n)} \label{eq:ratesDiscrete:estimate1} \\
    &\leq \frac{1}{2} \Vert \nabla \cR_{n,\tau}^{(\bm{\xi}_n)}(w_n) - \nabla \cR_{n,\tau}^{(\bm{\xi}_n)}(\tilde{w}_n)\Vert_{\Lp{2}(\mu_n)} + \Vert \tilde{w}_n \Vert_{\Lp{2}(\mu_n)} \notag \\
    &\leq C \l \Vert \nabla \cR_{n,\tau}^{(\bm{\xi}_n)}(\tilde{w}_n) \Vert_{\Lp{2}(\mu_n)}  + \frac{\l \eps_n^{(1)} \r^{2p_1}}{\tau} \r \label{eq:ratesDiscrete:estimate2}\\
    &\leq C \l \sum_{k=1}^q \lambda_k \frac{\l \eps_n^{(1)} \r^{2p_1}}{\l\eps_n^{(k)}\r^{2p_k}} \l \frac{\log(n)}{n\l \eps_n^{(k)} \r^d} \r^{1/2}  + \frac{\l \eps_n^{(1)} \r^{2p_1}}{\tau} \r \label{eq:ratesDiscrete:estimate3}
\end{align}
where we used \eqref{eq:ratesDiscrete:equality} for \eqref{eq:ratesDiscrete:estimate1}, the fact that $\nabla\cR_{n,\tau}^{(\bm{\xi}_n)}(w_n) = 0$ and \eqref{eq:differenceWn:boundWnTilde} for \eqref{eq:ratesDiscrete:estimate2} as well as \eqref{eq:differenceWn:boundGradient} for \eqref{eq:ratesDiscrete:estimate3}.
\end{proof}

\begin{proposition}[Rates between discrete noiseless and continuum minimizers] \label{prop:ratesDiscreteContinuum}

Assume that \ref{ass:Main:Ass:S2}, \ref{ass:Main:Ass:M1}, \ref{ass:Main:Ass:M2}, \ref{ass:Main:Ass:D1} and \ref{ass:Main:Ass:W3} hold. Furthermore, let $\rho \in \Ck{\infty}$ and assume that $W_{ii} = 0$.  Let $q \geq 1$, $\{\lambda_k\}_{k=1}^q$ be a sequence of positive numbers, $P = \{p_k\}_{k=1}^q \subseteq \bbN$ with $1 \leq p_1 \leq \cdots \leq p_q$ and $ E_n = \{\eps_n^{(k)}\}_{k=1}^q$ with $\eps_n^{(1)} > \cdots > \eps_n^{(q)} >0 $. Let $\xi_i$ be iid, mean zero, sub-Gaussian random variables, $g \in \Ck{\infty}$ and $\mathbf{g}_n = \{g(x_i)\}_{i=1}^n$. Then, for all $\alpha > 1$ and $\tau_0$, there exists $\eps_0 > 0$ and $C > c > 0$ such that for all $E_n$ satisfying  \[
    \eps_0 \geq \eps_n^{(1)} \geq \cdots \geq \eps_n^{(q)} \geq C \l \frac{\log(n)}{n} \r^{1/d},
    \]
    and $0 < \tau < \tau_0$, the following holds with probability $1-Cn^{-\alpha} - Cne^{-cn\l \eps_n^{(q)} \r^{d + 4p_q}}$:
    \begin{equation} \label{eq:ratesDiscreteContinuum:rates}
    \Vert u_\tau\vert_{\Omega_n}  - u_{n,\tau}^{(\mathbf{g}_n)} \Vert_{\Lp{2}(\mu_n)} \leq C\tau \sum_{k=1}^q \lambda_k \eps_n^{(k)}.
    \end{equation}
    where $u_{n,\tau}^{(\mathbf{g}_n)}$ and $u_\tau$ are the minimizers of $\cR_{n,\tau}^{(\mathbf{g}_n)}$ and $\cR_{\infty,\tau}^{(g)}$ respectively.
\end{proposition}

\begin{proof}
In the proof $C>0$ ($c >0$) will denote a constant that can be arbitrarily large (small), is independent of $n$, and that may change from line to line. 

We start the proof by proving the following fact: if $w_n$ satisfies \begin{equation} \label{eq:ratesDiscreteContinuum:identity}
    \l \Id + \tau \sum_{k=1}^q \lambda_k \Delta_{n,\eps_n^{(k)}}^{p_k} \r v_n = \mathbf{a}_n,
\end{equation}
then $\Vert v_n \Vert_{\Lp{2}(\mu_n)} \leq \Vert \mathbf{a}_n \Vert_{\Lp{2}(\mu_n)}$. Indeed, by Proposition \ref{prop:laplacian}, we know that $\sum_{k=1}^q \lambda_k \Delta_{n,\eps_n^{(k)}}^{p_k}$ is a graph Laplacian, so we can apply the same proof as in \cite[Lemma 2.14]{trillos2022rates} with the eigenpairs of $\mathcal{L}_n^{(q)} = \sum_{k=1}^q \lambda_k \Delta_{n,\eps_n^{(k)}}^{p_k}$ to deduce \eqref{eq:ratesDiscreteContinuum:identity}.

Next, by first order conditions, we note that $u_\tau$ satisfies the equivalent continuum identity 
\begin{equation} \label{eq:ratesDiscreteContinuum:focContinuum}
    \l \Id + \tau \sum_{k=1}^q \lambda_k \Delta_{\rho}^{p_k} \r u_\tau - g = 0
\end{equation}
from which we deduce that 
\begin{equation} \label{eq:ratesDiscreteContinuum:identity2}
\l \Id + \tau \sum_{k=1}^q \lambda_k \Delta_{n,\eps_n^{(k)}}^{p_k} \r u_\tau - g = \tau \l \sum_{k=1}^q \lambda_k \Delta_{n,\eps_n^{(k)}}^{p_k} - \sum_{k=1}^q \lambda_k \Delta_{\rho}^{p_k} \r u_\tau.  
\end{equation}

We then estimate as follows: \begin{align}
   \l \Id + \tau \sum_{k=1}^q \lambda_k \Delta_{n,\eps_n^{(k)}}^{p_k} \r \l u_\tau\vert_{\Omega_n}  - u_{n,\tau}^{(\mathbf{g}_n)} \r &= \l \Id + \tau \sum_{k=1}^q \lambda_k \Delta_{n,\eps_n^{(k)}}^{p_k} \r u_\tau\vert_{\Omega_n} - \mathbf{g}_n \label{eq:ratesDiscreteContinuum:estimate1} \\
   &= \tau \l \sum_{k=1}^q \lambda_k \Delta_{n,\eps_n^{(k)}}^{p_k} - \sum_{k=1}^q \lambda_k\Delta_{\rho}^{p_k} \r u_\tau \vert_{\Omega_n} \label{eq:ratesDiscreteContinuum:estimate2}
\end{align}
where we used the fact that $u_{n,\tau}^{(\mathbf{g}_n)}$ satisfies \eqref{eq:ratesDiscreteContinuum:identity} with $\mathbf{a}_n = \mathbf{g}_n$ by first order conditions for \eqref{eq:ratesDiscreteContinuum:estimate1} and where we used \eqref{eq:ratesDiscreteContinuum:identity2} (as well as a slight abuse of notation) for \eqref{eq:ratesDiscreteContinuum:estimate2}. 

Let $E_k$ be the event such that \cite[Theorem 2.8]{trillos2022rates} holds for $\eps_{n}^{(k)}$: we have \[\bbP(E_k) \geq 1 - Cn^{-\alpha} - Cne^{-cn\l \eps_{n}^{(k)} \r^{d + 4 p_k}}\] which implies that \begin{align}
    \bbP \l \bigcup_{k=1}^q E_k^c \r &\leq \sum_{k=1}^q \bbP(E_k^c) \leq C n^{-\alpha} + C n e^{-cn \l \eps_n^{(q)} \r^{d + 4 p_q}} \notag 
\end{align}
where we used the fact that $p_1 \leq \cdots \leq p_q$ and $ \{\eps_n^{(k)}\}_{k=1}^q$ with $\eps_n^{(1)} > \cdots > \eps_n^{(q)}$ for the last inequality. In turn, this means that \[
\bbP \l \bigcap_{k=1}^q E_k \r \geq 1 -  C n^{-\alpha} - C n e^{-cn \l \eps_n^{(q)} \r^{d + 4 p_q}}.
\]
We therefore obtain, with probability at least $1 -  C n^{-\alpha} - C n e^{-cn \l \eps_n^{(q)} \r^{d + 4 p_q}}$: \begin{align}
    \Vert u_\tau\vert_{\Omega_n}  - u_{n,\tau}^{(\mathbf{g}_n)} \Vert_{\Lp{2}(\mu_n)} &\leq \left\Vert \tau \l \sum_{k=1}^q \lambda_k \Delta_{n,\eps_n^{(k)}}^{p_k} - \sum_{k=1}^q \lambda_k\Delta_{\rho}^{p_k} \r u_\tau \right\Vert_{\Lp{2}(\mu)} \label{eq:ratesDiscreteContinuum:estimate3} \\
    &\leq \tau \sum_{k =1}^q \lambda_k \left\Vert \l \Delta_{n,\eps_n^{(k)}}^{p_k} - \Delta_{\rho}^{p_k} \r u_\tau \right\Vert_{\Lp{2}(\mu)} \notag \\
    &\leq C \tau \sum_{k=1}^q \lambda_k \eps_n^{(k)} \l 1 + \Vert u_\tau \Vert_{\Ck{2p_k +1}} \r \label{eq:ratesDiscreteContinuum:estimate4}
\end{align}
where we used the fact that $u_\tau\vert_{\Omega_n}  - u_{n,\tau}^{(\mathbf{g}_n)}$ satisfies \eqref{eq:ratesDiscreteContinuum:identity} with $\mathbf{a}_n = \tau \l \sum_{k=1}^q \lambda_k \Delta_{n,\eps_n^{(k)}}^{p_k} - \sum_{k=1}^q \lambda_k\Delta_{\rho}^{p_k} \r u_\tau \vert_{\Omega_n}$ for \eqref{eq:ratesDiscreteContinuum:estimate3} and \cite[Theorem 2.8]{trillos2022rates} for \eqref{eq:ratesDiscreteContinuum:estimate4}. 

To establish the desired result, it remains to verify that $ \sup_{0 < \tau < \tau_0} \Vert u_\tau \Vert_{\Ck{2p_k +1}} \leq C$. To that end, we start by noting that \eqref{eq:ratesDiscreteContinuum:focContinuum} implies that \begin{align}
\langle g, \psi_i \rangle_{\Lp{2}(\mu)} &= \langle u_\tau, \psi_i \rangle_{\Lp{2}(\mu)} + \tau \sum_{k=1}^q \lambda_k \langle \Delta_\rho^{p_k} u_\tau, \psi_i \rangle_{\Lp{2}(\mu)} \notag  \\
&= \langle u_\tau, \psi_i \rangle_{\Lp{2}(\mu)} + \tau \sum_{k=1}^q \lambda_k \beta_i^{p_k} \langle u_\tau, \psi_i \rangle_{\Lp{2}(\mu)} \label{eq:ratesDiscreteContinuum:estimate5} \\
&= \langle u_\tau, \psi_i \rangle_{\Lp{2}(\mu)} \l 1 + \tau \sum_{k=1}^q \lambda_k \beta_i^{p_k} \r \label{eq:ratesDiscreteContinuum:estimate6}
\end{align}
where we used the fact that $\Delta_\rho$ is self-adjoint for \eqref{eq:ratesDiscreteContinuum:estimate5}. Then, for $s>0$, we compute as follows: \begin{align}
    \Vert u_\tau \Vert_{\cH^s(\Omega)}^2 &= \sum_{i=1}^\infty \beta^s_i \langle u_\tau, \psi_i \rangle_{\Lp{2}(\mu)}^2 \notag \\
    &= \sum_{i=1}^\infty \beta_i^s \frac{\langle g, \psi_i\rangle_{\Lp{2}(\mu)}^2}{\l 1 + \tau \sum_{k=1}^q \lambda_k \beta_i^{p_k} \r^2} \label{eq:ratesDiscreteContinuum:estimate7} \\
    &\leq \sum_{i=1}^\infty \beta_i^s \langle g, \psi_i\rangle_{\Lp{2}(\mu)}^2 \notag \\
    &= \Vert g \Vert_{\cH^s(\Omega)}^2 \notag
\end{align}
where we used \eqref{eq:ratesDiscreteContinuum:estimate6} for \eqref{eq:ratesDiscreteContinuum:estimate7}. By \cite[Lemma 17]{Stuart}, there exists $c$ and $C$ such that $c\Vert h \Vert_{\Wkp{s}{2}(\Omega)} \leq \Vert h \Vert_{\cH^s(\Omega)} \leq C \Vert h \Vert_{\Wkp{s}{2}(\Omega)}$ for all $h \in \cH^s(\Omega)$. From the above, we therefore deduce that $\Vert u_\tau \Vert_{\Wkp{s}{2}(\Omega)} \leq C \Vert g \Vert_{\Wkp{s}{2}(\Omega)}$. Finally, by Morrey’s inequality~\cite{leoni2017first}, for $s$ sufficiently large there exists $C’>0$ such that \[
\Vert u_\tau \Vert_{\Ck{2p_k +1}} \leq C' \Vert u_\tau \Vert_{\Wkp{s}{2}} \leq C \Vert g \Vert_{\Wkp{s}{2}(\Omega)} .
\]
Since $g \in \Ck{\infty}(\Omega)$, taking the supremum of $\tau$ over $(0,\tau_0)$ concludes the proof.
\end{proof}

\begin{proof}[Proof of Theorem \ref{thm:ratesDiscreteLebellingFunction}]
In the proof $C>0$ will denote a constant that can be arbitrarily large, is independent of $n$, and that may change from line to line. 

We start with an estimate between the continuum solution $u_\tau$ and $g$. Similarly to \eqref{eq:differenceWn:gradientEnergy}, it can easily be verified that\[
\frac{1}{2} \nabla\cR_\infty^{(g)}(v) = v-g + \tau \sum_{k=1}^q \lambda_k \Delta_{\rho}^{p_k} v 
\]
from which we deduce the following identity
\begin{equation}\label{eq:thmRates:identity1}
    \langle \nabla \cR_{\infty}^{(g)}(w), w-v \rangle_{\Lp{2}(\mu)} - \Vert w - v \Vert_{\Lp{2}(\mu)}^2 - \tau \sum_{k=1}^q \lambda_k\langle w-v, \Delta_{\rho}^{p_k} (w-v) \rangle_{\Lp{2}(\mu)} = \cR_{\infty}^{(g)}(w) - \cR_{\infty}^{(g)}(v).
\end{equation}
for any $w,v \in \Wkp{p_q}{2}$. Then, we have \begin{align}
    &\Vert u_\tau - g \Vert_{\Lp{2}(\mu)}^2 + \tau \sum_{k=1}^q \lambda_k\langle u_\tau - g , \Delta_{\rho}^{p_k} (u_\tau - g)  \rangle_{\Lp{2}(\mu)} = \cR_\infty^{(g)}(g) - \cR_\infty^{(g)}(u_\tau) \label{eq:thmRates:estimate1} \\
    &\qquad \qquad = \langle \nabla \cR_{\infty}^{(g)}(g), g-u_\tau \rangle_{\Lp{2}(\mu)} - \Vert g - u_\tau \Vert_{\Lp{2}(\mu)}^2 - \tau \sum_{k=1}^q \lambda_k\langle g-u_\tau, \Delta_{\rho}^{p_k} (g-u_\tau) \rangle_{\Lp{2}(\mu)}\label{eq:thmRates:estimate2}
\end{align}
where we used \eqref{eq:thmRates:identity1} for \eqref{eq:thmRates:estimate1} with $w = u_\tau$, $v=g$ and \eqref{eq:thmRates:identity1} for \eqref{eq:thmRates:estimate2} with $w = g$, $v = u_\tau$. We can therefore conclude that
\[
\Vert u_\tau - g \Vert_{\Lp{2}(\mu)}^2 \leq \frac12 \Vert \nabla \cR_{\infty}^{(g)}(g) \Vert_{\Lp{2}(\mu)} \Vert g-u_\tau \Vert_{\Lp{2}(\mu)} 
\]
or equivalently 
\begin{equation} \label{eq:thmRates:boundBias}
    \Vert u_\tau - g \Vert_{\Lp{2}(\mu)} \leq \tau \sum_{k=1}^q \lambda_k \langle g , \Delta_{\rho}^{p_k} g  \rangle_{\Lp{2}(\mu)} \leq C\tau. 
\end{equation}

We now combine all the previous rates: \begin{align*}
    \Vert u_{n,\tau}^{(\mathbf{y}_n)} -g\vert_{\Omega_n} \Vert_{\Lp{2}(\mu_n)} &= \Vert u_{n,\tau}^{(\mathbf{y}_n)} - u_{n,\tau}^{(\mathbf{g}_n)} \Vert_{\Lp{2}(\mu_n)} + \Vert u_{n,\tau}^{(\mathbf{g}_n)} - u_{\tau}\vert_{\Omega_n} \Vert_{\Lp{2}(\mu_n)} + \Vert u_{\tau}\vert_{\Omega_n} - g\vert_{\Omega_n} \Vert_{\Lp{2}(\mu_n)} \\
    &=: T_1 + T_2 + T_3.
\end{align*}
We can bound $T_1$ using Proposition \ref{prop:ratesDiscreteDiscrete} and $T_2$ using Proposition \ref{prop:ratesDiscreteContinuum}. For $T_3$, we proceed as follows. Let $T_n:\Omega_n \to \Omega$ be a transport map satisfying $(T_n)_{\#}\mu = \mu_n$. Then, we have
\begin{align*}
\|u_\tau|_{\Omega_n} - g|_{\Omega_n}\|_{\Lp{2}(\mu_n)}
&= \|u_\tau|_{\Omega_n}\circ T_n - g|_{\Omega_n}\circ T_n\|_{\Lp{2}(\mu)} \\
&\leq
\|u_\tau |_{\Omega_n}\circ T_n - u_\tau \|_{\Lp{2}(\mu)}
+
\|u_\tau - g\|_{\Lp{2}(\mu)}
+
\|g - g|_{\Omega_n}\circ T_n\|_{\Lp{2}(\mu)} \notag \\
&=: T_4 + T_5 + T_6. \notag
\end{align*}
Since $g \in \Ck{\infty}(\Omega)$, $g$ is Lipschitz and 
\begin{equation} \label{eq:thmRates:T6}
    T_6
\leq C \|T_n - \Id\|_{\Lp{2}(\mu)}.
\end{equation}
Similarly, from the proof of Propositon \ref{prop:ratesDiscreteContinuum}, we recall that
$\sup_{0 < \tau < \tau_0} \Vert u_\tau \Vert_{\Ck{2p_k + 1}(\Omega)} \leq C$ which implies that $u_\tau$ is bounded in $\Ck{1}(\Omega)$ and hence Lipschitz. Consequently, we can bound \begin{equation} \label{eq:thmRates:T4}
   T_4 \leq C \|T_n - \Id\|_{\Lp{2}(\mu)}. 
\end{equation} 
Since the choice of $T_n$ is arbitrary among all maps satisfying $(T_n)_{\#}\mu = \mu_n$,
we take the optimal one minimizing $\|T_n - \Id\|_{\Lp{2}(\mu)}$.
By the probabilistic transport bound of~\cite{fournier2015}, this distance satisfies
\[
\|T_n - \Id\|_{\Lp{2}(\mu)} \le C\,\l \frac{|\log(\delta)|}{n}\r^{1/d}
\]
with probability at least $1-\delta$. By picking $\delta = n^{-\alpha}$, combining \eqref{eq:thmRates:T6}, \eqref{eq:thmRates:T4} and \eqref{eq:thmRates:boundBias} for $T_5$, we obtain \[
T_3 \leq C\l \tau + \l \frac{\log(n)}{n} \r^{1/d} \r
\] 
with probability $1-n^{-\alpha}$ which concludes the proof. 
\end{proof}

\subsection{Truncated energies}

For clarity of presentation, we divide the proof of Theorem~\ref{thm:truncatedEnergies} into two parts. We begin by establishing the result in the simpler case $q = 1$. Next, we examine the spectral convergence properties of the operator $\mathcal{L}_n^{(q)}$, and by incorporating this analysis into the $q=1$ argument, we obtain the general case.  

\subsubsection{Convergence of truncated energies in the single Laplacian case}

The aim of this section is to prove the following result which corresponds to Theorem \ref{thm:truncatedEnergies} when $q=1$.
For notational simplicity, we make the following assumption on the length scale $\eps_n$.

\begin{assumptions}
Assumptions on the length-scale.
\begin{enumerate}[label=\textbf{L.\arabic*}]
\item The length scale $\eps=\eps_n$ is positive, converges to 0, i.e. $0<\eps_n \to 0$ and satisfies the following lower bound: 
\begin{equation*} 
\lim_{n \to \infty} \frac{\log(n)}{n \eps_n^{d+4}} = 0.
\end{equation*}
\label{ass:Main:Ass:L2}
\end{enumerate}
\end{assumptions}

\begin{proposition} \label{prop:truncated}
    Assume that \ref{ass:Main:Ass:S2}, \ref{ass:Main:Ass:M1}, \ref{ass:Main:Ass:M2}, \ref{ass:Main:Ass:D1} and \ref{ass:Main:Ass:W3} hold.
    Let $s >0$ and $\eps_n$ satisfy
    \ref{ass:Main:Ass:L2}.
    Let $K_n \leq n$ be a sequence with $K_n \to \infty$, $\Psi:\TLp{2}(\Omega) \to \bbR$ a continuous function and  
    $(\mu_n,u_n)$ the minimizer of $(\cS\cJ)_{n,\{\eps_n\},\Psi,R_n}^{(1,\{s\})}$. Then, $\bbP$-a.e., there exists a subsequence $(\mu_{n_k},u_{n_k})$ converging to $(\mu,u)$ in $\TLp{2}(\Omega)$ where $(\mu,u)$ is a minimizer of $(\cS\cJ)_{\infty,\Psi}^{(1,\{s\})}$.
\end{proposition}

\begin{proof}
In the proof $C>0$ will denote a constant that can be arbitrarily large, independent of $n$ and that may change from line to line.

Our aim is to show that the functionals $(\cS\cJ)_{n,\{\eps_n\},\Psi,K_n}^{(1,\{s\})}$ $\Gamma$-converge to $(\cS\cJ)_{\infty,\Psi}^{(1,\{s\})}$ and satisfy the compactness property. Once we can do this, all conditions from Proposition \ref{prop:Back:Gamma:minimizers} are satisfied and we can conclude.

Let us define the functionals \[
    (\cS\cJ)_{n,\{\eps_n\},K_n}^{(1,\{s\})}((\nu,v)) = \begin{cases}
        \sum_{k=1}^{K_n} a_{n,\eps_n,k}^s \langle \phi_{n,\eps_n,k}, v \rangle_{\Lp{2}(\mu_n)}^2  &\text{if $\nu = \mu_n$ and $\langle \psi_{n,k}, v \rangle_n = 0$ for all $k > K_n$} \\
        \infty &\text{else}
    \end{cases}
    \]
where $a_{n,\eps_n,k}$ are the eigenvalues of $\Delta_{n,\eps_n}$, and 
\[
    (\cS\cJ)_{\infty}((\nu,v)) = \begin{cases}
        \sum_{k=1}^{\infty} \beta_{k}^s \langle \psi_{k}, v \rangle_{\Lp{2}(\mu)}^2  &\text{if $\nu = \mu$}\\
        \infty &\text{else.}
    \end{cases}
    \]
The latter functionals are similar to $(\cS\cJ)_{n,\{\eps_n\},\Psi,K_n}^{(1,\{s\})}$ and  $(\cS\cJ)_{\infty,\Psi}^{(1,\{s\})}$, the only difference being that they do not contain the data fidelity term $\Psi$. 

First, we tackle the $\liminf$-inequality. We assume that $(\nu,v) \in \TLp{2}(\Omega)$ and that $(\nu_n,v_n) \to (\nu,v)$ in $\TLp{2}$. If $\liminf_{n \to \infty} (\cS\cJ)_{n,\{\eps_n\},K_n}^{(1,\{s\})}((\nu_n,v_n)) = + \infty$, then the inequality is trivial. Hence, without loss of generality, let us assume that $\sup_{n \in \mathbb{N}} (\cS\cJ)_{n,\{\eps_n\},K_n}^{(1,\{s\})}((\nu_n,v_n)) \leq C$. In particular, this implies that $\nu_n = \mu_n$, $\langle \phi_{n,\eps_n,k}, v_n \rangle_{\Lp{2}(\mu_n)} = 0$ for all $k > K_n $ and $\sup_{n \in \mathbb{N}} \sum_{k=1}^{K_n} a_{n,\eps_n,k}^s \langle \phi_{n,\eps_n,k}, v \rangle_{\Lp{2}(\mu_n)}^2 \leq C$. Since we have $\mu_n \to \nu$ weakly (by the $\TLp{2}$-convergence assumption---see Proposition \ref{prop:Back:TLp}) and $\mu_n \to \mu$ weakly (convergence of the empirical measures), we conclude (by the uniqueness of weak limits) that $\nu = \mu$. We then proceed as in \cite[Theorem 2.2]{Stuart}.  

Let us start by assuming that $\sum_{k=1}^\infty \beta_k^s \langle v, \psi_k \rangle_{\Lp{2}(\Omega)}^2 < \infty$. 
In particular, since $\phi_{n,\eps_n,k} \to \psi_k$  and $v_n \to v$ in $\TLp{2}(\Omega)$ \cite{Stuart}, we have that $\langle \phi_{n,\eps_n,k} , v_n \rangle_{\Lp{2}(\mu_n)} \to \langle v, \psi_k \rangle_{\Lp{2}(\mu)}$ by \cite[Proposition 2.6]{GARCIATRILLOS2018239}. Furthermore, by \cite[Theorem 1.2]{GARCIATRILLOS2018239}, we have $a_{n,\eps_n,k} \to \beta_k$. Now, let $\delta >0$ and pick $K$ such that \[
\sum_{k=1}^K \beta_k^s \langle v, \psi_k \rangle_{\Lp{2}(\Omega)}^2 \geq \sum_{k=1}^\infty \beta_k^s \langle v, \psi_k \rangle_{\Lp{2}(\Omega)}^2 - \delta.
\]   
Since $K_n \to \infty$, we have \begin{align*}
\liminf_{n \to \infty} \sum_{k=1}^{K_n} a_{n,\eps_n,k}^s \langle \phi_{n,\eps_n,k}, v \rangle_{\Lp{2}(\mu_n)}^2 &\geq \liminf_{n \to 
\infty} \sum_{k=1}^{K} a_{n,\eps_n,k}^s \langle \phi_{n,\eps_n,k}, v \rangle_{\Lp{2}(\mu_n)}^2 \\
&= \sum_{k=1}^K \beta_k^s \langle v, \psi_k \rangle_{\Lp{2}(\Omega)}^2 \\
&\geq \sum_{k=1}^\infty \beta_k^s \langle v, \psi_k \rangle_{\Lp{2}(\Omega)}^2 - \delta.
\end{align*}
Taking $\delta \to 0$, we obtain the $\liminf$-inequality. Now, assume that $\sum_{k=1}^\infty \beta_k^s \langle v, \psi_k \rangle_{\Lp{2}(\Omega)}^2 = \infty$. Then, for any $K \in \mathbb{N}$, we have \begin{align*}
C &\geq \liminf_{n \to \infty} \sum_{k=1}^{K_n} a_{n,\eps_n,k}^s \langle \phi_{n,\eps_n,k}, v \rangle_{\Lp{2}(\mu_n)}^2 \\
&\geq \lim_{K \to \infty} \liminf_{n \to 
\infty} \sum_{k=1}^{K} a_{n,\eps_n,k}^s \langle \phi_{n,\eps_n,k}, v \rangle_{\Lp{2}(\mu_n)}^2 \\
&= \lim_{K \to \infty} \sum_{k=1}^K \beta_k^s \langle v, \psi_k \rangle_{\Lp{2}(\Omega)}^2 \\
&= \infty
\end{align*}
which is a contradiction.

For the $\limsup$-inequality, we let $(\nu,v) \in \TLp{2}(\Omega)$. If $(\cS\cJ)_{n,\{\eps_n\},K_n}^{(1,\{s\})}((\nu,v)) = \infty$, the inequality is trivial, so we assume that $\nu = \mu$ and $\sum_{k=1}^{\infty} \beta_{k}^s \langle \psi_{k}, v \rangle_{\Lp{2}(\mu)}^2 <\infty$ or equivalently $v \in \mathrm{W}^{s,2}(\Omega)$ \cite{Stuart}. If we can prove the $\limsup$-inequality on a dense subset of $\{\mu\} \times \mathrm{W}^{s,2}(\Omega)$, namely $\{\mu\} \times \mathrm{C}_c^\infty(\Omega)$, we can conclude due to \cite[Remark 2.7]{Trillos3}. 

Let $v \in \mathrm{C}_c^\infty(\Omega)$ and define $v_n$ to be the restriction of $v$ to $\Omega_n$. Let us consider the sequence $(\mu_n,\bar{v}_n)$ where $\bar{v}_n = v_n - \sum_{k = K_n + 1}^n \langle v_n, \phi_{n,\eps_n,k} \rangle_n \phi_{n,\eps_n,k}$. It is clear that $\langle \phi_{n,\eps_n,k} , \bar{v}_n \rangle_n = 0$ for $k > K_n$. We now verify that $(\mu_n,\bar{v}_n) \to (\mu,v)$ in $\TLp{2}(\Omega)$.
With $T_n$ the transport maps of Theorem \ref{thm:transport}, we estimate as follows: 
\begin{align}
    \int_\Omega \vert \bar{v}_n \circ T_n -v \vert^2 \, \dd \mu & \leq 2 \underbrace{\int_\Omega \vert v_n \circ T_n -v \vert^2 \, \dd \mu}_{=:T_1} + 2\int_\Omega \vert \bar{v}_n\circ T_n - v_n\circ T_n\vert^2 \, \dd \mu \notag \\ 
    & \leq 2T_1 + 2\sum_{i=1}^n \langle \bar{v}_n-v_n,\phi_{n,\eps_n,k}\rangle_{\Lp{2}(\mu_n)}^2 \notag \\
    &\leq 2 T_1 + \frac{2}{a_{n,\eps_n,K_n + 1}^2} \sum_{k=K_n + 1}^n a_{n,\eps_n,k}^2 \langle v_n , \phi_{n,\eps_n,k} \rangle_{\Lp{2}(\mu_n)}^2 \label{eq:trunc:estimate2} \\
    &\leq 2 T_1 + \frac{C}{a_{n,\eps_n,K_n + 1}^2} \label{eq:trunc:estimate3}
\end{align}
where we used the fact that 
the eigenvalues are ordered for \eqref{eq:trunc:estimate2} and \cite[Lemma 4.19]{weihs2023consistency} for  \eqref{eq:trunc:estimate3}. We know 
from \cite[Theorem 1.4]{GARCIATRILLOS2018239}
that $T_1 \to 0$ 
and from the proof of \cite[Theorem 2.2]{Stuart} that $a_{n,\eps_n,K_n + 1}^2 \to \infty$ which allows us to conclude that $(\mu_n,\bar{v}_n) \to (\mu,v)$ in $\TLp{2}(\Omega)$.

Since $v\in\Ck{\infty}_c$ then $v\in \Wkp{m}{2}$ for any $m\in\bbN$.
Choose $m\in\bbN$ with $m>\frac{s}{2}$ and let
$\delta > 0$ be such that $s + \delta = 2m$. 
As an intermediary step, let us compute: \begin{align}
    T_2 &:= \sum_{k=1}^{n} a_{n,\eps_n,k}^s \left\langle \phi_{n,\eps_n,k}, \sum_{j = K_n + 1}^n \langle v_n, \phi_{n,\eps_n,j} \rangle_{\Lp{2}(\mu_n)} \phi_{n,\eps_n,j} \right\rangle_{\Lp{2}(\mu_n)}^2 \notag \\
    &= \sum_{k=K_n + 1}^{n} a_{n,\eps_n,k}^s \langle \phi_{n,\eps_n,k}, v_n \rangle_{\Lp{2}(\mu_n)}^2 \notag \\
    &\leq \frac{1}{a_{n,\eps_n,{K_n+1}}^{\delta}} \sum_{k=K_n + 1}^{n} a_{n,\eps_n,k}^{s + \delta} \langle \phi_{n,\eps_n,k}, v_n \rangle_{\Lp{2}(\mu_n)}^2 \notag \\
    &\leq \frac{C}{a_{n,\eps_n,{K_n+1}}^{\delta}}. \notag 
\end{align}
Arguing as above, we obtain that $T_2 \to 0$. We conclude by estimating as follows: \begin{align}
    \limsup_{n \to \infty} \sqrt{\sum_{k=1}^{K_n} a_{n,\eps_n,k}^s \langle \phi_{n,\eps_n,k}, \bar{v}_n \rangle_{\Lp{2}(\mu_n)}^2} &\leq  \limsup_{n \to \infty} \sqrt{\sum_{k=1}^{n} a_{n,\eps_n,k}^s \langle \phi_{n,\eps_n,k}, \bar{v}_n \rangle_{\Lp{2}(\mu_n)}^2} \notag \\
    &\leq \limsup_{n \to \infty} \sqrt{\sum_{k=1}^{n} a_{n,\eps_n,k}^s \langle \phi_{n,\eps_n,k}, v_n \rangle_{\Lp{2}(\mu_n)}^2} +  \limsup_{n \to \infty} \sqrt{T_2} \label{eq:trunc:estimate5} \\
    &\leq \sqrt{\sum_{k=1}^{\infty} \beta_{k}^s \langle \psi_{k}, v \rangle_{\Lp{2}(\mu)}^2} \label{eq:trunc:estimate6}
\end{align}
where used \cite[Lemma 4.15]{weihs2023consistency} for \eqref{eq:trunc:estimate5}, \cite[Proposition 4.21]{weihs2023consistency} and the fact that $T_2 \to 0$ for \eqref{eq:trunc:estimate6}. Squaring the last inequality, we obtain the $\limsup$-inequality. 

Summarizing the above two results, we obtain that $(\cS\cJ)_{n,\{\eps_n\},K_n}^{(1,\{s\})}$ $\Gamma$-converges to $(\cS\cJ)_{\infty}^{(1,\{s\})}$. Since $\Psi$ is continuous in $\TLp{2}(\Omega)$, we use Proposition \ref{prop:additivity} to deduce that 
\[
(\cS\cJ)_{n,\{\eps_n\},\Psi,K_n}^{(1,\{s\})} \qquad \text{$\Gamma$-converges to} \qquad (\cS\cJ)_{\infty,\Psi}^{(1,\{s\})}.
\]

Let us now consider a sequence $(\mu_n,v_n)$ minimizing $(\cS\cJ)_{n,\{\eps_n\},\Psi,K_n}^{(1,\{s\})}$ with $\sup_{n \in \mathbb{N}} \Vert v_n \Vert_{\Lp{2}(\mu_n)} \leq C$. In particular, we note that $\langle \phi_{n,\eps_n,k}, v_n \rangle_n = 0$ for all $k > K_n$ and recall that $K_n \to \infty$. Therefore, we can apply the same proof as in \cite[Theorem 2.2]{Stuart} to show that there exists a converging subsequence in $\TLp{2}(\Omega)$. 

Specifically, $\sup_{n \in \mathbb{N}} \Vert v_n \Vert_{\Lp{2}(\mu_n)} \leq C$ implies that $ \sup_{n \in \mathbb{N}} \sum_{k=1}^{K_n} \langle v_n, \phi_{n,\eps_n,k} \rangle_{\Lp{2}(\mu_n)}^2 \leq C.$ Hence, by a diagonal procedure, we can find a sequence $n_m \to \infty$ such that for every $k$, $\langle v_{n_m}, \psi_{n_m,\eps_{n_m},k} \rangle_{\Lp{2}(\mu_{n_m})}$ converges to some coefficient $\gamma_k$. By Fatou's lemma, $\sum_{k=1}^\infty \vert \gamma_k \vert^2 \leq \liminf_{m \to \infty} \sum_{k=1}^{n_m} \vert \langle v_{n_m}, \phi_{n_m,\eps_{n_m},k} \rangle_{\Lp{2}(\mu_{n_m})}  \vert^2 \leq C$, so we can define $v = \sum_{k=1}^\infty \gamma_k \psi_k \in \Lp{2}(\mu)$. Using \cite[Lemma 7.7]{Stuart}, we obtain a sequence $R_{n_m} \to \infty$ such that $\sum_{k=1}^{R_{n_m}} \langle v_{n_m}, \phi_{n_m,\eps_{n_m},k} \rangle_{\Lp{2}(\mu_{n_m})} \phi_{n_m,\eps_{n_m},k} \to v$ in $\TLp{2}(\Omega)$. We note that $R_{n_m}$ can always be picked such that $R_{n_m} \leq K_{n_m}$. Indeed, the $\TLp{2}$-convergence resulting from \cite[Lemma 7.7]{Stuart} holds for any sequence converging to $\infty$ and majorized by $R_{n_m}$: therefore, we can always pick $\tilde{R}_{n_m} = \min\{R_{n_m}, K_{n_m}\}$. 

Then, we check the convergence in $\TLp{2}$ of $v_{n_m}$ to $v$:  
\begin{align*}
    \Vert v_{n_m} \circ T_{n_m} - v \Vert_{\Lp{2}(\mu)} &\leq \Vert v_{n_m} - \sum_{k=1}^{R_{n_m}} \langle v_{n_m}, \phi_{n_m,\eps_{n_m},k} \rangle_{\Lp{2}(\mu_{n_m})} \phi_{n_m,\eps_{n_m},k} \Vert_{\Lp{2}(\mu_n)} \\
    &\qquad \qquad+\Vert \sum_{k=1}^{R_{n_m}} \langle v_{n_m}, \phi_{n_m,\eps_{n_m},k} \rangle_{\Lp{2}(\mu_{n_m})} \phi_{n_m,\eps_{n_m},k} \circ T_{n_m} - v \Vert_{\Lp{2}(\mu)} \\
    &\leq \frac{1}{a_{n_m,\eps_{n_m},R_{n_m}}^s} \sum_{k=R_{n_m}+1}^{K_{n_m}} a_{n_m,\eps_{n_m},k}^s \langle v_{n_m}, \phi_{n_m,\eps_{n_m},k} \rangle_{\Lp{2}(\mu_{n_m})}^2 \phi_{n_m,\eps_{n_m},k} \\
    &\qquad \qquad+\Vert \sum_{k=1}^{R_{n_m}} \langle v_{n_m}, \phi_{n_m,\eps_{n_m},k} \rangle_{\Lp{2}(\mu_{n_m})} \phi_{n_m,\eps_{n_m},k} \circ T_{n_m} - v \Vert_{\Lp{2}(\mu)} \\
    &\leq \frac{C}{a_{n_m,\eps_{n_m},R_{n_m}}^s} + \Vert \sum_{k=1}^{R_{n_m}} \langle v_{n_m}, \phi_{n_m,\eps_{n_m}k} \rangle_{\Lp{2}(\mu_{n_m})} \phi_{n_m,\eps_{n_m},k} \circ T_{n_m} - v \Vert_{\Lp{2}(\mu)}
\end{align*} 
where the last inequality follows from the fact that $\sup_{n \in \mathbb{N}} \sum_{k=1}^{K_n} a_{n,\eps_n,k}^s \langle v_n, \phi_{n,\eps_n,k} \rangle_{\Lp{2}(\mu_n)}^2 \leq C$ since $(\mu_n,v_n)$ are minimizers of $(\cS\cJ)_{n,\{\eps_n\},\Psi,K_n}$ (see also \cite[Lemma 4.25]{weihs2023consistency}).
In order to conclude that $v_{n_m}$ to $v$ in $\TLp{2}(\Omega)$, we note the following two facts: the first term in the last inequality tends to $0$ as argued in \cite[Theorem 2]{Stuart}; the second term tends to 0 since $\sum_{k=1}^{R_{n_m}} \langle v_{n_m}, \psi_{n_m,k} \rangle_{\Lp{2}(\mu_{n_m})} \psi_{n_m,k} \to v$ in $\TLp{2}(\Omega)$. By Proposition \ref{prop:Back:Gamma:minimizers}, we know that the limiting point $v$ is a minimizer of $(\cS\cJ)_{\infty,\Psi}^{(1,\{s\})}$.
\end{proof}

\subsubsection{Spectral convergence of \texorpdfstring{$\mathcal{L}_n^{(q)}$}{Lnq}}

In this section, we analyze the spectral convergence of $\mathcal{L}_n^{(q)}$ and, by combining with the results of the previous section, prove Theorem \ref{thm:truncatedEnergies}.

\begin{lemma}[Eigenpairs of $\mathcal{L}^{(q)}$] \label{lem:eigenpairs}
    Assume that \ref{ass:Main:Ass:S2}, \ref{ass:Main:Ass:M1}, \ref{ass:Main:Ass:M2}, \ref{ass:Main:Ass:W2} and \ref{ass:Main:Ass:D1} hold. Let $q \geq 1$, $P = \{p_k\}_{k=1}^q \subseteq \bbR$ with $p_1 \leq \cdots \leq p_q$ and $E_n = \{\eps_n^{(k)}\}_{k=1}^q$ with $\eps_n^{(1)} > \cdots > \eps_n^{(q)}$. Assume that $\rho \in \Ck{\infty}$. The eigenpairs of $\mathcal{L}^{(q)}$ are $
    \left\{\l \sum_{k=1}^q \lambda_k \beta_i^{p_k},\psi_i \r\right\}_{i=1}^\infty.$
\end{lemma}

\begin{proof}
    First, we see that \[
    \mathcal{L}^{(q)}\psi_i = \sum_{k=1}^{q} \lambda_k \Delta_{\rho}^{p_k} \psi_i = \sum_{k=1}^q \lambda_k \beta_i^{p_k} \psi_i.
    \]
    This implies that $\l \sum_{k=1}^q \lambda_k \beta_i^{p_k},\psi_i \r$ is an eigenpair of $\mathcal{L}^{(q)}$. 
    We now consider two cases.

    \paragraph{Case 1.} Assume $(\sum_{k=1}^q \lambda_k \beta_j^{p_k},\psi)$ is an eigenpair of $\cL^{(q)}$. 
    Then,
    \begin{align*}
        \mathcal{L}^{(q)} \psi &= \mathcal{L}^{(q)} \sum_{i=1}^\infty \langle \psi_i, \psi \rangle_{\Lp{2}(\mu)} \psi_i \\
        &= \sum_{i=1}^\infty \langle \psi_i, \psi \rangle_{\Lp{2}(\mu)} \l \sum_{k=1}^q \lambda_k \beta_i^{p_k} \r \psi_i 
    \end{align*}
    and
    \begin{align*}
        \mathcal{L}^{(q)} \psi &= \sum_{i=1}^\infty \langle \psi_i, \psi \rangle_{\Lp{2}(\mu)} \l \sum_{k=1}^q \lambda_k \beta_j^{p_k} \r \psi_i.
        \end{align*}
        Hence, 
        \[
        \sum_{i=1}^\infty \langle \psi_i, \psi \rangle_{\Lp{2}(\mu)} \psi_i \l  \sum_{k=1}^q \lambda_k \beta_i^{p_k} -  \sum_{k=1}^q \lambda_k \beta_j^{p_k}  \r =0. 
        \]
        Since, $\{\psi_i\}_{i=1}^\infty$ are linearly independent then $\langle \psi_i, \psi \rangle_{\Lp{2}(\mu)} \l  \sum_{k=1}^q \lambda_k \beta_i^{p_k} -  \sum_{k=1}^q \lambda_k \beta_j^{p_k}  \r = 0$ for all $i\in\bbN$.
        Hence for $i\neq j$ (since $\beta_i\neq\beta_j$) we have $\langle \psi_i,\psi\rangle_{\Lp{2}(\mu)} = 0$.
        As $\|\psi\|_{\Lp{2}(\mu)}=1$ (assuming we normalised) then $\psi=\pm \psi_j$.
    \paragraph{Case 2.} Assume $(\beta,\psi)$ is an eigenpair of $\cL^{(q)}$. 
    Then an analogous calculation to the one above implies
    \[ \sum_{i=1}^\infty \langle \psi_i,\psi\rangle_{\Lp{2}(\mu)} \l \beta - \sum_{k=1}^q \lambda_k \beta_i^{p_k}\r \psi_i = 0. \]
    Again, as $\{\psi_i\}_{i=1}^\infty$ are linearly independent then $\langle \psi_i, \psi \rangle_{\Lp{2}(\mu)} \l  \beta -  \sum_{k=1}^q \lambda_k \beta_i^{p_k}  \r = 0$ for all $i\in\bbN$.
    Since $\psi\neq 0$, then at least one $\langle\psi_i,\psi\rangle_{\Lp{2}(\mu)} \neq 0$.
    For this $i$ we then must have $\beta =  \sum_{k=1}^q \lambda_k \beta_i^{p_k}$ and so we are back in Case 1.
\end{proof}

\begin{proposition}[Convergence of eigenpairs] \label{prop:convergenceEigenpairs}
Assume that \ref{ass:Main:Ass:S2}, \ref{ass:Main:Ass:M1}, \ref{ass:Main:Ass:M2}, \ref{ass:Main:Ass:W2} and \ref{ass:Main:Ass:D1} hold. Let $q \geq 1$, $P = \{p_k\}_{k=1}^q \subseteq \bbR$ with $p_1 \leq \cdots \leq p_q$ and $ E_n = \{\eps_n^{(k)}\}_{k=1}^q$ with $\eps_n^{(1)} > \cdots > \eps_n^{(q)}$. Assume that $\rho \in \Ck{\infty}$. Assume that $\eps_n^{(q)}$ satisfies \ref{ass:Main:Ass:L2}. Then, $\bbP$-a.s., the following holds:\begin{enumerate}
    \item $\beta_{n,i} \to \sum_{k=1}^q \lambda_k \beta_i^{p_k}$;    \item $(\mu_n,\psi_{n,i}) \to (\mu,\psi_i)$ in $\TLp{2}(\Omega)$.
\end{enumerate}
\end{proposition}

\begin{proof}
In the proof $C>0$ will denote a constant that can be arbitrarily large, is independent of $n$ and that may change from line to line.

    In order to prove the proposition, we want to proceed as in \cite[Theorem 1.2]{GARCIATRILLOS2018239} where the authors show the analogous result for a single Laplacian matrix $\Delta_{n,\eps_n}$. In particular, the proof relies on the following results: \begin{enumerate}
        \item $\Delta_{n,\eps_n}$ and $\Delta_\rho$ are self-adjoint and positive semi-definite;
        \item the functional $\langle v, \Delta_{n,\eps_n} v \rangle_{\Lp{2}(\mu_n)}$ $\Gamma$-converges to $ \langle v, \Delta_{\rho} v \rangle_{\Lp{2}(\mu)}$ \cite[Theorem 1.4]{GARCIATRILLOS2018239};
        \item if a sequence satisfies $\sup_n \langle v, \Delta_{n,\eps_n} v \rangle_{\Lp{2}(\mu_n)} \leq C$ and $\Vert v_n \Vert_{\Lp{2}(\mu_n)} \leq C$, then there exists a converging subsequence in $\TLp{2}(\Omega)$ \cite[Theorem 1.4]{GARCIATRILLOS2018239}.
    \end{enumerate}
    Our Laplacian $\cL^{(q)}_n$ satisfies the same three properties: \begin{enumerate}
        \item Since each $\Delta_{n,\eps_n^{(k)}}$ is self-adjoint and positive semi-definite, so is $\mathcal{L}_n^{q} = \sum_{k=1}^q \lambda_k \Delta_{n,\eps_n^{(k)}}$. The same argument applies to $\mathcal{L}^{(q)}$.  
        \item The fact that $\langle v, \mathcal{L}_n^{(q)} v \rangle_{\Lp{2}(\mu_n)}$ $\Gamma$-converges to $ \langle v, \mathcal{L}^{(q)} v \rangle_{\Lp{2}(\mu)}$ was shown in the proof of \cite[Theorem 3.5]{weihs2025Hypergraphs}.
        \item If we assume that a sequence satisfies $\sup_n \langle v_n, \mathcal{L}_n^{(q)} v_n \rangle_{\Lp{2}(\mu_n)} \leq C$ and $\Vert v_n \Vert_{\Lp{2}(\mu_n)} \leq C$, then in particular $\sup_n \langle v_n, \Delta_{n,\eps_n^{(q)}}^{p_1} v_n \rangle_{\Lp{2}(\mu_n)} \leq C$ and $\Vert v_n \Vert_{\Lp{2}(\mu_n)} \leq C$: we can therefore use \cite[Theorem 2]{Stuart} to deduce the existence of a converging subsequence in $\TLp{2}(\Omega)$.
    \end{enumerate}


Specifically, let us start with the eigenvalues. First, we recall that since $\mathcal{L}_n^{(q)}$ is self-adjoint and positive semi-definite, we can apply the Courant-Fisher characterization of eigenvalues \cite[Max-min theorem]{Chavel} to infer that \begin{equation} \label{eq:convergenceEigenpairs:discreteCourantFisher}
\beta_{n,i} = \sup_{S \in \Sigma_{n,i-1}} \,  \min_{ v \in S^{\perp}, \, \Vert v \Vert_{\Lp{2}(\mu_n)} =1 } \langle \mathcal{L}_n^{(q)} v,v \rangle_{\Lp{2}(\mu_n)}
\end{equation}
where $\Sigma_{n,i-1}$ denotes the subspaces of $\bbR^n$ of dimension $i-1$ and $S^{\perp}$ denotes the orthogonal complement of $S$ with respect to the inner product in $\Lp{2}(\mu_n)$.  We now proceed by induction on \( i \).

\paragraph{Base case $i = 1$.}

We first observe that the graphs \( (\Omega_n, W_{n,\varepsilon_n^{(k)}}) \) are connected \( \mathbb{P} \)-a.e. for \( n \) large enough. This follows from the ordering \( \varepsilon_n^{(1)} > \cdots > \varepsilon_n^{(q)} \) and from Assumption~\eqref{ass:Main:Ass:L2}, which guarantees connectivity in the random geometric graph regime~\cite{goel,DBLP:books/ox/P2003}. Consequently, for all sufficiently large \( n \), the first eigenpair \( (0, \mathbf{1}) \), where \( \mathbf{1} \in \mathbb{R}^n \) is the constant-one vector, is shared across all Laplacians \( \Delta_{n,\varepsilon_n^{(k)}} \). Thus, the first eigenpair of the discrete operator \( \mathcal{L}_n^{(q)} \) is given by \( (\beta_{n,1}, \psi_{n,1}) = (0, \mathbf{1}) \).

Furthermore, since the domain \( \Omega \) is connected by Assumption \ref{ass:Main:Ass:S2}, the continuum Laplacian \( \Delta_\rho \) has first eigenpair \( (0, \one) \), where \( \one \) denotes the constant function equal to one. By Lemma~\ref{lem:eigenpairs}, the first eigenpair of the continuum limit operator \( \mathcal{L}^{(q)} \) is
\[
\left( \sum_{k=1}^q \lambda_k \beta_1^{p_k}, \psi_1 \right) = (0, \one).
\]
It follows that \( \beta_{n,1} \to \beta_1 = 0 \) and \( (\mu_n, \psi_{n,1}) \to (\mu, \psi_1) \) in \( \TLp{2}(\Omega) \) is satisfied.

\paragraph{Induction step.} Now, suppose that $\beta_{n,\ell} \to \beta_\ell$ for all $\ell \leq i-1$.

\noindent \textit{Proof of the lower bound.} \hspace{0.2cm} Let \( S \in \Sigma_{i-1} \), where $\Sigma_{i-1}$ denotes the subspaces of $\Lp{2}(\Omega)$ of dimension $i-1$. In this case, we will also write $S^{\perp}$ for the orthogonal complement of $S$ with respect to the inner product in $\Lp{2}(\mu)$. Let \( \{v_1, \dots, v_{i-1}\} \) be an orthonormal basis of \( S \). For each \( \ell = 1, \dots, i-1 \), the $\limsup$-inequality in \cite[Theorem 3.5]{weihs2025Hypergraphs} ensures the existence of a sequence of functions $v_{n,\ell} \in \Lp{2}(\mu_n)$  such that \( (\mu_n,v_{n,\ell}) \to (\mu,v_\ell) \) in $\TLp{2}(\Omega)$ as \( n \to \infty \). By \cite[Proposition 2.6]{GARCIATRILLOS2018239}, we have for all \( 1 \leq \ell \leq i-1 \),
\[
\lim_{n \to \infty} \|v_{n,\ell}\|_{\Lp{2}(\mu_n)} = \|v_\ell\|_{\Lp{2}(\mu)} = 1,
\]
and for all \( \ell \neq j \),
\begin{equation} \label{eq:convergenceEigenpairs:discreteOrthogonality}
    \lim_{n \to \infty} \langle v_{n,\ell}, v_{m,j} \rangle_{\Lp{2}(\mu_n)} = \langle v_\ell, v_j \rangle_{\Lp{2}(\mu)} = 0.
\end{equation}
These results guarantee that for sufficiently large \( n \), the set \( \{v_{n,1}, \dots, v_{n,i-1}\} \) spans a \( (i-1) \)-dimensional subspace of \( \Lp{2}(\mu_n) \). We can then apply the Gram–Schmidt orthonormalization process to obtain an orthonormal basis \( \{\tilde{v}_{n,1}, \dots, \tilde{v}_{n,i-1}\} \). Namely, we define
\[
\tilde{v}_{n,1} := \frac{v_{n,1}}{\|v_{n,1}\|_{\Lp{2}(\mu_n)}},
\]
and recursively for \( \ell = 2, \dots, i-1 \),
\[
\tilde{h}_{n,\ell} := v_{n,\ell} - \sum_{j=1}^{i-1} \langle v_{n,\ell}, \tilde{v}_{n,j} \rangle_{\Lp{2}(\mu_n)} \tilde{v}_{n,j},
\qquad
\tilde{v}_{n,\ell} := \frac{\tilde{h}_{n,\ell}}{\|\tilde{h}_{n,\ell}\|_{\Lp{2}(\mu_n)}}.
\]
By \eqref{eq:convergenceEigenpairs:discreteOrthogonality} and \cite[Proposition 2.6]{GARCIATRILLOS2018239}, it is straight-forward to check that $\tilde{v}_{n,\ell} \to v_\ell$ in $\TLp{2}(\Omega)$ for $1\leq \ell \leq i-1$. Let $S_n \in \Sigma_{n,i-1}$ be the subset spanned by \( \{\tilde{v}_{n,1}, \dots, \tilde{v}_{n,i-1}\} \).

We now want to show that \begin{equation} \label{eq:convergenceEigenpairs:liminfIdentity}
    \liminf_{n \to \infty} \beta_{n,i} \geq \min_{v \in S , \, \Vert v \Vert_{\Lp{2}(\mu)} = 1} \langle v, \mathcal{L}^{(q)} v \rangle_{\Lp{2}(\mu)}. 
\end{equation}
First, by \eqref{eq:convergenceEigenpairs:discreteCourantFisher}, since \(
\beta_{n,i} \geq \min_{ v \in S_n^{\perp}, \, \Vert v \vert_{\Lp{2}(\mu_n)} =1 } \langle \mathcal{L}_n^{(q)} v,v \rangle_{\Lp{2}(\mu_n)}
\), if $$\liminf_{n \to \infty} \min_{ v \in S_n^{\perp}, \, \Vert v \Vert_{\Lp{2}(\mu_n)} =1 } \langle \mathcal{L}_n^{(q)} v,v \rangle_{\Lp{2}(\mu_n)} = \infty,$$ then \eqref{eq:convergenceEigenpairs:liminfIdentity} is trivially satisfied. We therefore consider the case when 
$$\liminf_{n \to \infty} \min_{ v \in S_n^{\perp}, \, \Vert v \Vert_{\Lp{2}(\mu_n)} =1 } \langle \mathcal{L}_n^{(q)} v,v \rangle_{\Lp{2}(\mu_n)} < \infty$$ 
and, without loss of generality (see \cite{GARCIATRILLOS2018239,weihs2023consistency}), we can assume that \[
\liminf_{n \to \infty} \min_{ v \in S_n^{\perp}, \, \Vert v \Vert_{\Lp{2}(\mu_n)} =1 } \langle \mathcal{L}_n^{(q)} v,v \rangle_{\Lp{2}(\mu_n)} = \lim_{n \to \infty} \min_{ v \in S_n^{\perp}, \, \Vert v \Vert_{\Lp{2}(\mu_n)} =1 } \langle \mathcal{L}_n^{(q)} v,v \rangle_{\Lp{2}(\mu_n)} < \infty.
\]

Let $w_n \in S_n^\perp$ be a sequence such that $\Vert w_n \Vert_{\Lp{2}(\mu_n)} = 1$ and $$\lim_{n \to \infty} \langle \mathcal{L}_n^{(q)} w_n, w_n \rangle_{\Lp{2}(\mu_n)} = \lim_{n \to \infty} \min_{ v \in S_n^{\perp}, \, \Vert v \Vert_{\Lp{2}(\mu_n)} } \langle \mathcal{L}_n^{(q)} v,v \rangle_{\Lp{2}(\mu_n)} < \infty.$$ Since $\lim_{n \to \infty} \langle \mathcal{L}_n^{(q)} w_n, w_n \rangle_{\Lp{2}(\mu_n)} < \infty$, we have that $\sup_{n}  \langle \mathcal{L}_n^{(q)} w_n, w_n \rangle_{\Lp{2}(\mu_n)} < \infty$ and, in particular, $$\sup_n  \langle \Delta_{n,\eps_n^{(1)}}^{p_1} w_n, w_n \rangle_{\Lp{2}(\mu_n)} < \infty.$$ By \cite[Theorem 2]{Stuart}, we therefore obtain a converging subsequence $(\mu_{n_m},w_{n_m}) \to (\mu,w)$ in $\TLp{2}(\Omega)$. By \cite[Proposition 2.6]{GARCIATRILLOS2018239}, we deduce that $\Vert w \Vert_{\Lp{2}(\mu)} = \lim_{m\to \infty} \Vert w_{n_m} \Vert_{\Lp{2}(\mu_{n_m})} = 1$. Furthermore, since $w_{n_m} \in S_{n_m}^{\perp}$ and $\tilde{v}_{n_m,\ell} \to v_\ell$, we also have $\langle w,v_\ell \rangle_{\Lp{2}(\mu)} = \lim_{m\to \infty} \langle w_{n_m}, \tilde{v}_{n_m,\ell} \rangle_{\Lp{2}(\mu_{n_m})} = 0$ for $1 \leq \ell \leq i-1$, which implies that $w \in S^\perp$. Combining the latter facts about $w$, we estimate as follows: \begin{align}
    \min_{v \in S^\perp , \, \Vert v \Vert_{\Lp{2}(\mu)} = 1} \langle v, \mathcal{L}^{(q)} v \rangle_{\Lp{2}(\mu)} &\leq \langle w, \mathcal{L}^{(q)} w \rangle_{\Lp{2}(\mu)} \notag \\
    &\leq \liminf_{m \to \infty} \langle w_{n_m}, \mathcal{L}^{(q)}_{n_m} w_{n_m} \rangle_{\Lp{2}(\mu_{n_m})} \label{eq:convergenceEigenpairs:liminfInequality} \\
    & = \lim_{n \to \infty} \min_{ v \in S_n^{\perp}, \, \Vert v \Vert_{\Lp{2}(\mu_n)} =1 } \langle \mathcal{L}_n^{(q)} v,v \rangle_{\Lp{2}(\mu_n)} \notag \\
    &\leq \liminf_{n \to \infty} \sup_{\bar{S} \in \Sigma_{n,i-1}} \,  \min_{ v \in \bar{S}^{\perp}, \, \Vert v \Vert_{\Lp{2}(\mu_n)} =1 } \langle \mathcal{L}_n^{(q)} v,v \rangle_{\Lp{2}(\mu_n)} \notag\\
    &= \liminf_{n \to \infty} \beta_{n,i} \label{eq:convergenceEigenpairs:COurantFisher2}
\end{align}
where we used the $\liminf$-inequality of \cite[Theorem 3.5]{weihs2025Hypergraphs} for \eqref{eq:convergenceEigenpairs:liminfInequality} and \eqref{eq:convergenceEigenpairs:discreteCourantFisher} for \eqref{eq:convergenceEigenpairs:COurantFisher2}. Finally, taking the supremum of all $S \in \Sigma_{i-1}$ in \eqref{eq:convergenceEigenpairs:liminfIdentity}, applying the Courant-Fisher characterization to the self-adjoint and positive semi-definite operator $\mathcal{L}^{(q)}$ and using Lemma \ref{lem:eigenpairs}, we obtain \begin{equation} \label{eq:convergenceEigenpairs:lowerBoundEigenpairs}
    \sum_{k=1}^q \lambda_k \beta_i^{p_k} = \sup_{S \in \Sigma_{i-1}} \min_{v \in S^\perp, \, \Vert v \Vert_{\Lp{2}(\mu_n)} = 1} \langle v, \mathcal{L}^{(q)} v \rangle_{\Lp{2}(\mu)} \leq \liminf_{n \to \infty}  \beta_{n,i}.
\end{equation}

\noindent \textit{Proof of the upper bound.} \hspace{0.2cm} We now derive the corresponding upper bound \begin{equation} \label{eq:convergenceEigenpairs:upperBoundEigenpairs}
    \limsup_{n \to \infty} \beta_{n,i} \leq \sum_{k=1}^q \lambda_k \beta_i^{p_k}.
\end{equation}
We define $S_n \in \Sigma_{n,i-1}$ to be the span of the orthonormal set $(\psi_{n,1},\dots,\psi_{n,i-1})$. By~\cite[Max-min theorem]{Chavel}, we have \[
\beta_{n,i} = \min_{v \in S_n^\perp, \, \Vert v \Vert_{\Lp{2}(\mu_n)} = 1} \langle \mathcal{L}_n^{(q)} v,v \rangle_{\Lp{2}(\mu_n)}
\]
and, similarly to the above, without loss of generality, let us assume that $\limsup_{n \to \infty} \beta_{n,i} = \lim_{n \to \infty} \beta_{n,i}$. 

By the induction hypothesis, for each \( 1 \leq \ell \leq i-1 \), we have the convergence of eigenvalues and hence
\begin{equation*}
\lim_{n \to \infty} \beta_{n,\ell} = \lim_{n \to \infty} \langle \mathcal{L}_n^{(q)} \psi_{n,\ell}, \psi_{n,\ell} \rangle_{\Lp{2}(\mu_n)} =  \beta_\ell < \infty.
\end{equation*}
This uniform boundedness implies that $\sup_n \langle \mathcal{L}_n^{(q)} \psi_{n,\ell}, \psi_{n,\ell} \rangle_{\Lp{2}(\mu_n)} < \infty$ for $1 \leq \ell \leq i-1$. We use the same compactness argument as above and a diagonal argument to obtain subsequences - which, to lighten notation, we do not relabel -  $(\mu_{n},\psi_{n,\ell})$ converging to $(\mu,h_\ell)$ in $\TLp{2}$ for some $h_\ell \in \Lp{2}(\mu)$. Moreover, by \cite[Proposition 2.6]{GARCIATRILLOS2018239} and recalling that $\langle \psi_{n,\ell},\psi_{n,j} \rangle_{\Lp{2}(\mu_{n})} = 0$, we have orthonormality in the limit
\[
\langle h_\ell, h_j \rangle_{\Lp{2}(\mu)} = \lim_{n \to \infty} \langle \psi_{n,\ell}, \psi_{n,j} \rangle_{\Lp{2}(\mu_{n})} = 0\]
for $\ell \neq j$ as well as \[
\|h_\ell\|_\rho = \lim_{n \to \infty} \|\psi_{n,\ell}\|_{\mu_{n}} = 1
\]
for $1 \leq \ell \leq i-1$.
Let \( S \) be the set spanned by \( \{ h_1, \dots, h_{i-1} \} \). In particular, the above implies that \( S \in \Sigma_{k-1} \). We also consider \( w \in S^\perp \) such that \( \|w\|_{\Lp{2}(\mu)} = 1 \) and
\begin{equation}
\langle w, \mathcal{L}^{(q)} w \rangle_{\Lp{2}(\mu)} = \min_{v \in S^\perp, \, \|v\|_{\Lp{2}(\mu)} = 1} \langle w, \mathcal{L}^{(q)} w \rangle_{\Lp{2}(\mu)} \leq \sum_{k=1}^q \lambda_k \beta_i^{p_k},
\end{equation}
where the last inequality follows from the Courant-Fisher characterization for $\mathcal{L}^{(q)}$ and Lemma \ref{lem:eigenpairs}.

By the $\limsup$-inequality in \cite[Theorem 3.5]{weihs2025Hypergraphs}, we obtain $w_n \in \Lp{2}(\mu_n)$ such that $(\mu_n,w_n) \to (\mu,w)$ in $\TLp{2}(\Omega)$ and $\limsup_{n \to \infty} \langle \mathcal{L}_n^{(q)} w_n,w_n \rangle_{\Lp{2}(\mu_n)} \leq \langle w, \mathcal{L}^{(q)} w \rangle_{\Lp{2}(\mu)}.$ Let us define the projection of \( w_n \) onto the orthogonal complement of \( S_n \) as
\[
\tilde{w}_n := w_n - \sum_{\ell=1}^{i-1} \langle w_n, \psi_{n,\ell} \rangle_{\Lp{2}(\mu_n)} \psi_{n,\ell}.
\]
By construction, \( \tilde{w}_n \in {S_n}^\perp \). Moreover, from \cite[Proposition 2.6]{GARCIATRILLOS2018239}, we have 
\(\langle w_n, \psi_{n,\ell} \rangle_{\Lp{2}(\mu_n)} \to \langle w, h_\ell \rangle_{\Lp{2}(\mu)}  = 0\) (since $w \in S^\perp$) as \(n \to \infty\) for all \( 1 \leq \ell \leq i-1\), 
and hence, it is straight-forward to check that \((\mu_n,\tilde{w}_n) \to (\mu,w)\) in $\TLp{2}(\Omega)$.

We next compute the energy of \(\tilde{w}_n\):
\begin{align*}
\langle \mathcal{L}_n^{(q)} \tilde{w}_n, \tilde{w}_n \rangle_{\Lp{2}(\mu_n)} 
&= \left\langle \mathcal{L}_n^{(q)} \left( w_n - \sum_{\ell=1}^{i-1} \langle w_n, \psi_{n,\ell} \rangle_{\Lp{2}(\mu_n)} \psi_{n,\ell} \right), 
w_n - \sum_{m=1}^{i-1} \langle w_n, \psi_{n,m} \rangle_{\Lp{2}(\mu_n)} \psi_{n,m} \right\rangle_{\Lp{2}(\mu_n)} \\
&=\left\langle \mathcal{L}_n^{(q)} w_n - \sum_{\ell=1}^{i-1} \beta_{n,\ell} \langle w_n, \psi_{n,\ell} \rangle_{\Lp{2}(\mu_n)} \psi_{n,\ell} , 
w_n - \sum_{m=1}^{i-1} \langle w_n, \psi_{n,m} \rangle_{\Lp{2}(\mu_n)} \psi_{n,m} \right\rangle_{\Lp{2}(\mu_n)} \\
&= \langle \mathcal{L}_n^{(q)} w_n, w_n \rangle_{\Lp{2}(\mu_n)} 
- 2 \sum_{\ell=1}^{i-1} \beta_{n,\ell} \langle w_n, \psi_{n,\ell} \rangle_{\Lp{2}(\mu_n)}^2 
+ \sum_{\ell=1}^{i-1} \beta_{n,\ell} \langle w_n, \psi_{n,\ell} \rangle_{\Lp{2}(\mu_n)}^2 \\
&= \langle \mathcal{L}_n^{(q)} w_n, w_n \rangle_{\Lp{2}(\mu_n)} 
- \sum_{\ell=1}^{i-1} \beta_{n,\ell} \langle w_n, \psi_{n,\ell} \rangle_{\Lp{2}(\mu_n)}^2.
\end{align*}
This implies 
\begin{align}
\limsup_{n \to \infty}\langle \mathcal{L}_n^{(q)} \tilde{w}_n, \tilde{w}_n \rangle_{\Lp{2}(\mu_n)} & \leq \limsup_{n \to \infty} \langle \mathcal{L}_n^{(q)} w_n, w_n \rangle_{\Lp{2}(\mu_n)} 
- \sum_{\ell=1}^{i-1} \beta_{n,\ell} \langle w_n, \psi_{n,\ell} \rangle_{\Lp{2}(\mu_n)}^2 \notag \\
 & \leq \langle w, \mathcal{L}^{(q)} w \rangle_{\Lp{2}(\mu)} \label{eq:convergenceEigenpairs:Inequality1}
\end{align}
where we used the $\limsup$-inequality of \cite[Theorem 3.5]{weihs2025Hypergraphs} and the fact that \(\beta_{n,\ell}\langle w_n, \psi_{n,\ell} \rangle_{\Lp{2}(\mu_n)}^2 \geq  0\) for \eqref{eq:convergenceEigenpairs:Inequality1}.

Since \((\mu_n,\tilde{w}_n) \to (\mu,w)\) in \(\TLp{2}(\Omega)\) and \(\|w\|_{\Lp{2}(\mu)} = 1\), 
\cite[Proposition 2.6]{GARCIATRILLOS2018239} implies
\(
\lim_{n \to \infty} \|\tilde{w}_n\|_{\Lp{2}(\mu_n)} = 1
\)
and we can thus define
\[
\bar{w}_n := \frac{\tilde{w}_n}{\|\tilde{w}_n\|_{\Lp{2}(\mu_n)}}.
\]

We conclude by estimating as follows:
\begin{align}
\lim_{n \to \infty} \beta_{n,i}
&= \lim_{n \to \infty} 
\min_{v \in S_n^\perp, \, \Vert v \Vert_{\Lp{2}(\mu_n) = 1}} \langle \mathcal{L}_n^{(q)} v,v \rangle_{\Lp{2}(\mu_n)} \notag \\
&\leq \limsup_{n \to \infty} \langle \mathcal{L}_n^{(q)} \bar{w}_n, \bar{w}_n \rangle_{\Lp{2}(\mu_n)} \label{eq:convergenceEigenpairs:Inequality2} \\
&\leq \langle w, \mathcal{L}^{(q)} w \rangle_{\Lp{2}(\mu)} \label{eq:convergenceEigenpairs:Inequality3} \\
&\leq \sum_{k=1}^q \lambda_k \beta_i^{p_k} \label{eq:convergenceEigenpairs:Inequality4}
\end{align}
where we used the fact that $\Vert \bar{w}_n \Vert_{\Lp{2}(\mu_n)} = 1$ and $\bar{w}_n \in S_n^\perp$ for \eqref{eq:convergenceEigenpairs:Inequality2}, \eqref{eq:convergenceEigenpairs:Inequality1} and the fact that $\lim_{n \to \infty} \|\tilde{w}_n\|_{\Lp{2}(\mu_n)} = 1$ for \eqref{eq:convergenceEigenpairs:Inequality3}, as well as the facts that $w \in S^{\perp}$ and $\Vert w \Vert_{\Lp{2}(\mu)}$, the Courant-Fisher characterization and Lemma \ref{lem:eigenpairs} for \eqref{eq:convergenceEigenpairs:Inequality4}. This proves \eqref{eq:convergenceEigenpairs:upperBoundEigenpairs}.
\vspace{\baselineskip}

By combining \eqref{eq:convergenceEigenpairs:lowerBoundEigenpairs} and \eqref{eq:convergenceEigenpairs:upperBoundEigenpairs}, we get the convergence of eigenvalues. We now consider the convergence of eigenfunctions and proceed similarly by induction. 
\vspace{\baselineskip}

Before starting, we introduce some additional notation. We denote the ordered eigenvalues of 
$\mathcal{L}^{(q)}$ by $\gamma_i$ (which are equal to $\sum_{k=1}^q \lambda_k \beta_i^{p_k}$ by Lemma \ref{lem:eigenpairs}). We then write $\bar{\gamma}_{i}$ for the distinct eigenvalues. Furthermore, for each $i \in \mathbb{N}$, let $s(i)$ denote the multiplicity of the eigenvalue $\bar{\gamma}_i$, and let $\hat{i} \in \mathbb{N}$ be such that
\[
\bar{\gamma}_i = \gamma_{\hat{i}+1} = \cdots = \gamma_{\hat{i}+s(i)}.
\]
We define $E_i$ as the eigenspace of $\mathcal{L}^{(q)}$ in $\Lp{2}(\mu)$ corresponding to $\bar{\gamma}_i$. For $n$ sufficiently large, let $E_{n,i} \subset \mathbb{R}^n$ be the subspace spanned by the eigenvectors of $\mathcal{L}_n^{(q)}$ associated with the eigenvalues 
\(
\beta_{n,\hat{i}+1}, \dots, \beta_{n,\hat{i}+s(i)}.
\)
Due to the eigenvalue convergence results derived above, we have:
\begin{equation}
\lim_{n \rightarrow \infty} \dim(E_{n,i}) = \dim(E_i) = s(i).
\label{convergenceDimension}
\end{equation}
We denote by $\mathrm{Proj}_i : \Lp{2}(\mu) \mapsto \Lp{2}(\mu)$ the orthogonal projection (with respect to the inner product $\langle \cdot, \cdot \rangle_{\Lp{2}(\mu)}$) onto $E_i$. Analogously, for all sufficiently large $n$, we denote by $\mathrm{Proj}_{n,i} : \Lp{2}(\mu_n) \mapsto \Lp{2}(\mu_n)$ the orthogonal projection (with respect to the inner product $\langle \cdot, \cdot \rangle_{\Lp{2}(\mu_n)}$) onto the subspace spanned by $E_{n,i}$. 

The following induction will prove that not only eigenfunctions converge, but also the projections, i.e. if $(\mu_n,v_n) \to (\mu,v)$ in $\TLp{2}(\Omega)$, then $\mathrm{Proj}_{n,i}(v_n) \to \mathrm{Proj}_{i}(v)$ in $\TLp{2}(\Omega)$. 

\paragraph{Base case $i=1$.} We covered the convergence of $\psi_{n,1}$ to $\psi_1$ in $\TLp{2}(\Omega)$ in the base case of the convergence of eigenvalues. Regarding the projections, assume that $(\mu_n,v_n) \to (\mu,v)$ in $\TLp{2}(\Omega)$. Since by Assumption \ref{ass:Main:Ass:S2} $\Omega$ is connected, the first eigenvalue $\overline{\gamma}_1 = 0$ is simple, and $\mathrm{Proj}_1(v)$ corresponds to the constant function equal to the mean of $v$ with respect to $\mu$, that is,
\(
\mathrm{Proj}_1(v) = \langle v, \one \rangle_{\Lp{2}(\mu)}.
\)
Similarly, convergence of eigenvalue multiplicities~\eqref{convergenceDimension} implies that for $n$ large enough, $E_{n,i}$ is one-dimensional. In this case, $\mathrm{Proj}_{n,1}(v_n)$ is the constant vector equal to $\langle v_n, \mathbf{1} \rangle_{\Lp{2}(\mu_n)}$. By \cite[Proposition 2.6]{GARCIATRILLOS2018239}, we have
\[
\lim_{n \to \infty} \langle v_n, 1 \rangle_{\Lp{2}(\mu_n)} = \langle v, 1 \rangle_{\Lp{2}(\mu)},
\]
establishing the convergence of the projections. 

\paragraph{Induction step.} Now, suppose that $\psi_{n,\ell} \to \psi_\ell$ in $\TLp{2}(\Omega)$ and that $\mathrm{Proj}_{n,\ell}$ converges to $\mathrm{Proj}_\ell$ for all $\ell \leq i-1$. 

Let $j \in \{\hat{i}+1, \dots,\hat{i} + s(i) \}$ and consider $\psi_{n,j}$. From the convergence of eigenvalues, we have \[
\lim_{n \to \infty} \langle \mathcal{L}_n^{(q)} \psi_{n,j}, \psi_{n,j} \rangle_{\Lp{2}(\mu_n)} = \lim_{n \to \infty} \beta_{n,j} = \gamma_j < \infty.
\]
In particular, $\sup_n \langle \mathcal{L}_n^{(q)} \psi_{n,j}, \psi_{n,j} \rangle_{\Lp{2}(\mu_n)} < \infty$ and we can apply the same compactness result as previously, to obtain a subsequence $(\mu_{n_m},\psi_{n_m,j}) \to (\mu,h_j)$ for some $h_j \in \Lp{2}(\mu)$. 

We note that $\mathrm{Proj}_{n,\ell}(\psi_{n,j}) = 0$ for all $1 \leq \ell \leq i-1$ (since $\psi_{n,j}$ is associated with the eigenvalue $\beta_{n,i}$)) and therefore, by the induction hypothesis, $\mathrm{Proj}_{\ell}(h_j) = 0$ for all $1 \leq \ell \leq i-1$. This allows us to deduce that (using the spectral decomposition of $\mathcal{L}^{(q)}$) \[
\langle \mathcal{L}^{(q)} h_j, h_j \rangle_{\Lp{2}(\mu)} = \sum_{r=i}^\infty \bar{\gamma}_r \Vert \mathrm{Proj}_r(h_j) \Vert_{\Lp{2}(\mu)}^2 \geq \bar{\gamma}_i \sum_{r=i}^\infty \Vert \mathrm{Proj}_r(h_j) \Vert_{\Lp{2}(\mu)}^2 = \bar{\gamma}_i \Vert h_j \Vert_{\Lp{2}(\mu)}^2.
\]
By \cite[Proposition 2.6]{GARCIATRILLOS2018239} and the fact that $\Vert \psi_{n,j} \Vert_{\Lp{2}(\mu_n)} = 1$, we also have $\Vert h_j \Vert_{\Lp{2}(\mu)} = 1$, implying that \begin{equation} \label{eq:upperBound}
    \langle \mathcal{L}^{(q)} h_j, h_j \rangle_{\Lp{2}(\mu)} \geq \bar{\gamma}_i. 
\end{equation}

By using the convergence of eigenvalues, the $\liminf$-inequality of \cite[Theorem 3.5]{weihs2025Hypergraphs} and \eqref{eq:upperBound}, we obtain \begin{equation*} 
    \bar{\gamma}_i = \gamma_j = \lim_{n \to \infty} \beta_{n,j} = \liminf_{n \to \infty} \langle \mathcal{L}_n^{(q)} \psi_{n,j}, \psi_{n,j} \rangle_{\Lp{2}(\mu_n)} \geq \langle \mathcal{L}^{(q)} h_j, h_j \rangle_{\Lp{2}(\mu)} \geq \bar{\gamma}_i.
\end{equation*}
This implies that $\langle \mathcal{L}^{(q)} h_j, h_j \rangle_{\Lp{2}(\mu)} = \bar{\gamma}_i$ and~\eqref{eq:upperBound} also allows us to deduce that $\mathrm{Proj}_{r}(h_j) = 0$ for all $r \neq i$. 
Indeed, suppose that $\|\mathrm{Proj}_r(h_j)\|_{\Lp{2}(\mu)}>0$ for $r>i$. Then (following the same spectral decomposition preceding~\eqref{eq:upperBound}), we have $\langle \mathcal{L}^{(q)} h_j, h_j \rangle_{\Lp{2}(\mu)}>\bar{\gamma}_i$ which is a contradiction.
We conclude that $h_j$ is an eigenvector of $\mathcal{L}^{(q)}$ with eigenvalue $\bar{\gamma}_i$, establishing the convergence of eigenvectors. 

It only remains to prove the convergence of $\mathrm{Proj}_{n,i}$ to $\mathrm{Proj}_{i}$. Consider $(\mu_n,w_n) \to (\mu,w)$ in $\TLp{2}(\Omega)$. According to~\eqref{convergenceDimension}, for sufficiently large \( n \), \(\dim(E_{n,i}) \) equals \( s(i) \). We can therefore choose an orthonormal basis \( \{v_{n,1}, \dots, v_{n,s(i)}\} \) of \( E_{n,i} \) with respect to the inner product \( \langle \cdot, \cdot \rangle_{\Lp{2}(\mu_n)} \), where each \( v_{n,j} \) is an eigenvector of \( \mathcal{L}_{n}^{(q)} \) corresponding to the eigenvalue \( \beta_{n,\hat{i} + j} \).

Similarly to the above, for each \( j = 1, \dots, s(i) \), the sequence \( \{v_{n,j}\}_{n \in \mathbb{N}} \) is precompact in $\TLp{2}$ and - without relabeling the subsequences - we may assume that
\(
(\mu_n,v_{n,j}) \to (\mu,v_j) 
\)
in $\TLp{2}(\Omega)$ for some \( v_j \in \Lp{2}(\mu) \). From \cite[Proposition 2.6]{GARCIATRILLOS2018239}, it follows that each \( v_j \) satisfies \( \|v_j\|_{\Lp{2}(\mu_j)} = 1 \), and that the family \( \{v_1, \dots, v_{s(i)}\} \) is orthonormal with respect to \( \langle \cdot, \cdot \rangle_{\Lp{2}(\mu)} \). Moreover, by the convergence of eigenvectors, each \( v_j \) lies in the limiting eigenspace \( E_i \), so that \( \{v_1, \dots, v_{s(i)}\} \) forms an orthonormal basis for \( E_i \). Hence, the projection \( \mathrm{Proj}_i \) can be written for $v \in \Lp{2}(\mu)$ as
\[
\mathrm{Proj}_i(v) = \sum_{j=1}^{s(i)} \langle v, v_j \rangle_{\Lp{2}(\mu)} v_j.
\]
On the discrete side, for all large enough \( n \) and $v_n \in \Lp{2}(\mu_n)$, we have
\[
\mathrm{Proj}_{n,i}(v_n) = \sum_{j=1}^{s(i)} \langle v_n, v_{n,j} \rangle_{\Lp{2}(\mu_n)} v_{n,j}.
\]
Now, since $(\mu_n, w_n) \to (\mu,w)$ and $(\mu_n, v_{n,j}) \to (\mu,v_j)$ in $\TLp{2}(\Omega)$, we apply \cite[Proposition 2.6]{GARCIATRILLOS2018239} to conclude.
\end{proof}

\begin{corollary}[$\Gamma$-convergence of quadratic forms] \label{cor:gammaConvergence}
Assume that \ref{ass:Main:Ass:S2}, \ref{ass:Main:Ass:M1}, \ref{ass:Main:Ass:M2}, \ref{ass:Main:Ass:W2}, and \ref{ass:Main:Ass:D1} hold. Let $q \geq 1$, $P = \{p_k\}_{k=1}^q \subseteq \bbR$ with $p_1 \leq \cdots \leq p_q$ and $ E_n = \{\eps_n^{(k)}\}_{k=1}^q$ with $\eps_n^{(1)} > \cdots > \eps_n^{(q)}$. Assume that $\rho \in \Ck{\infty}$. Assume that $\eps_n^{(q)}$ satisfies \ref{ass:Main:Ass:L2}. Then, $\bbP$-a.e., for every $s > 0$, the following holds:\begin{enumerate}
    \item $\langle v_n , \l \mathcal{L}_n^{(q)} \r^s v_n \rangle_{\Lp{2}(\mu_n)}$ $\Gamma$-converges to $\langle v , \l \mathcal{L}^{(q)} \r^{s} v \rangle_{\Lp{2}(\mu)}$;    
    \item If a sequence satisfies $\sup_n \max\{\langle v_n , \l \mathcal{L}_n^{(q)} \r^{s} v_n \rangle_{\Lp{2}(\mu_n)},\Vert v_n \Vert_{\Lp{2}(\mu_n)}\} \leq C$, then there exists a converging subsequence in $\TLp{2}(\Omega)$.
\end{enumerate}
\end{corollary}

\begin{proof}
We want to proceed as in the proof of \cite[Theorem 2]{Stuart} where the analogous statement is proven for $\Delta_{n,\eps_n}^{s}$ (see also the proof of Proposition \ref{prop:truncated} for a similar argument). In particular, the authors mainly rely on the fact that the eigenpairs of $\Delta_{n,\eps_n}$ converge to the eigenpairs of $\Delta_\rho$. In our case, by Proposition \ref{prop:convergenceEigenpairs}, the eigenpairs of $\mathcal{L}_{n}^{(q)}$ converge to eigenpairs of $\mathcal{L}^{(q)}$ and we can therefore apply the same argument to conclude. 
\end{proof}

\begin{proposition}[Bounded energies] \label{prop:boundedEnergies}
    Assume that \ref{ass:Main:Ass:S2}, \ref{ass:Main:Ass:M1}, \ref{ass:Main:Ass:M2}, \ref{ass:Main:Ass:W2} and \ref{ass:Main:Ass:D1} hold. Let $q \geq 1$, $P = \{p_k\}_{k=1}^q \subseteq \bbR$ with $p_1 \leq \cdots \leq p_q$ and $ E_n = \{\eps_n^{(k)}\}_{k=1}^q$ with $\eps_n^{(1)} > \cdots > \eps_n^{(q)}$. Assume that $\rho \in \Ck{\infty}$ and that $\eps_n^{(q)}$ satisfies \begin{equation} \label{eq:boundedEnergies:rates}
    \lim_{n \to \infty} \frac{\log(n)}{n \l\eps_n^{(q)}\r^{d+4p_q}} = 0.
    \end{equation}
    For a continuous function $v$, let $v_n$ denote its restriction to $\Omega_n$. For any $k \in \bbN$ and $u \in \Ck{\infty}(\Omega)$, $\bbP$-a.e., there exists a constant $C(k,u) > 0$ such that \[
    \sup_{n} \langle v_n, \l \cL_n^{(q)} \r^{(2k)}    v_n \rangle_{\Lp{2}(\mu_n)} \leq C(k,u).
    \]
\end{proposition}

\begin{proof}
We want to proceed as in the proof of \cite[Lemma 4.19]{weihs2023consistency} where the results was shown for the energy $\langle v_n, \Delta_{n,\eps_n} v_n \rangle_{\Lp{2}(\mu_n)}$. In particular, the proof relies on the following elements: \begin{enumerate}
    \item $\Delta_{n,\eps_n}$ and $\Delta_\rho$ are self-adjoint and positive semi-definite;
    \item $\langle v_n, \Delta_{n,\eps_n}^{s} v_n \rangle_{\Lp{2}(\mu_n)}$ $\Gamma$-converges to $\langle v_n, \Delta_{\rho}^{s} v_n \rangle_{\Lp{2}(\mu_n)}$;
     \item There exists a constant $C(u)$ such that $\Vert \Delta_\rho(u) - \Delta_{n,\eps_n}(u) \Vert_{\Lp{2}(\mu_n)} \leq C(u) \eps_n \to 0$ \cite[Theorem 2.8]{trillos2022rates}.
\end{enumerate}
For our energy, $\langle v_n, \cL_n^{(q)} v_n \rangle_{\Lp{2}(\mu_n)}$, we have: \begin{enumerate}
    \item $\mathcal{L}_n^{(q)}$ and $\mathcal{L}^{(q)}$ are positive semi-definite and self-adjoint as shown in Proposition \ref{prop:convergenceEigenpairs}.
    \item the fact that $\langle v_n , \l \mathcal{L}_n^{(q)} \r^s v_n \rangle_{\Lp{2}(\mu_n)}$ $\Gamma$-converges to $\langle v , \l \mathcal{L}^{(q)} \r^s v \rangle_{\Lp{2}(\mu)}$ is shown in Corollary \ref{cor:gammaConvergence}.
    \item the fact that \(\Vert \mathcal{L}_n^{(q)}(u) - \mathcal{L}^{(q)}(u) \Vert_{\Lp{2}(\mu)} \leq C(u) \sum_{k=1}^q \lambda_k \eps_n^{(k)} \to 0 \). Specifically, let $E_k$ be the set such that \cite[Theorem 2.8]{trillos2022rates} holds for $\eps_n^{(k)}$: we can apply the latter result since the assumptions that $p_1 \leq \cdots \leq p_q$ and $\eps_n^{(1)} > \cdots \eps_n^{(q)}$ imply that \eqref{eq:boundedEnergies:rates} holds for any $1 \leq k \leq q$. We know from the proof of Proposition \ref{prop:ratesDiscreteContinuum} that, for any $\alpha > 1$, there exists $0<c<C$ and $\eps_0>0$ such that $\bbP\l \cap_{k=1}^q E_k \r \geq 1 - Cn^{-\alpha}-Cne^{-cn\l \eps_n^{(q)}\r^{d+4p_q}}$ as long as $\eps_0 \geq \eps_n^{(1)} > \dots > \eps_n^{(q)}$. On this intersection, we have \begin{align}
        \Vert \mathcal{L}_n^{(q)}u - \mathcal{L}^{(q)}u \Vert_{\Lp{2}(\mu)} &\leq \sum_{k=1}^q \lambda_k \left\Vert \l \Delta_{n,\eps_n^{(k)}}^{p_k} - \Delta_\rho^{p_k} \r u \right\Vert_{\Lp{2}(\mu)} \notag \\
        &\leq C \sum_{k=1}^q \lambda_k \eps_n^{(k)} \l \Vert u \Vert_{\Ck{2p_k + 1}(\Omega)} +1 \r \notag \\
        &= C(u) \sum_{k=1}^q \lambda_k \eps_n^{(k)} \notag
    \end{align}
where we used \cite[Theorem 2.8]{Trillos} for the inequality. The last term tends to 0 and, by applying the Borel-Cantelli lemma with \eqref{eq:boundedEnergies:rates}, we can show that this convergence holds $\bbP$-a.e..

\end{enumerate}
We therefore apply the same argument as in \cite[Lemma 4.19]{weihs2023consistency} to deduce the claim.
\end{proof}

\begin{proof}[Proof of Theorem \ref{thm:truncatedEnergies}]
We are going to proceed as in the proof of Proposition \ref{prop:truncated} where the same result is proven for the truncated energy of a single Laplacian matrix $\Delta_{n,\eps_n}^{s}$. If we replace the latter by $\mathcal{L}_n^{(q)}$, \cite[Lemma 4.19]{weihs2023consistency} by Proposition \ref{prop:boundedEnergies}, the convergence of eigenpairs by Proposition \ref{prop:convergenceEigenpairs} and \cite[Proposition 4.21]{weihs2023consistency} by Corollary \ref{cor:gammaConvergence} the same proof applies. 
\end{proof}

\subsection{Non-geometric setting}

\begin{proof}[Proof of Proposition \ref{prop:laplacian}]
    We start by showing that the set of matrices $$\mathcal{M} = \left\{ M \in \mathbb{R} \, | \, (M)_{ii} = -\sum_{j\neq i} (M)_{ij}\right\}$$ is closed under matrix product, and addition and multiplication by scalars. Closures under addition and multiplication by scalars are straight-forward to check. Let $P,Q \in \mathcal{M}$ and consider \begin{align*}
        \sum_{j\neq i} (PQ)_{ij} &= \sum_{j \neq i} \sum_{k=1}^n (P)_{ik} (Q)_{kj} = \sum_{k=1}^n (P)_{ik} \sum_{j\neq i} (Q)_{kj} = -\sum_{k=1}^n (P)_{ik} (Q)_{ki} = - (PQ)_{ii} 
    \end{align*}
    since, by assumption on $Q$, $(Q)_{kk} = -\sum_{j\neq k} (Q)_{kj}$ implying that $(Q)_{kk} = -\sum_{j\neq i} (Q)_{kj} - (Q)_{ki} + (Q)_{kk}$ or $(Q)_{ki} = -\sum_{j\neq i} (Q)_{kj}$. This implies that $PQ \in \mathcal{M}$.

    Now, by definition, any Laplacian matrix $L$ is in $\mathcal{M}$ and, by the above, so is $L^k$ for any $k \in \mathbb{N}$. Furthermore, since $L$ is symmetric, $L^k$ is too. This implies that  $\mathcal{L}^{(q)}_{\mathrm{dis}} = \sum_{k=1}^q \lambda_k (L^{(k)})^k$ is symmetric and in $\mathcal{M}$.

    Let us now define a graph $\tilde{G} = (V,W)$ where $V$ is the same set of vertices used to define $L_n$ (in our case, this corresponds to $\Omega_n$ but our proof holds for any set of vertices) and $W$ is the symmetric matrix with entries $(W)_{ij} = -(\mathcal{L}^{(q)}_{\mathrm{dis}})_{ij}$ for $i\neq j$ and $(W)_{ii}$ can be arbitrarily chosen. Then, for the (diagonal) degree matrix $D$ with entries $(D_{ii}) = \sum_{j=1}^n (W)_{ij}$, the Laplacian of $\tilde{G}$ defined as $D - W$ is equal to $\mathcal{L}^{(q)}_{\mathrm{dis}}$. Finally, since $\mathcal{L}^{(q)}_\mathrm{dis}$ is a sum of positive semi-definite matrices, it is too and, therefore \eqref{eq:discussion:higherOrder} is a quadratic form.
\end{proof}

\section{Numerical experiments} \label{sec:numerical}

We present experiments illustrating HOHL’s flexibility and effectiveness. First, we show it can replace Laplace learning in active learning. Then, we apply HOHL to hypergraph-structured datasets, observing consistent gains over standard baselines.

\subsection{Active learning}

Optimization problems of the form \( \arg\min_{v} J(v) + \Psi(v, y) \), where \( J \) is a regularizer and \( \Psi \) enforces label fidelity, admit a Bayesian interpretation. Specifically, with a prior $\mu_0(v)$ proportional to $e^{-J(v)}$, a likelihood $\mu_1(y|v)$ proportional to $e^{-\Psi(v,y)}$, we obtain a posterior $\mu_2(v|y)$ that is proportional to $e^{-J(v) - \Psi(v,y)}$ implying that the maximum à posteriori estimator of $\mu_2$ is the minimizer of $J(v) + \Psi(v,y)$. This formulation enables uncertainty quantification and active learning, see~\cite{LapRefActiveLearning,ji12,Miller}.

Active learning is an iterative learning paradigm in which the most informative data points to label are selected at each iteration by an acquisition function, rather than passively relying on a fixed labeled dataset. The goal is to achieve high prediction accuracy with as few labeled examples as possible, making it especially valuable in scenarios where labeling is expensive or time-consuming. Within the Bayesian framework, uncertainty estimates derived from the posterior distribution $\mu_2(v \mid y)$ can guide this selection process—for instance, by querying points where the predictive variance is high. This uncertainty-aware strategy helps prioritize data that is expected to most improve the model.

In graph-based approaches, the regularizer is often chosen as \( J(v) = \langle v, L^s v \rangle_n \) for some \( s > 0 \), where \( L \) is the graph Laplacian~\cite{LapRef,Stuart,Miller,TutSpec}. This choice induces a Gaussian prior over functions, leveraging the fact that \( L \) is a symmetric and positive semi-definite matrix~\cite{TutSpec}. By Proposition~\ref{prop:laplacian}, an analogous construction is possible on hypergraphs using the operator \( \mathcal{L}^{(q)}_{\mathrm{Dis}} \), allowing us to define Gaussian priors in the hypergraph setting as well. This introduces higher-order structure into the prior, effectively encoding regularity up to the \( p_q \)-th derivative~\cite{weihs2025Hypergraphs}.

We evaluate this approach within an active learning setting, 
employing uncertainty sampling as the acquisition function~\cite{settles2012active}. Experiments are conducted on the MNIST~\cite{LeCun1998} and FashionMNIST~\cite{xiao2017fashionmnist} datasets. Since both datasets can be embedded in metric spaces, we approximate HOHL by \eqref{eq:multiscale} and, following standard practice to speed-up computation on large datasets ~\cite{poissonLearning}, we construct $k$-nearest neighbor graphs instead, replacing the scale sequence \( \varepsilon^{(\ell)} \) in Eq.~\eqref{eq:multiscale} with neighborhood sizes \( k^{(1)} \geq \dots \geq k^{(q)} \). Edge weights are defined by
\(
w_{k^{(\ell)},ij} = \exp\left(-\frac{4 \|x_i - x_j\|^2}{d_{k^{(\ell)}}(x_i)^2}\right),
\)
where \( d_{k^{(\ell)}}(x_i) \) is the distance from \( x_i \) to its \( k^{(\ell)} \)-th nearest neighbor. We choose the norm $\Vert \cdot \Vert$ to be the cosine/angular distance. 

We compare Laplacian and HOHL-based priors across 100 trials. As shown in Figure~\ref{fig:active}, HOHL priors yield substantial improvements over graph-based priors, particularly at low label rates where higher-order smoothness improves sample efficiency: with only 100 labeled points (i.e., $0.17\%$ of MNIST and $0.20\%$ of FashionMNIST), on MNIST, accuracy improves from approximately $35\%$ to $75\%$ ($+40$ points), and on FashionMNIST from $35\%$ to $65\%$ ($+30$ points), highlighting HOHL's ability to leverage higher-order structure under severe label constraints.

Our results suggest that smoother priors in high-density regions enable more informative sampling in early rounds, which is critical when label budgets are small.

\begin{figure}[htbp]
  \centering
  \begin{minipage}[b]{0.48\linewidth}
    \centering
    \includegraphics[width=\linewidth]{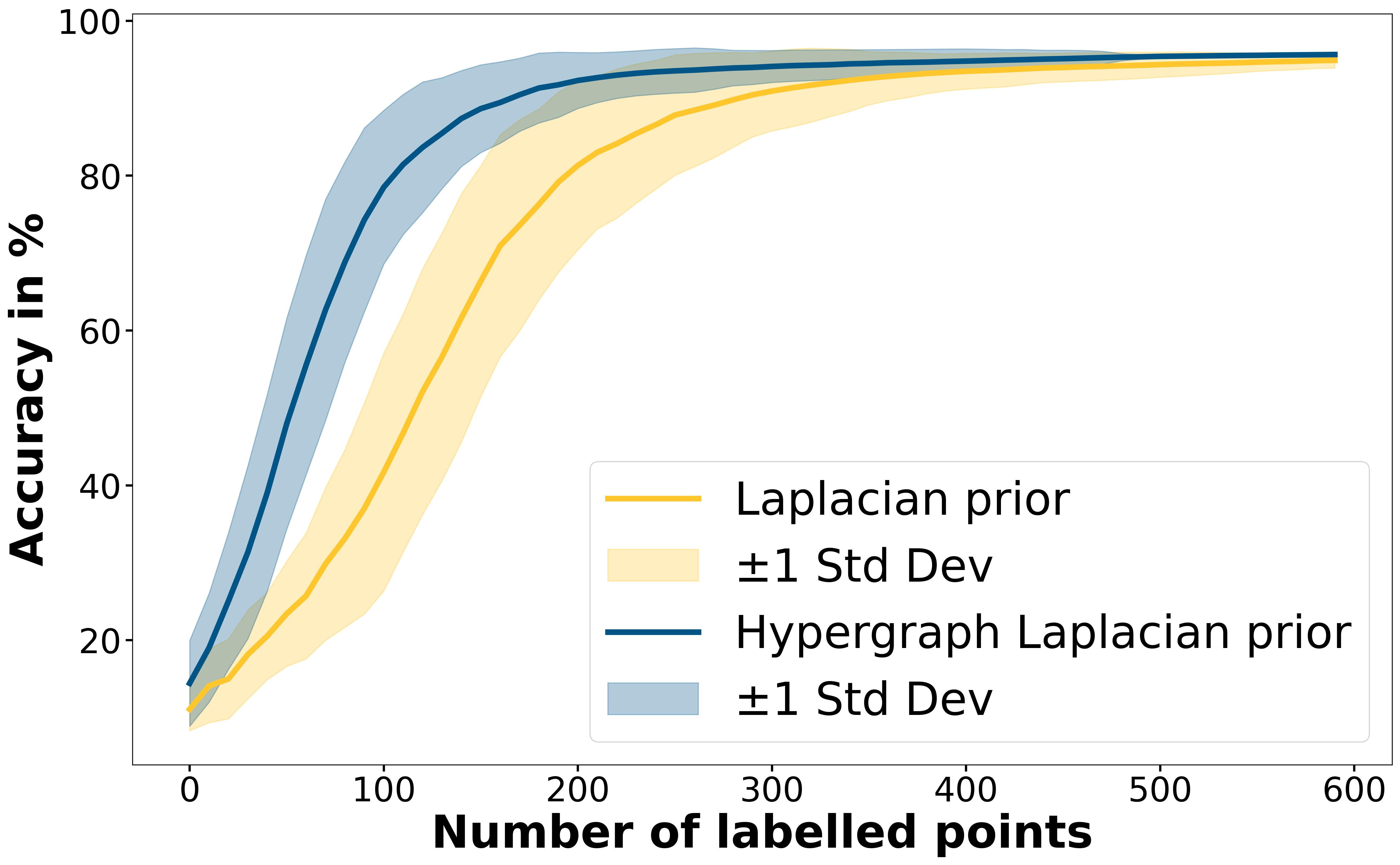}
  \end{minipage}
  \hfill
  \begin{minipage}[b]{0.48\linewidth}
    \centering
        \includegraphics[width=\linewidth]{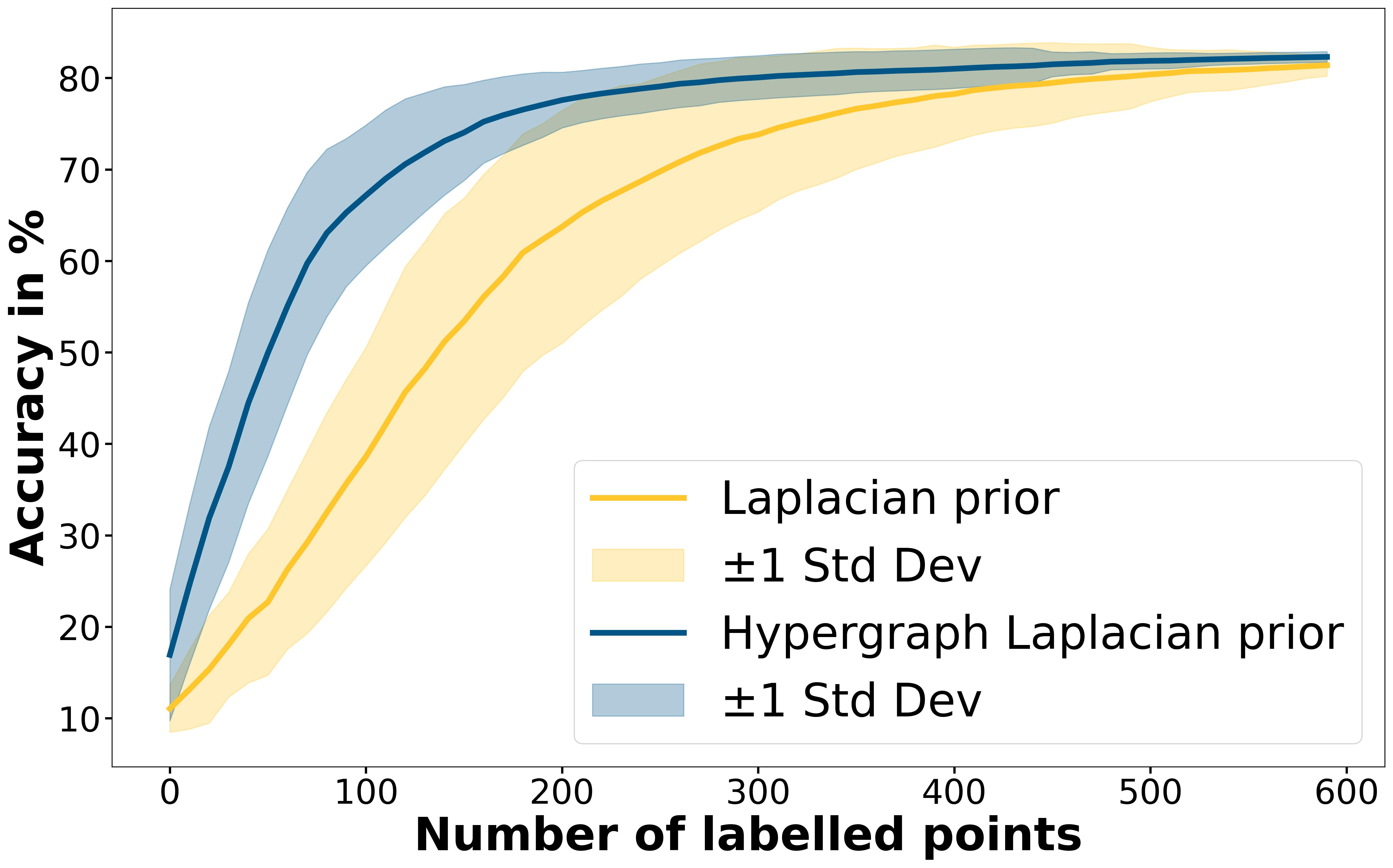}
  \end{minipage}
  \caption{Accuracy in active learning using Laplacian and HOHL priors. We use $k^{(1)} = 50$, $k^{(2)} = 30$, $\lambda_1 = 1$, $\lambda_2 = 4$, $p_1 = 1$, $p_2 = 2$. Left: MNIST dataset. Right: fashionMNIST dataset.}
  \label{fig:active}
\end{figure}

\subsection{HOHL for semi-supervised learning in non-geometric setting}

We consider the Zoo \cite{Dua:2017}, Mushroom \cite{Dua:2017}, Cora \cite{McCallum2000} and Citeseer \cite{Sen2008} datasets. The hyperedges are created following the procedure detailed in Section \ref{sec:main:nongeometric}. To ease notation, in this section, we will write $\mathcal{L}^{(q)}$ instead of $\mathcal{L}^{(q)}_{\mathrm{Dis}}$.

Using Algorithm \ref{alg:hohl}, we consider the HOHL energy \eqref{eq:discussion:higherOrder} for semi-supervised learning with \begin{itemize}
    \item $1 \leq q \leq 4$;
    \item powers $p_\ell = \ell$;
    \item regular growth coefficients (RC) $\lambda_\ell = \ell$ or quickly growing coefficients (QC) $\lambda_\ell = \ell^2$ (QC). 
\end{itemize}
 We compare against Laplace Learning using the clique expansion---chosen over other hypergraph-to-graph reductions for its preservation of the vertex set, see~\cite{clique}---as well as three non-deep hypergraph methods implemented in \cite{hgLearningPractice}: transductive learning from~\cite{scholkopfHyper2006}, hyperedge-weighted transduction from~\cite{hgWeighting}, and dynamic hypergraph learning from~\cite{dynamicHG}. 
 We report mean accuracies and standard deviation in percentages over 100 trials at different labelling rates in Tables \ref{tab:zoo}, \ref{tab:mushroom}, \ref{tab:cora} and \ref{tab:citeseer}. We summarize the terminology used in our experiments in Table \ref{tab:qj:terminology:qonly}. Similar experiments have been performed to test HOHL in the geometric setting in \cite{weihs2025Hypergraphs}.

\begin{table}[H]
\centering
\small
\renewcommand{\arraystretch}{1.5}
\setlength{\tabcolsep}{8pt}
\begin{tabularx}{\linewidth}{>{\bfseries}p{6.5cm} Y}
\toprule
\textbf{Term / Abbreviation} & \textbf{Explanation} \\
\midrule
Aim of experiment & Analysis of HOHL \eqref{eq:discussion:higherOrder} as a function of maximum powers $q$ and coefficients $\lambda_\ell$  \\
\midrule
$\ell$ & Index over scales $1 \leq \ell \leq q$ \\
$q$ & Number of Laplacians $1 \leq q \leq 4$ \\
\midrule
$\lambda_\ell$ & Increasing coefficients: $\lambda_\ell = \ell$ or $\lambda_\ell = \ell^2$ \\
$p_\ell$ & Increasing powers: $p_\ell = \ell$ \\
\midrule
RC & $\lambda_\ell = \ell$ \\
QC & $\lambda_\ell = \ell^2$ \\
$\mathcal{L}^{(q)}$ & HOHL using Algorithm \ref{alg:hohl} for $1 \leq q \leq 4$ \\
\bottomrule
\end{tabularx}
\caption{Terminology used in the $q$-experiments.} 
\label{tab:qj:terminology:qonly}
\end{table}

\begin{figure}
  \centering  \includegraphics[width=\linewidth]{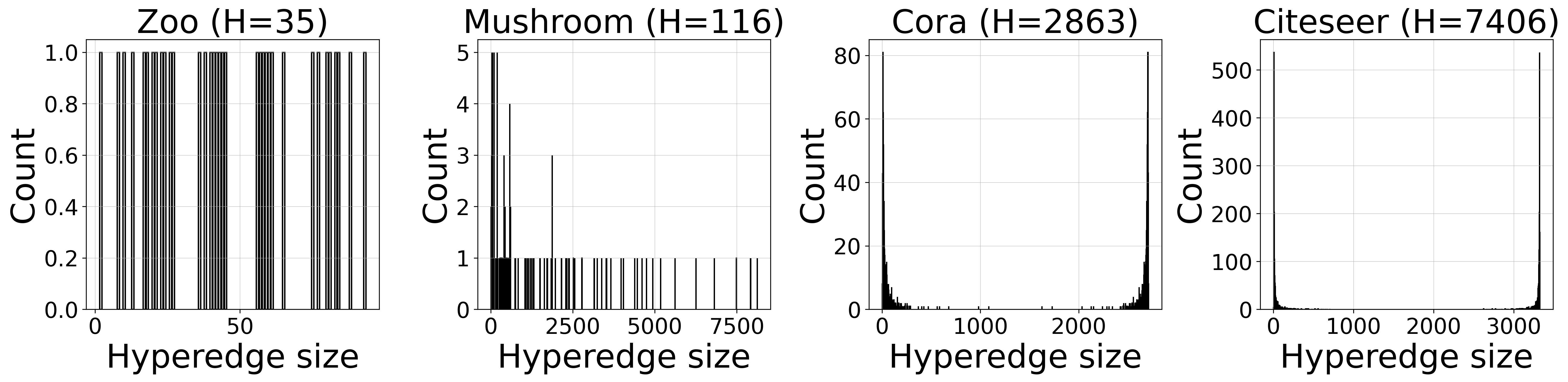} \caption{Hyperedge size distributions for all datasets. Zoo and Mushroom exhibit nearly uniform distributions; Cora and Citeseer are bimodal, with both large and small hyperedges. $H$ denotes the total number of hyperedges in each case.} \label{fig:hyperedge_sizes}
\end{figure}

\begin{table}[htbp]

\begin{small}
\begin{flushleft}
\begin{tabular}{llllllll}
\toprule
\textbf{Rate} 
& \textbf{$\mathcal{L}^{(1)}$}  
& \textbf{$\mathcal{L}^{(2)}$ RC} 
& \textbf{$\mathcal{L}^{(2)}$ QC} 
& \textbf{$\mathcal{L}^{(3)}$ RC} 
& \textbf{$\mathcal{L}^{(3)}$ QC} 
& \textbf{$\mathcal{L}^{(4)}$ RC} 
& \textbf{$\mathcal{L}^{(4)}$ QC} \\
\midrule
0.05 & 39.80 (0.00) & 42.32 (6.14) & 44.69 (7.53) & 42.32 (11.57) & 33.13 (13.55) & 52.33 (8.77) & 53.05 (8.03) \\
0.1  & 39.78 (0.00) & 59.02 (5.52) & 62.77 (6.50) & 66.91 (13.05) & 64.03 (14.72) & 74.88 (6.56) & \textbf{75.35} (6.59) \\
0.2  & 39.76 (0.00) & 75.52 (6.15) & 75.88 (5.00) & 79.83 (4.28) & 77.05 (5.04) & 81.95 (3.25) & \textbf{82.14} (2.90) \\
0.3  & 39.73 (0.00) & 80.56 (1.93) & 80.56 (1.64) & 83.21 (3.99) & 81.05 (4.27) & 83.68 (2.82) & \textbf{83.70} (2.81) \\
0.5  & 40.38 (0.00) & 84.98 (3.22) & 85.38 (3.34) & 85.06 (2.86) & 83.13 (3.24) & \textbf{85.92} (3.49) & \textbf{85.92} (3.43) \\
0.8  & 40.91 (0.00) & 86.59 (4.49) & \textbf{87.91} (4.10) & 87.55 (3.63) & 85.68 (4.16) & 84.68 (3.86) & 84.86 (3.88) \\
\bottomrule
\end{tabular}
\end{flushleft}

\begin{flushleft}
\begin{tabular}{lllll}
\toprule
\textbf{Rate} 
& clique 
& transductive 
& weighted transductive 
& dynamic transductive \\
\midrule
0.05 & 39.80 (0.00) & \textbf{55.63} (3.57) & \textbf{55.63} (3.57) & 39.80 (0.00) \\
0.1  & 39.78 (0.00) & 56.96 (2.02) & 56.96 (2.02) & 39.78 (0.00) \\
0.2  & 39.76 (0.00) & 57.37 (1.27) & 57.37 (1.27) & 39.76 (0.00) \\
0.3  & 39.73 (0.00) & 58.18 (1.60) & 58.18 (1.60) & 39.73 (0.00) \\
0.5  & 40.38 (0.00) & 58.46 (1.83) & 58.46 (1.83) & 40.38 (0.00) \\
0.8  & 40.91 (0.00) & 57.50 (2.69) & 57.50 (2.69) & 40.91 (0.00) \\
\bottomrule
\end{tabular}
\end{flushleft}
\end{small}
    \centering
    \caption{Accuracy of various SSL methods on the Zoo dataset. The best-performing method in each row is highlighted in bold.}
    \label{tab:zoo}
\end{table}

\begin{table}[htbp]

\begin{small}

\begin{flushleft}
\begin{tabular}{llllllll}
\toprule
\textbf{Rate} 
& \textbf{$\mathcal{L}^{(1)}$}  
& \textbf{$\mathcal{L}^{(2)}$ RC} 
& \textbf{$\mathcal{L}^{(2)}$ QC} 
& \textbf{$\mathcal{L}^{(3)}$ RC} 
& \textbf{$\mathcal{L}^{(3)}$ QC} 
& \textbf{$\mathcal{L}^{(4)}$ RC} 
& \textbf{$\mathcal{L}^{(4)}$ QC} \\
\midrule
0.05 & 51.79 (0.00) & 86.34 (0.81) & 86.30 (0.83) & 88.70 (1.06) & 88.39 (1.19) & 63.42 (5.55) & 88.00 (1.31) \\
0.1  & 51.80 (0.00) & 87.22 (0.38) & 87.13 (0.38) & 88.43 (0.79) & 88.45 (0.79) & 78.99 (3.17) & 88.87 (0.95) \\
0.2  & 65.71 (3.76) & 88.26 (0.39) & 88.34 (0.45) & 90.57 (0.69) & 90.60 (0.83) & 86.92 (1.47) & \textbf{91.87} (0.77) \\
0.3  & 84.86 (1.01) & 89.20 (0.36) & 89.32 (0.28) & 92.54 (0.52) & 92.45 (0.71) & 89.31 (1.05) & \textbf{93.27} (0.45) \\
0.5  & 89.74 (0.31) & 90.36 (0.49) & 90.27 (0.54) & 94.22 (0.31) & 94.20 (0.44) & 89.65 (0.55) & \textbf{94.27} (0.45) \\
0.8  & 89.53 (0.63) & 91.32 (0.72) & 91.29 (0.70) & \textbf{94.68} (0.51) & 94.66 (0.44) & 90.03 (0.68) & 94.66 (0.44) \\
\bottomrule
\end{tabular}
\end{flushleft}

\begin{flushleft}
\begin{tabular}{lllll}
\toprule
\textbf{Rate} 
& clique 
& transductive 
& weighted transductive 
& dynamic transductive \\
\midrule
0.05 & 51.79 (0.00) & \textbf{90.72} (0.67) & 90.01 (0.38) & 51.79 (0.00) \\
0.1  & 51.80 (0.00) & \textbf{90.80} (0.60) & 89.96 (0.12) & 51.80 (0.00) \\
0.2  & 69.72 (3.33) & 90.66 (0.34) & 90.02 (0.31) & 51.80 (0.00) \\
0.3  & 85.70 (0.87) & 90.65 (0.34) & 90.13 (0.27) & 51.79 (0.00) \\
0.5  & 89.69 (0.35) & 90.62 (0.33) & 90.24 (0.42) & 51.80 (0.00) \\
0.8  & 89.73 (0.68) & 90.56 (0.64) & 90.38 (0.13) & 51.78 (0.00) \\
\bottomrule
\end{tabular}
\end{flushleft}

\end{small}
    \centering
    \caption{Accuracy of various SSL methods on the Mushroom dataset. The best-performing method in each row is highlighted in bold.}
    \label{tab:mushroom}
\end{table}

\begin{table}[htbp]

\begin{small}

\begin{flushleft}
\begin{tabular}{llllllll}
\toprule
\textbf{Rate} 
& \textbf{$\mathcal{L}^{(1)}$}  
& \textbf{$\mathcal{L}^{(2)}$ RC} 
& \textbf{$\mathcal{L}^{(2)}$ QC} 
& \textbf{$\mathcal{L}^{(3)}$ RC} 
& \textbf{$\mathcal{L}^{(3)}$ QC} 
& \textbf{$\mathcal{L}^{(4)}$ RC} 
& \textbf{$\mathcal{L}^{(4)}$ QC} \\
\midrule
0.05 & 30.19 (0.00) & 30.19 (0.00) & 30.19 (0.00) & 30.19 (0.00) & 30.19 (0.00) & 30.27 (0.12) & \textbf{30.44} (0.49) \\
0.1  & 30.19 (0.00) & 30.19 (0.00) & 30.19 (0.00) & 30.19 (0.00) & 30.19 (0.00) & 30.29 (0.16) & \textbf{31.20} (0.91) \\
0.2  & 30.20 (0.00) & 30.20 (0.00) & 30.20 (0.00) & 30.20 (0.00) & 30.20 (0.00) & 31.85 (0.67) & \textbf{34.74} (1.49) \\
0.3  & 30.19 (0.00) & 30.19 (0.00) & 30.19 (0.00) & 30.19 (0.00) & 30.19 (0.00) & 35.96 (0.97) & \textbf{40.15} (1.13) \\
0.5  & 30.18 (0.00) & 30.89 (0.30) & 30.89 (0.22) & 30.18 (0.00) & 30.18 (0.00) & 44.39 (1.33) & \textbf{50.47} (1.22) \\
0.8  & 30.09 (0.00) & 34.86 (0.93) & 35.44 (0.92) & 30.09 (0.00) & 30.09 (0.00) & 54.75 (1.49) & \textbf{60.01} (1.42) \\
\bottomrule
\end{tabular}
\end{flushleft}

\begin{flushleft}
\begin{tabular}{lllll}
\toprule
\textbf{Rate} 
& clique 
& transductive 
& weighted transductive 
& dynamic transductive \\
\midrule
0.05 & 30.19 (0.00) & 30.19 (0.00) & 30.19 (0.00) & 30.19 (0.00) \\
0.1  & 30.19 (0.00) & 30.19 (0.00) & 30.19 (0.00) & 30.19 (0.00) \\
0.2  & 30.20 (0.00) & 30.20 (0.00) & 30.20 (0.00) & 30.20 (0.00) \\
0.3  & 30.19 (0.00) & 30.19 (0.00) & 30.19 (0.00) & 30.19 (0.00) \\
0.5  & 30.18 (0.00) & 30.18 (0.00) & 30.18 (0.00) & 30.18 (0.00) \\
0.8  & 30.09 (0.00) & 30.09 (0.00) & 30.09 (0.00) & 30.09 (0.00) \\
\bottomrule
\end{tabular}
\end{flushleft}

\end{small}
    \centering
    \caption{Accuracy of various SSL methods on the Cora dataset. The best-performing method in each row is highlighted in bold.}
    \label{tab:cora}
\end{table}

\begin{table}[htbp]
    
\begin{small}

\begin{flushleft}
\begin{tabular}{llllllll}
\toprule
\textbf{Rate} 
& \textbf{$\mathcal{L}^{(1)}$}  
& \textbf{$\mathcal{L}^{(2)}$ RC} 
& \textbf{$\mathcal{L}^{(2)}$ QC} 
& \textbf{$\mathcal{L}^{(3)}$ RC} 
& \textbf{$\mathcal{L}^{(3)}$ QC} 
& \textbf{$\mathcal{L}^{(4)}$ RC} 
& \textbf{$\mathcal{L}^{(4)}$ QC} \\
\midrule
0.05 & 21.06 (0.00) & 30.52 (8.39) & 31.14 (7.17) & 21.13 (0.25) & 21.44 (0.92) & 33.71 (6.14) & \textbf{35.14} (6.03) \\
0.1  & 21.05 (0.00) & 40.44 (7.23) & 38.04 (10.97) & 21.23 (0.24) & 21.93 (1.17) & 47.14 (4.32) & \textbf{48.26} (3.87) \\
0.2  & 21.06 (0.00) & 51.13 (3.36) & 53.05 (2.65) & 22.55 (1.31) & 24.26 (2.56) & 57.74 (1.58) & \textbf{57.86} (1.40) \\
0.3  & 21.06 (0.00) & 56.99 (1.83) & 56.70 (2.36) & 25.15 (1.65) & 27.91 (2.12) & \textbf{61.07} (0.99) & 60.89 (0.95) \\
0.5  & 21.09 (0.00) & 62.37 (1.13) & 62.52 (1.26) & 29.65 (1.15) & 34.25 (1.34) & \textbf{64.08} (0.91) & 63.66 (0.89) \\
0.8  & 21.11 (0.00) & 66.04 (1.38) & \textbf{66.36} (1.15) & 37.16 (0.75) & 43.47 (1.03) & 65.63 (1.56) & 65.24 (1.57) \\
\bottomrule
\end{tabular}
\end{flushleft}

\begin{flushleft}
\begin{tabular}{lllll}
\toprule
\textbf{Rate} 
& clique 
& transductive 
& weighted transductive 
& dynamic transductive \\
\midrule
0.05 & 21.06 (0.00) & 21.06 (0.02) & 21.06 (0.00) & 21.09 (0.05) \\
0.1  & 21.05 (0.00) & 21.05 (0.00) & 21.05 (0.00) & 21.05 (0.00) \\
0.2  & 21.06 (0.00) & 21.06 (0.00) & 21.06 (0.00) & 21.06 (0.00) \\
0.3  & 21.06 (0.00) & 21.06 (0.00) & 21.06 (0.00) & 21.06 (0.00) \\
0.5  & 21.09 (0.00) & 21.09 (0.00) & 21.09 (0.00) & 21.09 (0.00) \\
0.8  & 21.11 (0.00) & 21.11 (0.00) & 21.11 (0.00) & 21.11 (0.00) \\
\bottomrule
\end{tabular}
\end{flushleft}

\end{small}
\centering
    \caption{Accuracy of various SSL methods on the Citeseer dataset. The best-performing method in each row is highlighted in bold.}
    \label{tab:citeseer}
\end{table}

We observe that HOHL with \( \mathcal{L}^{(q)} \) and $2 \leq q \leq 4$ consistently either closely matches or achieves higher accuracy than both baseline hypergraph methods and the Laplacian on the clique-expanded graph. This suggests that the skeleton-based segmentation employed by Algorithm~\ref{alg:hohl} succeeds in isolating subgraphs that reflect relevant structure in the data. In particular, HOHL methods achieve markedly stronger performance on Citeseer and Cora, where conventional hypergraph baselines remain almost flat across all labeling rates — exceeding their accuracy by more than threefold on Citeseer (66.36\% for $\mathcal{L}^{(2)}$ QC vs. 21.11\% for baselines at 0.8 label rate) and roughly doubling it on Cora (60.01\% for $\mathcal{L}^{(4)}$ QC vs. 30.09\% for baselines at 0.8 label rate).

We also perform an ablation study comparing HOHL with only first-order regularization \( \mathcal{L}^{(1)} \) to the higher-order variant \( \mathcal{L}^{(q)} \) with $2 \leq q \leq 4$. The consistent performance gains from adding the higher-order terms suggest that higher-order regularization significantly enhances HOHL’s ability to capture label-relevant structure. The gap between the two versions widens with increasing label rates: on Citeseer, the difference in accuracy grows from 14.08 percentage points at a 0.05 label rate (35.14\% for $\mathcal{L}^{(4)}$ QC vs. 21.06\% for $\mathcal{L}^{(1)}$) to 45.25 points at 0.8 (66.36\% for $\mathcal{L}^{(2)}$ QC vs. 21.11\% for $\mathcal{L}^{(1)}$); on Zoo, the gain grows from 13.25 points at a 0.05 label rate (53.05\% for $\mathcal{L}^{(4)}$ QC vs. 39.80\% for $\mathcal{L}^{(1)}$) to 47.00 points at 0.8 (87.91\% for $\mathcal{L}^{(2)}$ QC vs. 40.91\% for $\mathcal{L}^{(1)}$). This effect is strongest when small hyperedges encode local patterns: taking higher powers of their skeleton Laplacians enforces smoothness across these subsets, yielding sharper decision boundaries.

Furthermore, we note that increasing the value of $\lambda_\ell$, i.e. comparing RC and QC configurations, can lead to large improvements: 88.00\% for $\mathcal{L}^{(4)}$ QC vs. 63.42\% for $\mathcal{L}^{(4)}$ RC at 0.05 label rate on Mushroom; 50.47\% for $\mathcal{L}^{(4)}$ QC vs. 44.39\% for $\mathcal{L}^{(4)}$ RC at 0.5 label rate on Cora.

Figure~\ref{fig:hyperedge_sizes} shows variation in hyperedge size distribution across datasets which influences how HOHL captures structure across scales.

\begin{itemize}
    \item In Zoo, the small dataset size increases the chance that early labeled nodes span both fine and coarse hyperedges, enabling HOHL to leverage multiscale structure even at low label rates. In contrast, Mushroom’s larger size makes early labels less likely to touch smaller, more informative hyperedges. HOHL methods thus surpass the transductive baseline only at higher label rates (starting from 0.2), whereas in Zoo they already outperform it at rate 0.1. 
    \item In Cora and Citeseer, the clear size gap between small and large hyperedges creates a strong separation of local and global interactions. As the label rate increases, small hyperedges become more useful, and HOHL's higher-order regularization captures these patterns. On Citeseer, accuracy improves from 31.14\% to 66.36\% across label rates 0.05 to 0.8 for $\mathcal{L}^{(2)}$ QC, while the transductive baseline stays flat at $\sim$21\%.
    \item The bimodal nature of the hyperedge size distribution in the Cora and Citeseer datasets suggests that an even number of groupings in Algorithm \ref{alg:hohl} would better align with the data structure. This intuition is supported by our results: $\mathcal{L}^{(q)}$ with $q=2,4$ consistently outperform $\mathcal{L}^{(3)}$ ($\mathcal{L}^{(3)}$ RC and QC remain flat on Cora; $\mathcal{L}^{(3)}$ RC and QC achieve 37.16\% and 43.47\% in comparison with 66.36\% for $\mathcal{L}^{(2)}$ QC and 65.63\% for $\mathcal{L}^{(4)}$ RC at 0.8 label rate on Citeseer). In contrast, for datasets like Zoo and Mushroom, where hyperedge sizes are more evenly distributed, the number of groupings appears less critical. In these cases, $\mathcal{L}^{(3)}$ performs comparably to $\mathcal{L}^{(q)}$ with $q=2,4$, confirming that uniform distributions are less sensitive to the choice of segmentation (at 0.8 label rate on Mushroom, we have 94.68\% for $\mathcal{L}^{(3)}$ RC and 94.66 \% for $\mathcal{L}^{(4)}$ QC).
\end{itemize}

\begin{table}[htbp]
\begin{small}
\begin{flushleft}
\begin{tabular}{lccccccc}
\toprule
\textbf{Dataset} 
& \textbf{$\mathcal{L}^{(1)}$}  
& \textbf{$\mathcal{L}^{(2)}$ RC} 
& \textbf{$\mathcal{L}^{(2)}$ QC} 
& \textbf{$\mathcal{L}^{(3)}$ RC} 
& \textbf{$\mathcal{L}^{(3)}$ QC} 
& \textbf{$\mathcal{L}^{(4)}$ RC} 
& \textbf{$\mathcal{L}^{(4)}$ QC} \\
\midrule
Zoo       & \textbf{0.00} (0.00) & \textbf{0.00} (0.00) & \textbf{0.00} (0.00) & \textbf{0.00} (0.00) & \textbf{0.00} (0.00) & \textbf{0.00} (0.00) & \textbf{0.00} (0.00) \\
Mushroom  & 19.91 (0.12) & 21.07 (0.04) & 21.07 (0.04) & 22.01 (0.11) & 22.01 (0.11) & 22.05 (0.12) & 22.05 (0.12) \\
Cora      & \textbf{2.79} (0.05) & 2.82 (0.05) & 2.82 (0.05) & 2.84 (0.01) & 2.84 (0.01) & 2.86 (0.02) & 2.86 (0.02) \\
Citeseer  & \textbf{4.31} (0.04) & 4.36 (0.07) & 4.36 (0.07) & 4.37 (0.01) & 4.37 (0.01) & 4.38 (0.03) & 4.38 (0.03) \\
\bottomrule
\end{tabular}
\end{flushleft}

\vspace{0.3em}

\begin{flushleft}
\begin{tabular}{lcccc}
\toprule
\textbf{Dataset} 
& clique 
& transductive 
& weighted transductive 
& dynamic transductive \\
\midrule
Zoo      & \textbf{0.00} (0.00) & \textbf{0.00} (0.00) & 0.01 (0.00) & 0.14 (0.01) \\
Mushroom & 19.72 (0.04) & \textbf{4.35} (0.10) & 41.91 (5.14) & 301.26 (2.29) \\
Cora     & \textbf{2.79} (0.03) & 8.19 (0.16) & 83.85 (27.77) & 137.15 (1.01) \\
Citeseer & 4.32 (0.05) & 33.02 (0.43) & 395.14 (224.17) & 553.73 (1.30) \\
\bottomrule
\end{tabular}
\end{flushleft}
    
\end{small}
    \centering
    \caption{Computation time in seconds for various SSL methods at label rate 0.1.}
    \label{tab:complexity}
\end{table}

Table~\ref{tab:complexity} reports the average time to solve the learning problem at label rate 0.1 (results are similar at all rates), excluding graph or hypergraph construction, which is performed once and reused across experiments. HOHL methods are run with a fixed, untuned configuration and no hyperparameter optimization. By contrast, the last two hypergraph baselines involve iterative solvers and require tuning of regularization parameters, leading to significantly longer runtimes. Despite its simplicity, HOHL methods consistently achieves strong performance while being quick to compute, underscoring its practical efficiency.

\section{Conclusion}

On the theoretical side, we proved that HOHL is well-posed as a regularizer in the fully supervised setting and established convergence rates between the discrete graph-based approximation and the underlying continuum target function. We further showed that spectrally truncated variants of HOHL remain consistent in the limit, supporting their use in practice.

On the practical side, we demonstrated that HOHL retains the quadratic structure of Laplace learning, making it a viable drop-in replacement within graph-based pipelines. In particular, we integrated HOHL into an active learning framework and observed substantial performance gains in low-label regimes. To generalize HOHL beyond geometric settings, we proposed a multiscale skeleton aggregation algorithm that enables efficient regularization even in the absence of spatial embeddings. Our approach achieves state-of-the-art performance, and we analyzed the impact of HOHL’s parameters in relation to the hyperedge size distribution of the dataset.

Future work includes analyzing HOHL through the lens of reproducing kernel Hilbert space (RKHS) theory, following approaches such as \cite{zhang2005spectral}, to derive expected error bounds in the semi-supervised setting as a function of the length-scales. Additionally, adaptive skeleton segmentation and parameter selection strategies—e.g., cross-validation, meta-learning, or Bayesian optimization—could further improve robustness. Finally, integrating HOHL into end-to-end differentiable models may enable closer connections to neural architectures, while extending it to dynamic or multilayer hypergraphs opens avenues for application to temporal and multiplex data.

\paragraph{Acknowledgments}

AW and AB were supported in part by NSF grant DMS-2152717. 
MT acknowledges the support of the EPSRC Mathematical and Foundations of Artificial Intelligence Probabilistic AI Hub (grant agreement EP/Y007174/1), the Leverhulme Trust through the Project Award ``Robust Learning: Uncertainty Quantification, Sensitivity and Stability'' (grant agreement RPG-2024-051) and the NHSBT award 177PATH25 ``Harnessing Computational Genomics to Optimise Blood Transfusion Safety and Efficacy''.

\bibliography{references}{}
\bibliographystyle{plain}

\clearpage

\end{document}